\def\arxiv{}

\documentclass[11pt]{article}
\usepackage[margin=1in]{geometry}
\usepackage{amssymb,amsthm}
\usepackage{todonotes}
\usepackage{amsmath,amsthm,amssymb,amsfonts,cancel}
\usepackage{thmtools,enumitem,subfigure,xspace}

\usepackage{tikz}
\usetikzlibrary{bayesnet}
\usetikzlibrary{arrows}

\makeatletter
\def\blfootnote{\xdef\@thefnmark{}\@footnotet   ext}
\makeatother

\usepackage{makecell}

\usepackage{mathtools}
\usepackage{dsfont}
\usepackage{tcolorbox}
\usepackage{todonotes}
\usepackage{tensor}
\usepackage{caption}
\usepackage{bm}
\ifdefined\arxiv
\usepackage[colorlinks=true,breaklinks=true,bookmarks=true,urlcolor=magenta,citecolor=magenta,linkcolor=magenta,bookmarksopen=false,draft=false]{hyperref}
\else
\usepackage{hyperref}
\fi
\usepackage{placeins}
\usepackage[capitalize,noabbrev]{cleveref}

\newenvironment{rproof}{ \ifdefined\informs \proof{\it Proof.}
\else \begin{proof} \fi }{ \ifdefined\informs 
\Halmos \endproof \else \end{proof} \fi  }

\renewcommand{\P}{\mathcal{P}}
\newcommand{\D}{\mathcal{D}}
\newcommand{\A}{\mathcal{A}}

\newcommand{\X}{\mathcal{X}}

\newcommand{\M}{\mathcal{M}}
\newcommand{\inv}{\texttt{Inv}\xspace}

\makeatletter
\def\blfootnote{\xdef\@thefnmark{}\@footnotet   ext}
\makeatother

\crefformat{section}{\S#2#1#3} 
\crefformat{subsection}{\S#2#1#3}
\crefformat{subsubsection}{\S#2#1#3}

\newcommand{\GuidedMAC}{\textsc{Hindsight MAC}\xspace}
\newcommand{\GuidedRL}{Hindsight Learning~}
\newcommand{\GuidedQD}{\textsc{Hindsight Q-Distillation}\xspace}
\newcommand{\DeltaIPGap}{\Delta^{\texttt{\dag}}}
\newcommand{\Qip}{Q^{\texttt{\dag}}}
\newcommand{\Vip}{V^{\texttt{\dag}}}

\newcommand{\MARO}{\texttt{MARO}\xspace}

\newcommand{\pip}{\pi^{\texttt{\dag}}}
\newcommand{\perm}{\overline{\pi}^\star}
\newcommand{\ORSuite}{\texttt{ORSuite}\xspace}
\newcommand{\pibar}{\overline{\pi}}
\newcommand{\pstar}{\pi^\star}
\newcommand{\Exp}[1]{\mathbb E \left[ #1 \right]} 
\newcommand{\BarExp}[1]{\overline{\mathbb E} \left[ #1 \right]}
\newcommand{\bfxi}{\boldsymbol \xi}
\newcommand{\PXi}{\mathcal{P}_{\Xi}}
\newcommand{\plan}{\textsc{Hindsight}}
\newcommand{\Regret}{\textsc{Regret}}

\mathchardef\mhyphen="2D 

\ifdefined \argmin
\else \DeclareMathOperator*{\argmin}{arg\,min}
\fi
\ifdefined \argmax
\else \DeclareMathOperator*{\argmax}{arg\,max}
\fi
\DeclarePairedDelimiter{\norm}{\lVert}{\rVert}
\DeclarePairedDelimiter{\abs}{\lvert}{\rvert}

\newcommand{\E}{\mathbb{E}}

\let\originalleft\left
\let\originalright\right
\renewcommand{\left}{\mathopen{}\mathclose\bgroup\originalleft}
\renewcommand{\right}{\aftergroup\egroup\originalright}

\DeclareMathOperator*{\tsum}{\textstyle\sum}
\newcommand{\Ind}[1]{\mathds{1}_{\left[ #1 \right]}}

\setlist[itemize]{leftmargin=*}

\newtheorem{assumption}{Assumption}
\newtheorem{definition}{Definition}
\newtheorem{theorem}{Theorem}

\newtheorem{lemma}[theorem]{Lemma}

\usepackage[suppress]{color-edits}
\addauthor{sb}{violet}
\addauthor{srstodo}{blue}
\addauthor{srs}{orange}
\addauthor{ca}{purple}
\addauthor{as}{violet}

\usepackage{algorithm,algorithmic}
\usepackage{natbib}

\usepackage{authblk}

\usepackage{booktabs} 

\allowdisplaybreaks

\begin{document}
	\title{Hindsight Learning for MDPs with Exogenous Inputs}
        \author[1]{Sean R. Sinclair\footnote{Contact authors: Sean Sinclair \url{srs429@cornell.edu} and Adith Swaminathan \url{adswamin@microsoft.com}}}
        \author[2]{Felipe Frujeri}
        \author[2]{Ching-An Cheng}
        \author[2]{Luke Marshall}
        \author[2]{Hugo Barbalho}
        \author[3]{Jingling Li}
        \author[2]{Jennifer Neville}
        \author[2]{Ishai Menache}
        \author[2]{Adith Swaminathan$^*$}

        \affil[1]{School of Operations Research and Information Engineering, Cornell University}
        \affil[2]{Microsoft Research, Redmond}
        \affil[3]{Department of Computer Science, University of Maryland}

	\maketitle





	\begin{abstract}
     Many resource management problems require sequential decision-making under uncertainty, where the only uncertainty affecting the decision outcomes are exogenous variables outside the control of the decision-maker.
We model these problems as Exo-MDPs (Markov Decision Processes with Exogenous Inputs) 
and design a class of data-efficient algorithms for them termed \GuidedRL (HL). 
Our HL algorithms achieve data efficiency by leveraging a key insight: having samples of the exogenous variables, past decisions can be revisited in hindsight to infer counterfactual consequences that can accelerate policy improvements. 
We compare HL against classic baselines in the multi-secretary and airline revenue management problems. We also scale our algorithms to a  business-critical cloud resource management problem -- allocating Virtual Machines (VMs) to physical machines, and simulate their performance with real datasets from a large public cloud provider. We find that HL algorithms outperform domain-specific heuristics, as well as 
state-of-the-art reinforcement learning methods.
	\end{abstract}

	\newpage
	\setcounter{tocdepth}{3}
	\tableofcontents
	\newpage

\section{Introduction}
\label{sec:introduction}

\ifdefined\arxiv

\else
\begin{figure*}[htb!]
\centering
\subfigure{\includegraphics[\empty,height=30mm,]
{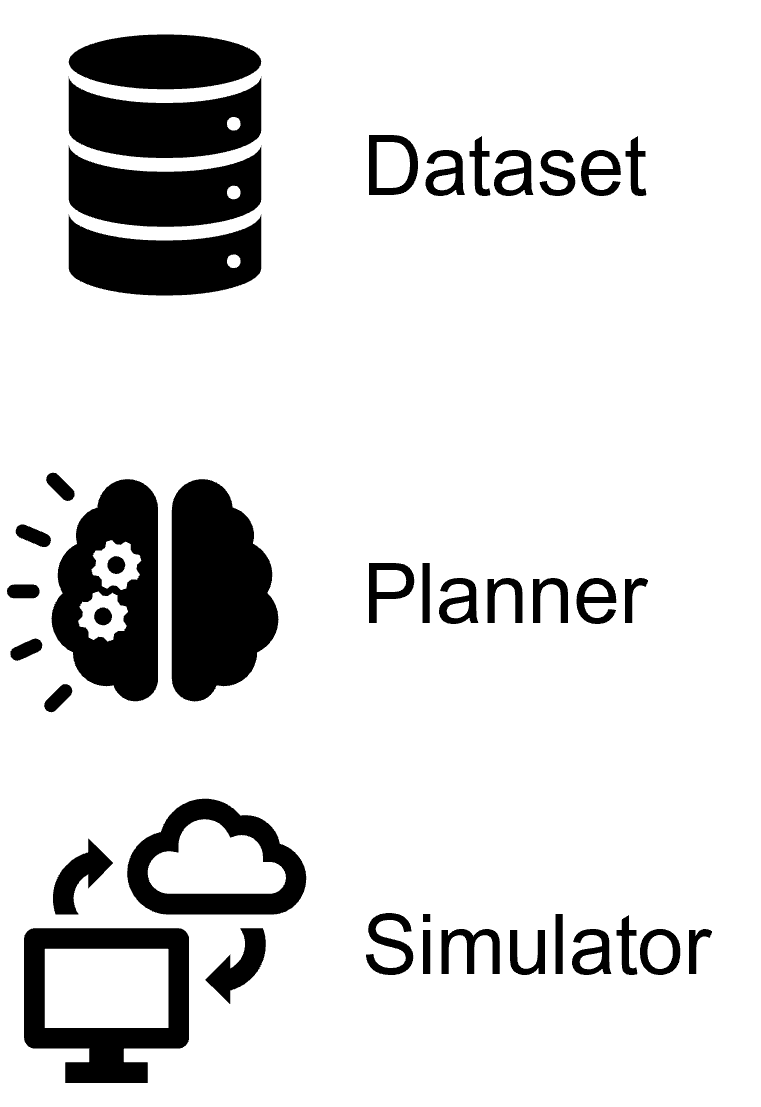}}\hspace{15mm}
\subfigure{\label{fig:pto}\includegraphics[height=30mm]{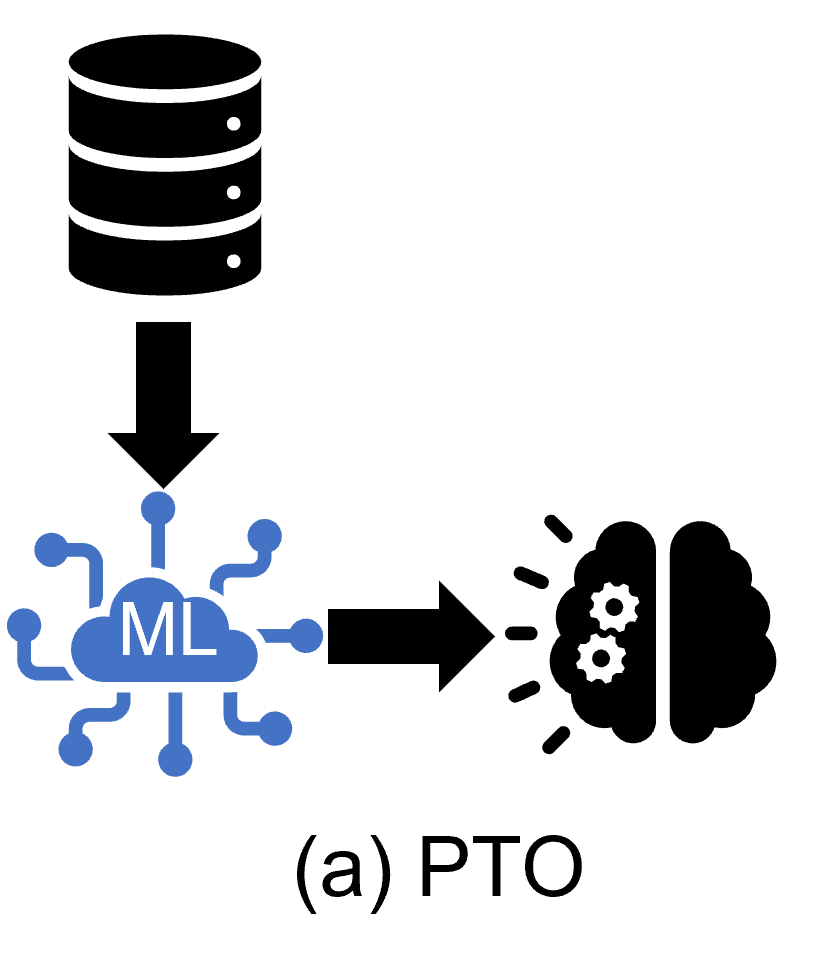}}\hspace{10mm}
\subfigure{\label{fig:rl}\includegraphics[height=30mm]{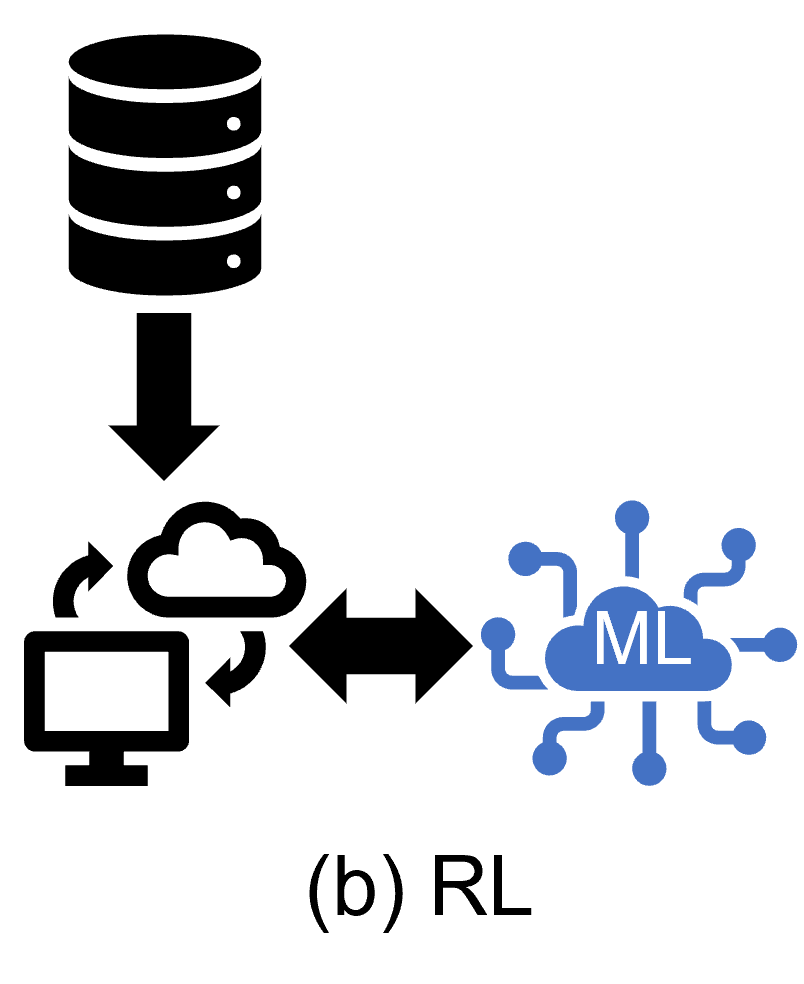}}\hspace{10mm}
\subfigure{\label{fig:hl}\includegraphics[height=30mm]{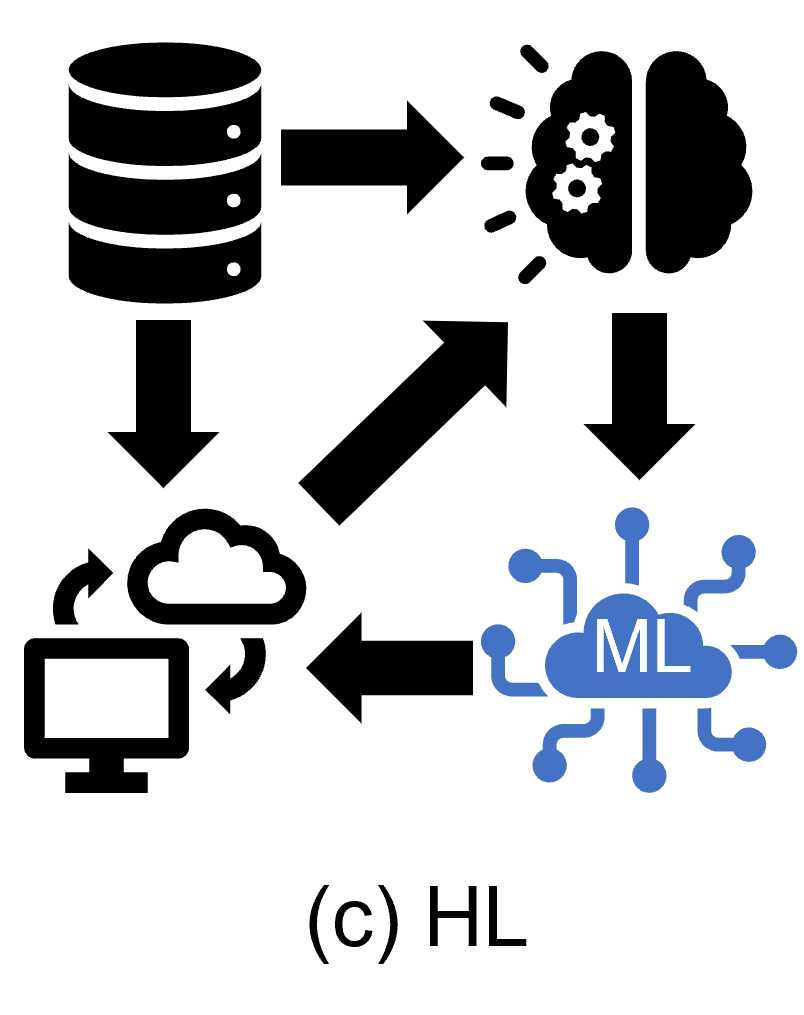}}
\vspace{-3mm}
\caption{\srsedit{Conceptual view of different ML approaches to solving Exo-MDPs ((a) and (b) are detailed in \cref{sec:exogenous_value}). Arrows in the figure indicate the flow of information from one component to the other.  In (a) {\bf Predict-Then-Optimize} (PTO) uses the dataset to train an ML forecasting model, uses the model to predict future exogenous inputs, and uses the forecasted inputs online for planning optimal actions.   
In (b) {\bf RL} replays the dataset through the simulator to evaluate an ML policy’s performance, and tunes policy parameters using the collected rewards as the training signal.  In (c) our {\bf \GuidedRL}(HL) approach uses the dataset directly with the planner (top arrow) to identify hindsight-optimal values, and trains the ML policy using state trajectories from the simulator annotated with the hindsight optimal values. }
} 
\label{fig:approaches}
\vspace{-2mm}
\end{figure*}
\fi

Many aspects of our physical and digital infrastructure --- like data centers, power grids, and supply chains --- can become more adaptive and efficient through data-driven decision-making. For instance, in a world-wide cloud service, even 1\% lower resource fragmentation can reduce energy use and save approximately \$100M per year~\citep{hadary_protean_2020}. This type of improvement could be achieved by examining historical patterns of compute demands and using machine learning (ML) to allocate future demands more efficiently. 

In this work, we make the key observation that in resource management applications the system is often partially known and the only uncertainty is due to \emph{exogenous} variables like resource requests --- that are (to a first-order approximation) \emph{independent} of an agent's decisions~\citep{powell_reinforcement_nodate}. For example, a cloud operator deciding to place a virtual machine (VM) on a specific server rack does not directly affect future VM requests, 
but future demands can strongly affect the eventual quality of their allocation decisions~\citep{hadary_protean_2020}. 
We define these problems as \emph{Exo-MDPs} (Markov Decision Processes with Exogenous Inputs), which are a subclass of \emph{Input-Driven MDPs}~\citep{mao_variance_2019}. In Input-Driven MDPs the minimal state describing the system dynamics decompose into (i) {\em exogenous} inputs, which evolve independently of the agent's actions, and (ii) {\em endogenous} factors that are impacted by the agent's actions and the exogenous inputs. 
Exo-MDPs make the additional assumption that the only unknowns are the distribution of future exogenous inputs (see~\cref{sec:preliminary}). 
\ifdefined\arxiv
This assumption often holds in resource management applications due to determinism in key system elements. As such, due to determinism {\em within} the system, the key challenge is to learn in the context of uncertainty {\em external} to the system. 
\else \fi

ML has been applied in several resource management applications, which we will show are Exo-MDPs, and found to outperform domain-specific heuristics~\citep{lykouris2021competitive,purohit2018improving,gollapudi2019online}.
The ML approaches often follow the Predict-Then-Optimize (PTO) paradigm~\citep{elmachtoub2022smart}, using ML to forecast the future exogenous inputs (e.g., demands). However, when the future is highly stochastic, forecasting is challenging. For example, VM requests from real-world data-centers~\citep{hadary_protean_2020} show long-term regularity of diurnal and weekly patterns but extreme short-term fluctuations (see Figure~\ref{fig:vm_trace_summary}). 

Reinforcement Learning (RL) is an alternative to PTO. RL directly optimizes decision quality~\citep{chen2020powernet,fang2019reinforcement,mao2016resource} by replaying historical samples of the exogenous inputs through the known dynamics of endogenous factors to learn good policies~\citep{madeka2022deep}. 
RL methods applied to Exo-MDPs must, however, learn by trial-and-error that their actions can never affect the exogenous inputs. RL is thus sensitive to variance in the outcomes introduced by the exogenous inputs and requires more data to learn an optimal policy when the variance is high~\citep{foster2021statistical} (see \cref{sec:example}).

Recent works have proposed to use \emph{hindsight} control variates to reduce the variance of RL for input-driven MDPs (which include Exo-MDPs)~\citep{mao_variance_2019,mesnard_counterfactual_2021}. 
They derive unbiased policy gradients by subtracting from the observed outcomes a function that depends additionally on hindsight information.  
However, we find that for many resource management scenarios, the variance reduction from hindsight control variates is not enough for data-efficient learning in practical regimes (see Table~\ref{tab:performance}).

\ifdefined\arxiv
\begin{figure*}[t!]
\centering
\subfigure{\includegraphics[\empty,height=30mm,]
{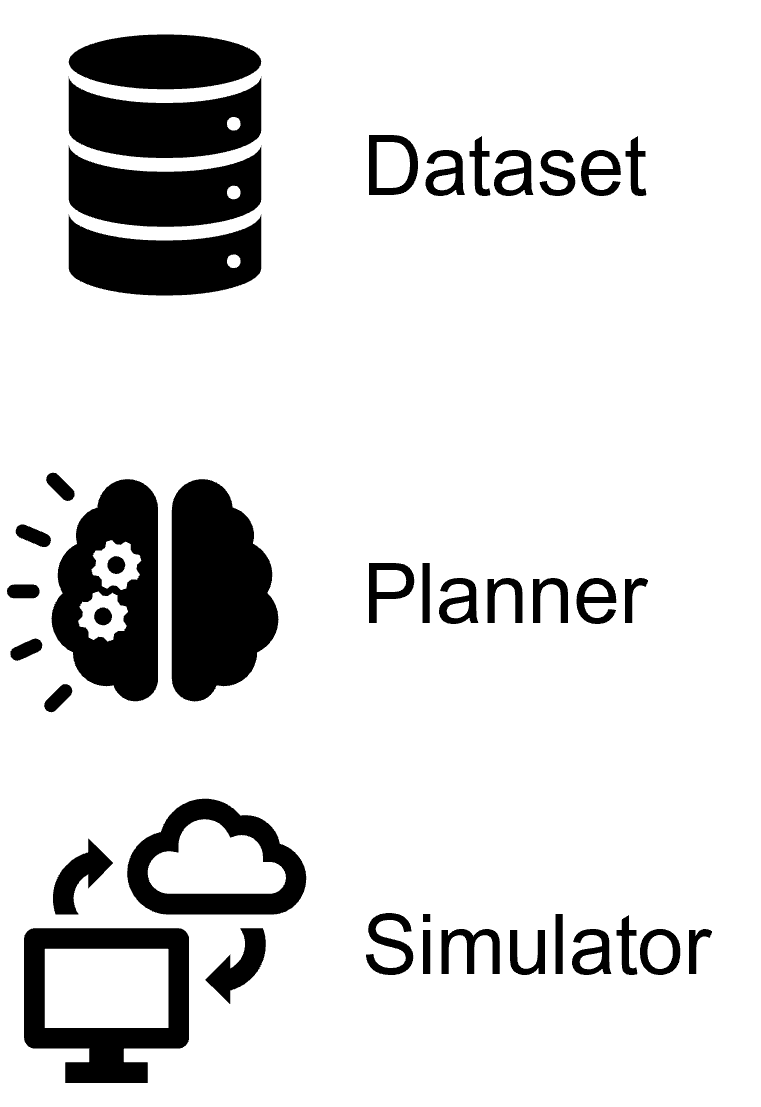}}\hspace{15mm}
\subfigure{\label{fig:pto}\includegraphics[height=30mm]{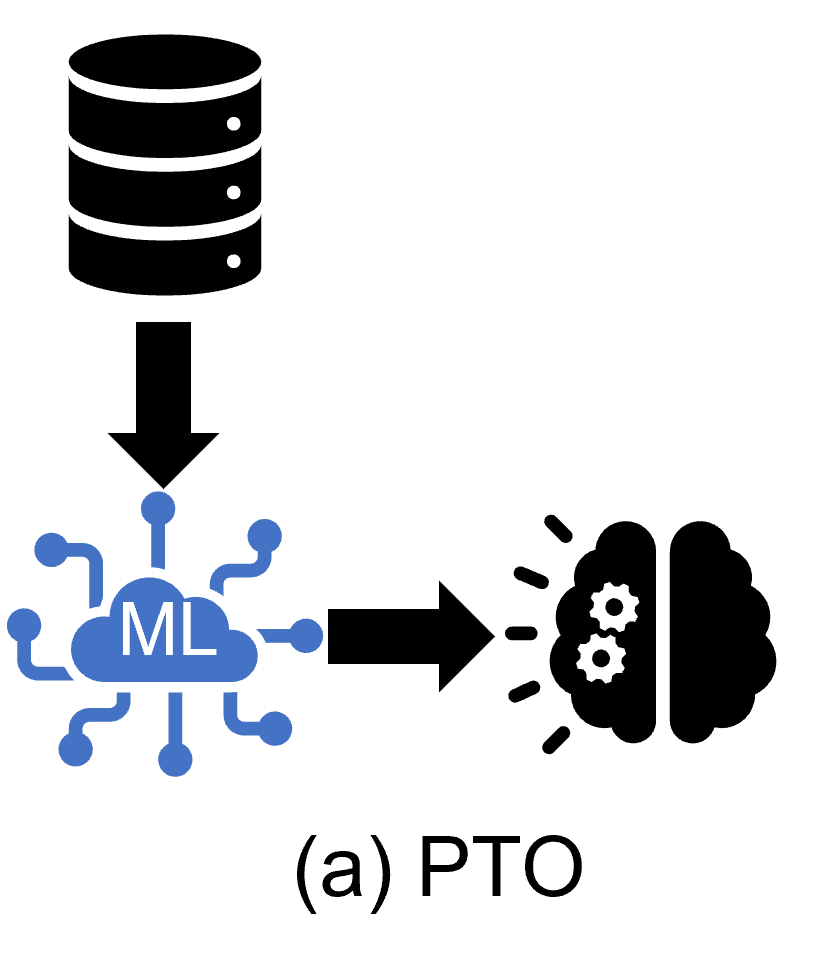}}\hspace{10mm}
\subfigure{\label{fig:rl}\includegraphics[height=30mm]{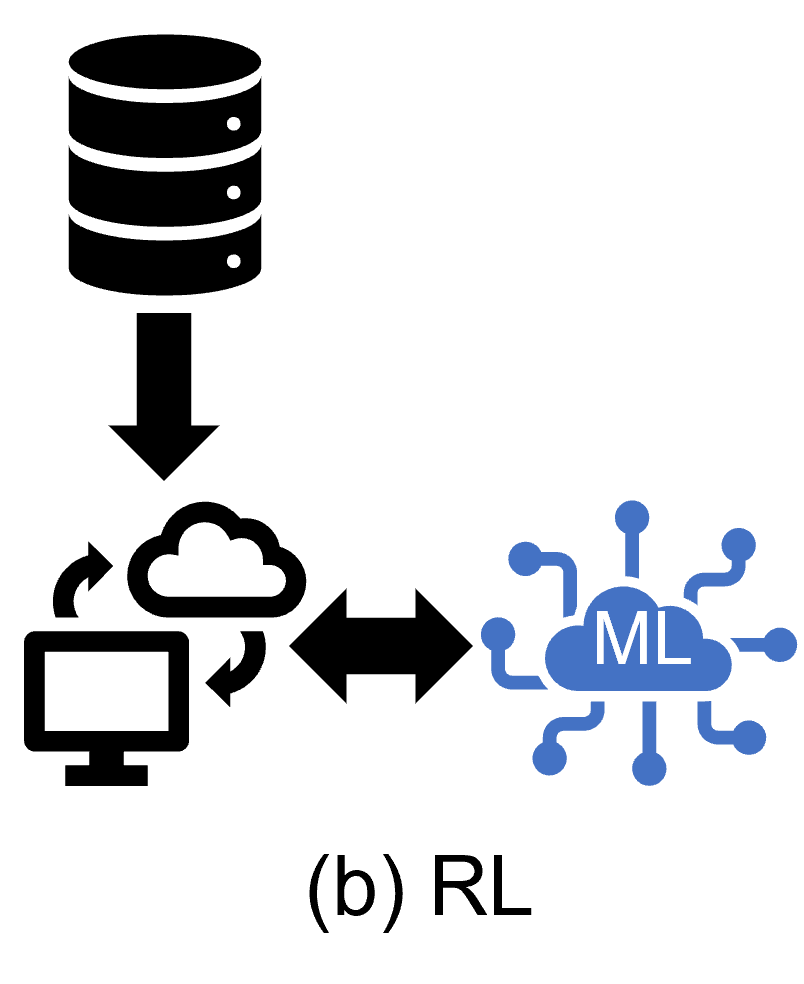}}\hspace{10mm}
\subfigure{\label{fig:hl}\includegraphics[height=30mm]{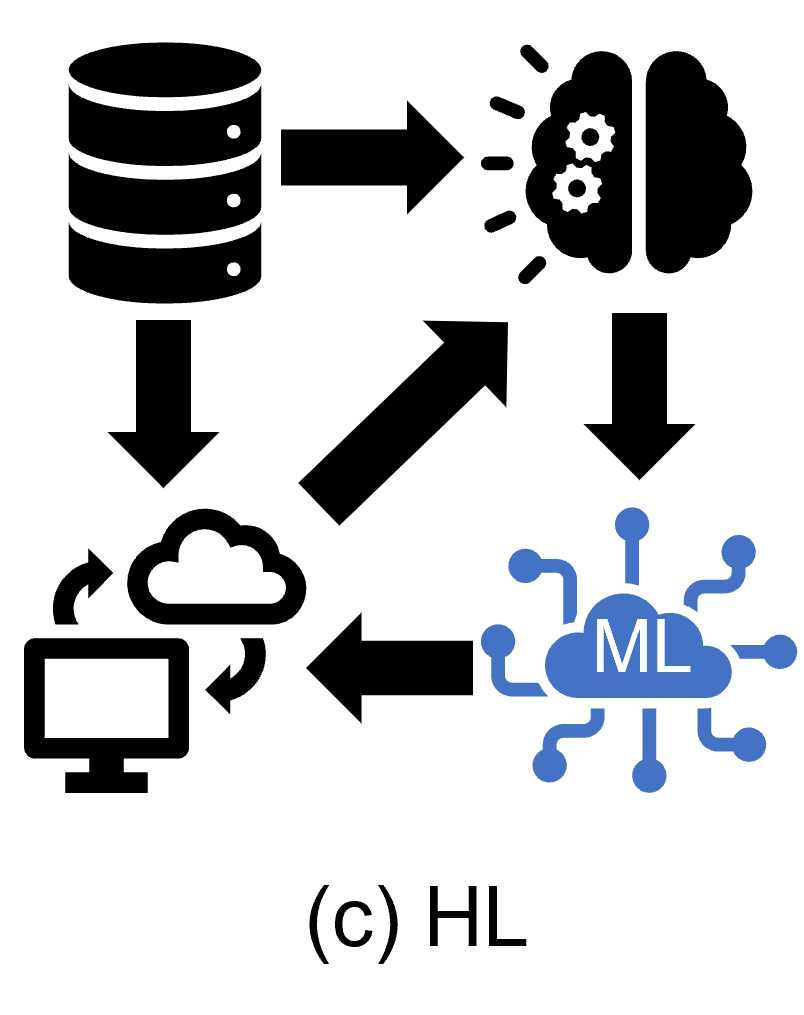}}
\vspace{-3mm}
\caption{\srsedit{Conceptual view of different ML approaches to solving exo-MDPs ((a) and (b) are detailed in \cref{sec:exogenous_value}). Arrows in the figure indicate the flow of information from one component to the other.  In (a) {\bf Predict-Then-Optimize} (PTO) uses the dataset to train an ML forecasting model, uses the model to predict future exogenous inputs, and uses the forecasted inputs online for planning optimal actions.   
In (b) {\bf RL} replays the dataset through the simulator to evaluate an ML policy’s performance, and tunes policy parameters using the collected rewards as the training signal.  In (c) our {\bf \GuidedRL}(HL) approach uses the dataset directly with the planner (top arrow) to identify hindsight-optimal values, and trains the ML policy using state trajectories from the simulator annotated with the hindsight optimal values. }
} 
\label{fig:approaches}
\vspace{-2mm}
\end{figure*}
\else
\fi

{
We argue that hindsight information can be more effectively used to improve learning, as it can largely reduce the variance of exogenous inputs at the cost of a small amount of asymptotic bias in many cases. Based on this insight, we develop a family of algorithms called \GuidedRL (HL), which uses hindsight planners during learning to lessen the variance from exogenous inputs.  
A hindsight planner~\citep{chong2000framework,gopalan2010polynomial,conforti2014integer} is an optimization algorithm that provides computationally tractable approximations to hindsight-optimal decisions, which is the optimal action specialized to a fixed sequence of future demands and thus is not affected by the exogenous variability. 
Therefore, by using hindsight planners during learning, HL can more efficiently identify decisions critical to future performance under high-variance exogenous inputs. Figure~\ref{fig:approaches} contrasts the HL algorithm schematic with PTO and RL.}


We theoretically characterize when HL succeeds using a novel quantity called {\em hindsight bias} (which arises due to the mismatch between truly optimal and hindsight-optimal decisions). Remarkably, we find that hindsight bias is small for many resource management problems, so HL can learn with extremely limited data for them. 
To prove this, we use recent advances in prophet inequalities~\citep{dutting2020prophet,vera2021bayesian} with novel adaptations for the Exo-MDP setting. 
We empirically test HL in several domains: Multi-Secretary problems, Airline Revenue Management (ARM) benchmarks, and Virtual Machine (VM) allocation. We find that HL is significantly better than both domain heuristics and RL (with and without hindsight control variates), illustrating that HL indeed strikes a better bias-variance trade-off. 
Notably, our VM allocation experiments use real historical traces from a large public cloud provider, where HL is the only approach that consistently beats the currently used heuristics ($0.1\%-5\%$ better). Recall that even a $1\%$ better allocator can yield massive savings in practice~\citep{hadary_protean_2020}.

\section{Related Work}
\label{sec:related_work}

Here we include a brief discussion of salient related work. Please see \cref{sec:related_work_full} for more details.

Recent studies have exploited causal structure in MDPs for better decision-making~\citep{lattimore2016causal,lu2022efficient}. Their causal graphs however do not capture Exo-MDPs (e.g., Figure~\ref{fig:example}). When the exogenous process is only partially observed, HL may additionally need causal RL techniques~\citep{zhang2020causal}; this is left for future work. 

Input-Driven MDPs have been specialized before with additional assumptions that either the rewards or the transitions factorize so that the exogenous process can be filtered out~\citep{dietterich_discovering_2018,efroni2022sample}. However, they are not suited for Exo-MDPs because filtering out the exogenous process yields demand-agnostic policies, which are highly sub-optimal for resource management problems.

\srsedit{Hindsight optimization has been previously attempted~\cite{chong2000framework,feng2021scalable}, and it is well known that these values are over-optimistic.  \citet{mercier2007performance} show that despite the over-optimism, the regret for hindsight optimization policies in several network scheduling and caching problems is a constant.  The Progressive Hedging (PH) algorithm in \citet{rockafellar1991scenarios} attempts to eliminate the over-optimism by iteratively refining the hindsight optimization solution by adding non-anticipative constraints.  PH has weak guarantees for convergence in non-convex problems (like our settings) and is intractable for the large problem sizes that we consider. 
Information relaxation~\citep{brown2021information} extends PH to non-convex rewards and arbitrary action spaces and allows for imperfect penalization of the non-anticipative constraint violations.  These schemes require hand-crafted penalties as well as tractable hindsight planning with those penalized objectives.  Instead, \citet{mercier2007performance} foregoes solving for the optimal policy of the MDP and instead produces a potentially sub-optimal non-anticipatory policy.  We provide a tighter regret analysis for their surrogate policy (\cref{lem:mercier_example}) and additionally describe an imitation learning algorithm to avoid unnecessary computation in large-scale problems. 
}

\section{Problem Setting and Definitions}
\label{sec:preliminary}

\subsection{MDPs with Exogenous Inputs (Exo-MDPs)}

We consider a subclass of finite horizon \ifdefined\short \else~\footnote{All the algorithms and analysis also apply analogously to infinite horizon problems with a discount factor.}\fi Markov Decision Processes (MDPs).
An MDP is defined by 
$(\mathcal{S}, \A, T, P, R, s_1)$ with horizon $T$, state space $\mathcal{S}$, action space $\A$, reward distribution $R$, state transition distribution $P$, and starting state $s_1$~\citep{puterman2014markov}.
An Exo-MDP further specializes an MDP by separating the process into \emph{endogenous} and \emph{endogenous} parts: a state $s\in\mathcal{S}$ factorizes into {endogenous}/{system} state $x \in \X$ and exogenous inputs $\bfxi \coloneqq (\xi_1, \ldots, \xi_T) \in \Xi^T$ (namely, $\mathcal{S} \coloneqq \X \times \Xi^{T}$). The state transition distribution $P$ also factors into an endogenous part $f$ and exogenous part $\PXi$ as follows. 
At time $t$, the agent selects action $a_t \in \A$ based on the current state $s_t \coloneqq (x_t, \bfxi_{< t})$ where $\bfxi_{< t} \coloneqq (\xi_1, \ldots, \xi_{t-1})$ is the observed exogenous inputs thus far, \emph{and then} $\xi_t$ is sampled from an unknown distribution $\PXi(\cdot \mid \bfxi_{< t})$, independently of the agent's action $a_t$.
With $\xi_t$, the endogenous state evolves according to $x_{t+1} = f(s_t, a_t, \xi_{t})$ and the reward earned is $r(s_t, a_t, \xi_{t}) \in [0,1]$.  Note $\xi_t$ is only observed when the agent makes the decision at time $t+1$, \emph{not} at time $t$.
We restrict our attention to policies $\pi_t : \X \times \Xi^{t-1} \rightarrow \Delta(\A)$ and let $\Pi$ denote the policy class of the agent.
The endogenous dynamics $f$ and reward function $r$ are assumed to be known to the agent\footnote{Thus, Exo-MDPs are a subclass of Input-Driven MDPs~\citep{mao_variance_2019} which more generally have unknown $f,r$. \ifdefined\short \else \cref{sec:gen_examples_mdp} shows that even though $f$ and $r$ are known in Exo-MDPs, they can be as difficult as  arbitrary MDPs with state space $\X$.\fi}; the only unknown in the Exo-MDP is $\PXi$.  \srsedit{For notational convenience, we also assume that $f$ and $r$ are deterministic; all of our insights carry over to stochastic rewards and transitions. }
These assumptions of Exo-MDPs are well-motivated for resource management problems, and we list the state decomposition along with $f$ and $r$ for several examples in \cref{sec:gen_examples_mdp}. 

\subsection{Value Decomposition in Exo-MDPs}

Since the only unknown in an Exo-MDP is $\PXi$, a policy's performance can be written as expectations over $\PXi$. This motivates the use of historical samples $\bfxi \sim \PXi$ to evaluate any policy and find optimal policies for the Exo-MDP.

For $\pi \in \Pi$, the \emph{values} and \emph{action-values} are defined as:
\begin{align*}
    V^\pi_t(s) &\coloneqq \E_{\bfxi_{\ge t}, \pi} [\tsum_{\tau \geq t} r(s_\tau, a_\tau, \xi_{\tau}) \mid s_t = s], \\
    Q^\pi_t(s,a) &\coloneqq \E_{\bfxi_{\ge t}, \pi} [\tsum_{\tau \geq t} r(s_\tau, a_\tau, \xi_{\tau}) \mid s_t = s, a_t = a ],
\end{align*}
where the expectation is taken over the randomness in $\pi$ and the exogenous inputs $\bfxi$.
We denote $\pi^\star$ as the optimal policy, i.e. the policy that maximizes $V_t^\pi(s)$ in each state $s$, and denote $Q^\star, V^\star$ for $Q^{\pi^\star}, V^{\pi^\star}$ respectively. Our goal is to find a policy with near-optimal returns, $\argmax_{\pi \in \Pi} \quad \{ V_1^{\pi} \coloneqq V_1^{\pi}(s_1) \}$. Or equivalently, minimize $\Regret(\pi)$ where 
$ \Regret(\pi) \coloneqq V_1^\star - V_1^{\pi}.$ 
For convenience we assume\ifdefined\short \else\footnote{If not, our theoretical results need an additional term for the difference in returns between the best policy in $\Pi$ and $\pi^\star$.}\fi \,that $\pi^\star \in \Pi$. 
We introduce value functions for fixed $\bfxi= \{\xi_1, \ldots, \xi_{T}\}$ as \begin{align}
\label{eq:q_exog}
\hspace{-2mm}
    Q_t^\pi(s,a,\bfxi_{\geq t}) & \coloneqq \E_\pi [\tsum_{\tau \geq t} r(s_\tau, a_\tau, \xi_\tau) | s_t = s, a_t = a], \\
    V_t^\pi(s,\bfxi_{\geq t}) & \coloneqq \sum_a \pi(a| s) Q_t^{\pi}(s,a,\bfxi_{\geq t}).
\end{align}
Note that the expectation is not over $\PXi$ because $\bfxi$ is fixed.
The $\bfxi$-specific values are related to policy values as follows.

\begin{lemma}
\label{thm:exog_bellman}
For every $t \in [T], (s,a) \in \mathcal{S} \times \A,$ policy $\pi \in \Pi$ we have that
\begin{align}
    \label{eq:exog_decomp_expectation}
    Q_t^\pi(s,a) & = \E_{\bfxi_{\geq t}}[Q_t^{\pi}(s,a,\bfxi_{\geq t})], \\
    V_t^\pi(s) & = \E_{\bfxi_{\geq t}}[V_t^{\pi}(s, \bfxi_{\geq t})]. \label{eq:exog_value}
\end{align}
In particular $V_1^{\pi} = \E_{\bfxi}[V_1^{\pi}(s_1, \bfxi)]$.
\end{lemma}
We relegate complete proofs to \cref{sec:proofs}.
Since the transition dynamics $f$ and reward function $r$ are known and the unknown $\PXi$ does not depend on the agent's actions, an immediate consequence is that \emph{action exploration is not needed}. Given any policy $\pi$ and exogenous trace $\bfxi$ we can simulate with $f$ and $r$ to calculate its return in Equation~\ref{eq:q_exog}. 

Suppose we collected traces $\D = \{\bfxi^1, \ldots, \bfxi^N\}$ where each trace $\bfxi^n = \{\xi_1^n \ldots, \xi_T^n\}$ is sampled independently from $\PXi$. 
Finding a near-optimal policy using this historical dataset is known as the \emph{offline RL} problem~\citep{fujimoto2019off,liu2020provably,rashidinejad2021bridging,cheng2022adversarially}, but this is much simpler in Exo-MDPs.
We do not face \emph{support mismatch} wherein trajectories from a data-collection policy may not cover the scenarios that the learner policy would encounter. 
Here $\D$ (collected by a behavior policy) can be safely replayed to evaluate any learner policy. This fact also implies that model selection and hyper-parameter tuning can be safely done using a held-out $\D$ akin to supervised learning. 
Our goal finally is to learn policies that generalize from $\D$ to the unknown $\PXi$, which can be challenging because the exogenous inputs $\bfxi$ introduce variance in a policy's return estimation.

\subsection{Hindsight Planner}

Exo-MDPs not only allow easy policy evaluation using a dataset of traces, but they also allow computing valuable hindsight information like the hindsight-optimal decisions for a trace $\bfxi$. This hindsight information can be stable even when $\PXi$ is highly stochastic. We now make a computational assumption for calculating hindsight-optimal decisions that will enable tractable algorithms for Exo-MDPs.

\begin{assumption}
\label{ass:planning_oracle}
Given any trace $\bfxi_{\geq t} = (\xi_t, \ldots, \xi_T)$ and state $s = (x_t, \bfxi_{< t})$ we can tractably solve:
\begin{align}
\label{eq:plan_assumption}
    \max_{a_t, \ldots, a_T} & \textstyle \sum_{\tau = t}^T r(s_\tau, a_\tau, \xi_{\tau}) \\
    \textrm{s.t. } x_{\tau+1} & = f(s_\tau, a_\tau, \xi_{\tau}), \textrm{ for } \tau = t, \ldots, T \nonumber\\
    s_\tau & = (x_\tau, \bfxi_{< \tau}), \textrm{ for } \tau = t, \ldots, T \nonumber. 
\end{align} 
We denote the optimal objective value to this problem as $\plan(t, \bfxi_{\geq t}, s)$. 
\end{assumption}
The optimization community has developed computationally efficient implementations for the $\plan(t, \bfxi_{\geq t}, s)$ oracle; with tight bounds on the optimal value even when \cref{eq:plan_assumption} is intractable. 
For example, online knapsack for a fixed input sequence can be solved in pseudo-polynomial time~\citep{gopalan2010polynomial}.
In many instances, Equation~\ref{eq:plan_assumption} can be represented as an integer program and solved via heuristics~\citep{conforti2014integer}.  
Recently RLCO (RL for Combinatorial Optimization) has proved to be an effective heuristic for hindsight planning~\citep{anthony2017thinking,fang2021universal}. Note that \cref{ass:planning_oracle} or RLCO cannot be used directly as a non-anticipatory policy for the Exo-MDP, since the whole sequence $\bfxi \sim \PXi$ is not observed upfront when making decisions. 
 We discuss several examples of hindsight planners in \cref{sec:planning_oracle_open_control} and assess the impact of approximate planning empirically in \cref{sec:nylon}. 

\section{Using Hindsight In Exo-MDPs: An Example}
\label{sec:example}

In an Exo-MDP, the only uncertainty is due to unknown $\PXi$. When $\PXi$ introduces substantial variability in outcomes, the natural question is whether generic MDP algorithms can learn effectively?  When $T = 1$, Exo-MDPs are isomorphic to a multi-armed bandit and so general bandit algorithms are optimal for the Exo-MDP.  However, we will see in the next example with $T > 1$ that the answer is in general no.  

Consider the sailing example in Figure~\ref{fig:example} which is an Exo-MDP where the decision is to pick between one of two routes {\em prior to observing wind} with hopes of minimizing the trip duration.  By direct calculation, $Q^\star(\text{route2}) - Q^\star(\text{route1}) = -48$.  Hence, the optimal non-anticipative policy will always pick route2.

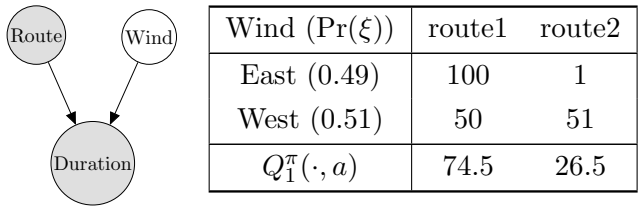
\begin{figure}[htb!]
\centering
\setcellgapes{3pt}
    \makegapedcells 
    \setlength\belowrulesep{0pt}   
    \setlength\aboverulesep{0pt}
\begin{tabular}{cc}
\scalebox{0.75}{\begin{tikzpicture}[baseline={(current bounding box.north)}]
 \node[obs] (Duration) {Duration};%
 \node[obs,above=of Duration,xshift=-1cm] (Route) {Route}; %
 \node[latent,above=of Duration,xshift=1cm] (Wind) {Wind}; %
 \edge {Route,Wind} {Duration}  
 \end{tikzpicture}}
 & 
        \begin{tabular}[t]{|c|cc|}
        \toprule
         Wind ($\Pr(\xi)$) & route1  & route2  \\
        \midrule
        East ($0.49$) & $100$  & $1$ \\
        West ($0.51$) & $50$ & $51$ \\
        \midrule
        $Q^\pi_1(\cdot, a)$ & $74.5$ & $26.5$ \\
        \bottomrule
        \end{tabular}

\end{tabular}
\vspace{-2mm}
\caption{An Exo-MDP for sailing in uncertain winds: $\bfxi =\{\text{Wind}\}, \A=\{ \text{Route} \}, r=\{\text{Duration}\}$ and $\X = \emptyset$. First, the agent picks a \emph{Route}. Then \emph{Wind} conditions are observed during the trip and the agent receives a cost with respect to \emph{Duration}.  Values in the table denote average trip duration $r(a)$ (accounting for random fluctuations in wind).} 
\vspace{-3mm}
\label{fig:example}
\end{figure}

\noindent {\bf RL:} If a classic RL approach was applied to this problem, it would estimate the average duration for each route using observed samples, include exploration bonuses to account for uncertainty, and compare the averages to pick future routes. 
This requires many $\bfxi$ samples because wind introduces large variance in the $Q$-estimates, and requires sufficient data to be collected across both the routes.

\noindent {\bf Hindsight Learning:} In hindsight, we can instead 
use {\em all }observed samples (leveraging known $f$ and $r$), with {\em no} exploration bonuses, and use paired comparisons to identify the optimal route. Variability due to wind means the routes' durations are typically positively correlated and thus a paired sample between routes will be more statistically efficient.  


\section{\GuidedRL}
\label{sec:hindsight_planning}

We introduce \GuidedRL (HL) to incorporate hindsight information in a principled manner, so as to reduce the exogenous variability and thereby speed-up learning.
 HL first uses a hindsight planner (\cref{ass:planning_oracle}) to derive a surrogate policy $\pip$:
 \ifdefined\arxiv
 \allowdisplaybreaks{
 \begin{align}
    \label{eq:qip}
    \pip_t(s) & \coloneqq \argmax_{a \in \A} \Qip_t(s,a); \\
    \Qip_t(s,a) & \coloneqq \E_{\bfxi_{\geq t}}[r(s,a,\xi_t) + \plan(t+1, \bfxi_{> t}, f(s, a, \xi_t))];  \label{eq:qip_value}\\
    \Vip_t(s) & \coloneqq \E_{\bfxi_{\geq t}}[\plan(t, \bfxi_{\geq t}, s)]. 
\end{align} }
\else
 \allowdisplaybreaks{
 \begin{align}
    \label{eq:qip}
    \pip_t(s) & \coloneqq \argmax_{a \in \A} \Qip_t(s,a), \\
    \Qip_t(s,a) & \coloneqq \E_{\bfxi_{\geq t}}[r(s,a,\xi_t) + \nonumber\\
    & \plan(t+1, \bfxi_{> t}, f(s, a, \xi_t))], \label{eq:qip_value}\\
    \Vip_t(s) & \coloneqq \E_{\bfxi_{\geq t}}[\plan(t, \bfxi_{\geq t}, s)]. 
\end{align} }
\fi

We define $\Qip_t(s,a,\bfxi_{\geq t})$ and $\Vip_t(s,\bfxi_{\geq t})$ as the terms inside of the respective expectations. 
Note that $\pip$ is a non-anticipatory policy, which considers expectation over future exogenous $\bfxi_{\geq t}$ rather than being defined for a fixed trace. 
$\pip$ is called ``Bayes Selector'' in the literature~\citep{vera2021bayesian,mercier2007performance} and has been used for applications like bin packing and refugee resettlement~\citep{bansak2022outcome,ahani2021dynamic,banerjee2020uniform}. Intuitively $\pip$ uses the returns accumulated by hindsight-optimal actions to score and rank good decisions, instead of mimicking the hindsight-optimal actions directly. 
However, it is always the case that $\Vip_t(s) \geq V^\star(s)$ and $\Qip_t(s,a) \geq Q^\star(s,a)$~\citep{chong2000framework}, and so $\pip$ can be a sub-optimal surrogate for $\pi^\star$.
In~\cref{sec:theory} we will bound its gap to $\pi^\star$ with a novel quantity called \emph{hindsight bias} and discuss the implications for learning.

\subsection{Imitating the ``Bayes Selector''}

Executing $\pip$ requires frequent calls to the hindsight planner online to evaluate $Q_t^\dag(S_t, a)$ on every observed state. Therefore running this \emph{tabular} policy (i.e., considering every state separately) can be prohibitively costly when policy execution must also satisfy latency constraints (e.g., in VM allocation).  Additionally, in resource allocation domains it is infeasible to enumerate all possible states.
We describe a family of algorithms in \cref{alg:training}
that offloads the hindsight planner invocations to an offline training phase and 
distills $\pip$ into a computationally feasible policy of neural networks that can extrapolate to unseen states.

\begin{algorithm}
\begin{algorithmic}[1]
  \STATE {\bfseries Input:} An empty buffer $\mathcal{B}$, simulator for $f$ and $r$, initial policy $\pi$, dataset $\mathcal{D}$, number of epochs $K$.
   \FOR{$k = 1,2,\dotsc,K$}
    \STATE Sample a trace $\bfxi$ from $\D$ 
    \STATE Sample trajectory $\{s_1 \dots s_T\}$ from $\pi$ using the $f, r, \bfxi$
    \STATE Label sampled states from the trajectory with $\{\Qip_t(s_t,a, \bfxi_{\geq t}) : a \in \A, t \in [T]\}$
    \STATE Aggregate the labeled data into the buffer $\mathcal{B}$
    \STATE Optimize $\pi$ on $\mathcal{B}$ either with Hindsight MAC (Equation~\ref{eq:guided_mac}) or Hindsight Q-Distillation (Equation~\ref{eq:guided_q}) 
   \ENDFOR
\end{algorithmic}
\captionof{algorithm}{\GuidedRL}
\label{alg:training}
\end{algorithm}

\srsedit{
We use online imitation learning (IL) with $\pip$ as the expert policy, specifically, the AggreVaTE algorithm~\citep{ross2014reinforcement} with expert $\Qip(s,a)$ values from the hindsight planner.  
By interleaving the trajectory sampling and policy updates we avoid querying $\Qip(s,a)$ (and hence solving hindsight planning problems) in uninformative states. 
We note that all of this interactive querying of the expert occurs during {\em offline} training, and hence the planning is not needed online during test time.
Since we allow $\bfxi$ to be arbitrarily correlated across $t$, in \cref{alg:training} we sample an entire trace $\bfxi^i$ from $\D$.  However, if $\bfxi$ is iid across time-steps $t$ we can enhance step 3 by resampling a single $\xi_t$ at each $t$.
}


Many popular RL algorithms~\citep{konda2000actor,van2016deep,schulman2017proximal} compute $Q$-values or advantages via Monte Carlo estimation. HL can be easily incorporated in them by replacing the Monte Carlo estimates (e.g., $Q^\pi(s,a)$) with Step 5 of \cref{alg:training} (i.e., $\Qip(s,a)$). 
We outline two such modifications below, one using a policy network and the other using a critic network. Since we can simulate $\Qip_t(s,a, \bfxi_{\geq t})$ for any action, we use the \emph{common random numbers} insight~\citep{ng2000pegasus} for variance reduction and sum across all actions in both instantiations. This additional sum over actions trick is not critical, and is used in our experiments only because the action spaces are relatively small.

\noindent {\bf Hindsight MAC:} We modify Mean Actor Critic~\citep{asadi2017mean} by incorporating $\Qip$ into differentiable imitation learning~\citep{sun2017deeply}. For the policy represented with a neural network $\pi_\theta$, consider the loss function:
\ifdefined\arxiv
\begin{equation} \label{eq:guided_mac}
 \ensuremath{\ell}(\pi_\theta) = \E_{\bfxi} [ \sum_{t=1}^T \E_{S_t \sim \Pr_t^{\pi}} \sum_{a \in \A} \pi_\theta(a \mid S_t) \Qip_t(S_t, a, \bfxi_{\geq t})].
\end{equation}
\else
\begin{equation} \label{eq:guided_mac}
\small \ensuremath{\ell}(\pi_\theta) = \E_{\bfxi} [ \sum_{t=1}^T \E_{S_t \sim \Pr_t^{\pi}} \sum_{a \in \A} \pi_\theta(a \mid S_t) \Qip_t(S_t, a, \bfxi_{\geq t})].
\end{equation}
\fi

\noindent {\bf Hindsight Q-Distillation:} We can represent $\Qip$ values directly using a neural network critic $Q^\theta$.   
The policy is defined implicitly w.r.t. $Q^\theta$ as $\pi_\theta = \argmax_{a \in \A} Q^\theta(s,a)$.  
The loss function $\ensuremath{\ell}(\pi_\theta)$ to fit $Q^\theta$ is:
\ifdefined\arxiv
\begin{equation}  \label{eq:guided_q}
 \ensuremath{\ell}(\pi_\theta) = \E_{\bfxi} [ \sum_{t=1}^T \E_{S_t \sim \Pr_t^{\pi}}[ \sum_{a \in \A} (Q_t^\theta(S_t,a) - \Qip_t(S_t, a, \bfxi_{\geq t}))^2 ]].
\end{equation}
\else
\begin{equation}  \label{eq:guided_q}
  \small  \ensuremath{\ell}(\pi_\theta) = \E_{\bfxi} [ \sum_{t=1}^T \E_{S_t \sim \Pr_t^{\pi}}[ \sum_{a \in \A} (Q_t^\theta(S_t,a) - \Qip_t(S_t, a, \bfxi_{\geq t}))^2 ]].
\end{equation} \fi
We optimize a sample approximation of Equation~\ref{eq:guided_mac} or~\ref{eq:guided_q} using the dataset $\D$.
The current policy defines the state sampling distribution, while $\Qip$ gives the long-term reward signal for each action.

\section{Theoretical Guarantees}
\label{sec:theory}

Compared with RL, HL uses the known dynamics $f$ and rewards $r$ of an Exo-MDP, the hindsight planner of \cref{ass:planning_oracle}, and the dataset $\D$ to trade-off a small asymptotic bias (which we define in \cref{eq:delta ip gap}) for a large reduction in the variance from exogenous inputs. 
To prove this, we first characterize the regret of $\pip$ in terms of a novel quantity, called the \emph{hindsight bias}, and show that it is negligible in many resource management problems. 
Next we show that HL is sample efficient, and can imitate the $\pip$ policy with faster optimization than RL. 
Finally, even if hindsight bias is large in an application, there are techniques to reduce it, including combinations of HL and RL in the future. 
\begin{definition}
The hindsight bias of $\pstar$ versus $\pip$ at time $t$ in state $s$ is defined as
\begin{align} \label{eq:delta ip gap}
\DeltaIPGap_t(s)  \coloneqq \hspace{1mm} &\Qip_t(s,\pip_t(s)) - Q^\star_t(s, \pip_t(s)) + \nonumber \\ &Q^\star_t(s, \pi^\star_t(s)) - \Qip_t(s, \pi^\star_t(s)).
\end{align}
\end{definition}
\srsedit{Consider the over-estimation error of the hindsight planner $\Omega_t(s,a) \coloneqq \Qip_t(s,a) - Q^\star_t(s, a)$ (referred to as {\em local loss} in \citet{mercier2007performance}).} \cref{eq:delta ip gap} subtracts the over-estimation of $\pi^\star(s)$ from the over-estimation of $\pip(s)$ and so the hindsight bias can be small even if the over-estimation is large; as an extreme example consider the case when the argmax of $\Qip$ and $Q^\star$ coincide (see~\cref{lem:app_eq_equiv_ip}). 
We show that $\DeltaIPGap_t(s)$ bounds the regret of $\pip$.
\begin{theorem}
\label{thm:ip_gap_regret}
$\Regret(\pip) \leq \sum_{t=1}^T \E_{S_t \sim \Pr_t^{\pip}}[\DeltaIPGap_t(S_t)],$ where $\Pr_t^{\pip}$ denotes the state distribution of $\pip$ at step $t$ induced by the exogenous process.
In particular, if $\DeltaIPGap_t(s) \leq \Delta$ for some constant $\Delta$ then we have: $\Regret(\pip) \leq \Delta \sum_{t=1}^T \E_{S_t \sim \Pr_t^{\pip}}[\Pr(\pip_t(S_t) \neq \pi^\star(S_t))].$
\end{theorem}
\srsedit{Regret bounds of this form appear in the prophet inequality literature~\citep{dutting2020prophet,vera2021online} however against a much stronger benchmark of the hindsight planner, where the regret is defined as $\Vip(s_1) - V_1^{\pi}(s_1)$. 
\citet{mercier2007performance} (Theorem 1) show that regret is bounded by the worst-case hindsight bias on the states visited by {\em any} decision policy.  However, there are examples (see~\cref{lem:mercier_example}) where their bound is large and one could incorrectly conclude that hindsight optimization should not be applied.  In contrast, \cref{thm:ip_gap_regret} is tighter, requiring that the hindsight bias be small only on states visited by $\pip$.}  

As corollaries of \cref{thm:ip_gap_regret}, 
the regret of $\pip$ is \emph{constant} for many resource management Exo-MDPs, such as stochastic online bin packing with i.i.d. arrivals. See~\cref{app:bin_packing_results} (\cref{thm:delta_ip_bin_packing}) for a formal statement of the result, and a discussion of related results from the literature. 

Finally, we show that the performance of the best policy produced by \cref{alg:training} will converge to that of $\pip$ under standard assumptions for online imitation learning~\citep{sun2017deeply,ross2011reduction,yan2021explaining}.  
We use the overline notation to denote quantities for an empirical MDP whose exogenous distribution $\PXi$ is replaced with the empirical one $\overline{\PXi} \sim \D$.

\begin{theorem} \label{th:weak theorem}
Let $\overline{\pi}^\dagger$ denote the hindsight planning surrogate policy for the empirical Exo-MDP w.r.t. $\mathcal{D}$. Assume $\overline{\pi}^\dagger \in \Pi$ and \cref{alg:training} achieves no-regret in the optimization problem of Equation~\ref{eq:guided_mac}. 
Let $\pi$ be the best policy from \cref{alg:training}.
Then, for any $\delta\in(0,1)$, with probability $1-\delta$,
\ifdefined\arxiv
\begin{align*}
    \Regret(\pi) &\leq 2 T \sqrt{\frac{2 \log(2|\Pi| / \delta)}{N}}+\sum_{t=1}^T \E_{S_t \sim \overline{\textrm{Pr}}_t^{\overline{\pi}^\dagger}}[\bar{\Delta}_t^{\dagger}(S_t)] + o(1),
\end{align*}
\else
\begin{align*}
    \Regret(\pi) &\leq 2 T \sqrt{\frac{2 \log(2|\Pi| / \delta)}{N}}+ \\ & \sum_{t=1}^T \E_{S_t \sim \overline{\textrm{Pr}}_t^{\overline{\pi}^\dagger}}[\bar{\Delta}_t^{\dagger}(S_t)] + o(1),
\end{align*}
\fi
for $\bar{\Delta}_t^{\dagger}$ the sample average of \eqref{eq:delta ip gap}, and $\overline{\textrm{Pr}}_t^{\overline{\pi}^\dagger}$ is the state probability of $\overline{\pi}^\dagger$ 
in the empirical MDP.
\end{theorem}

\srsedit{In \cref{sec:exogenous_value} we also derive sample complexity results for PTO and RL when applied to Exo-MDPs.  In PTO (see \cref{thm:pto_regret}) the guarantees scale quadratically in $T$ and depend on the complexity of $\PXi$ (which in the worst case can scale by $|\Xi|^T$ if the $\xi_t$ are strongly correlated across $t$).  In contrast, \cref{th:weak theorem} scales linearly in $T$ and is only affected by the randomness over the induced $\Qip$ values and not directly by the complexity of $\PXi$. 
Unlike guarantees for RL (see \cref{thm:erm_regret}), \cref{th:weak theorem} is not asymptotically consistent due to the hindsight bias.  However, RL methods have asymptotic consistency {\em only if} they converge to the optimal policy in the empirical MDP. This convergence is an idealized computation assumption that hides optimization issues when studying statistical guarantees and is incomparable to \cref{ass:planning_oracle} for the hindsight planner (for which we show several examples in \cref{sec:planning_oracle_open_control}).}

Although hindsight bias is small for many practical Exo-MDPs, this is not universally true as we now show. Since hindsight bias bounds the regret of $\pip$ (\cref{thm:ip_gap_regret}), Exo-MDPs with large regret must also have large hindsight bias.
\begin{theorem}
\label{lem:neg_instance}
There exists a set of Exo-MDPs such that $\Regret(\pip) \geq \Omega(T)$.
\end{theorem}
Hence, for an arbitrary Exo-MDP the hindsight bias needs to be properly controlled to successfully leverage hindsight planning~\citep{chong2000framework}. 
The information relaxation literature~\citep{brown2014information,el2020lookahead} subtracts a carefully chosen baseline $b(s,a,\bfxi)$ in special cases of \cref{eq:qip_value}; viewed through our results, this procedure essentially reduces the hindsight bias of the eventual $\pip$.  
Building on this technique, we anticipate future works to design HL variants that are more robust to hindsight bias.

\vspace{-2mm}
\section{Experiments}
\label{sec:experiments}

\ifdefined\arxiv
\begin{table*}[t]
\caption{Performance of heuristics, $\pip$, RL and HL algorithms on multi-secretary and ARM problems benchmarked against the optimal policy.  Values are $V^\pi$, the performance of the compared policy evaluated using the Bellman equations, and error bars computed via a standard normal approximation averaging over the randomly sampled dataset.  \srsedit{Since this is a tabular problem, \GuidedMAC and \GuidedQD are identical, so we report the performance of both as \GuidedMAC.}  \srsedit{Relative performance compared against $V^{\pi^\star}$ is shown in parenthesis}.
} \label{tab:multi_sec_performance}
\setlength\tabcolsep{0pt} 
\smallskip 
\begin{tabular*}{\columnwidth}{@{\extracolsep{\fill}}lcccccc}
\toprule
  Multi-Secretary & \multicolumn{2}{c}{$T = 5$} & \multicolumn{2}{c}{$T = 10$} & \multicolumn{2}{c}{$T = 100$} \\
\midrule
   $\pi^\star$ & $2.22$ & & $5.09$ & & $49.9$ & \\
  $\pip$ & $2.21$ & $(-0.5\%)$ & $4.95$ & $(-2.7\%)$	& $49.85$ & $(-0.2\%)$ \\ \midrule
 Greedy & $1.67$ & $(-24.8\%)$ & $3.81$ & $(-25.1\%)$ & $38.76$ & $(-22.4\%)$ \\
 Tabular Q-learning & $1.67 \pm 0.0032$ & $(-24.8\%)$ & $3.81 \pm 0.0037$ & $(-25.1\%)$ & $48.10 \pm 0.027$ & $(-3.7\%)$ \\
 Hindsight MAC & $\mathbf{2.17 \pm 0.0040}$ & $(-2.4\%)$ & $\mathbf{4.98 \pm 0.0035}$ & $ (-2.1\%)$ & $\mathbf{48.65 \pm 0.022}$ & $(-2.6\%)$ \\
 \hline
 \hline
 ARM & \multicolumn{2}{c}{$T = 5$} & \multicolumn{2}{c}{$T = 10$} & \multicolumn{2}{c}{$T = 100$} \\
\midrule
$\pi^\star$ & $1.89$ & & $3.72$ & & $39.03$ & \\
$\pip$ & $1.88$ & $(-0.3\%)$ & $3.61$ & $(-2.9\%)$ & $37.27$ & $(-4.5\%)$ \\ \midrule
Greedy & $1.39$ & $(-26.5\%)$ & $2.50$ & $(-32.9\%)$ & $31.54$ & $(-19.2\%)$ \\
Tabular Q-learning & $1.28 \pm 0.015$ & $(-32.2\%)$ & $2.75 \pm 0.064$ & $(-26.0\%)$ & $32.59 \pm 0.25$ & $(-16.5\%)$ \\
Hindsight MAC & $\mathbf{1.81 \pm 0.032}$ & $(-4.0\%)$ & $\mathbf{3.30 \pm 0.095}$ & $(-11.4\%)$ & $\mathbf{33.84 \pm 0.37}$ & $(-13.3\%)$ \\
\bottomrule
\end{tabular*}
\end{table*}
\else
\begin{table*}[!t]
\caption{Performance of heuristics, $\pip$, RL and HL algorithms on multi-secretary and ARM problems benchmarked against the optimal policy.  Values are $V^\pi$, the performance of the compared policy evaluated using the Bellman equations, and error bars computed via a standard normal approximation averaging over the randomly sampled dataset. \srsedit{Since these are tabular problems, \GuidedMAC and \GuidedQD are identical, so we report the performance of both as \GuidedMAC.} \srsedit{Relative performance compared against $V^{\pi^\star}$ is shown in parenthesis}.
} \label{tab:multi_sec_performance}
\setlength\tabcolsep{0pt} 
\smallskip 
\begin{tabular*}{2\columnwidth}{@{\extracolsep{\fill}}lcccccc}
\toprule
  Multi-Secretary & \multicolumn{2}{c}{$T = 5$} & \multicolumn{2}{c}{$T = 10$} & \multicolumn{2}{c}{$T = 100$} \\
\midrule
   $\pi^\star$ & $2.22$ & & $5.09$ & & $49.9$ & \\
  $\pip$ & $2.21$ & $(-0.5\%)$ & $4.95$ & $(-2.7\%)$	& $49.85$ & $(-0.2\%)$ \\ \midrule
 Greedy & $1.67$ & $(-24.8\%)$ & $3.81$ & $(-25.1\%)$ & $38.76$ & $(-22.4\%)$ \\
 Tabular Q-learning & $1.67 \pm 0.0032$ & $(-24.8\%)$ & $3.81 \pm 0.0037$ & $(-25.1\%)$ & $48.10 \pm 0.027$ & $(-3.7\%)$ \\
 Hindsight MAC & $\mathbf{2.17 \pm 0.0040}$ & $(-2.4\%)$ & $\mathbf{4.98 \pm 0.0035}$ & $ (-2.1\%)$ & $\mathbf{48.65 \pm 0.022}$ & $(-2.6\%)$ \\
 \hline
 \hline
 ARM & \multicolumn{2}{c}{$T = 5$} & \multicolumn{2}{c}{$T = 10$} & \multicolumn{2}{c}{$T = 100$} \\
\midrule
$\pi^\star$ & $1.89$ & & $3.72$ & & $39.03$ & \\
$\pip$ & $1.88$ & $(-0.3\%)$ & $3.61$ & $(-2.9\%)$ & $37.27$ & $(-4.5\%)$ \\ \midrule
Greedy & $1.39$ & $(-26.5\%)$ & $2.50$ & $(-32.9\%)$ & $31.54$ & $(-19.2\%)$ \\
Tabular Q-learning & $1.28 \pm 0.015$ & $(-32.2\%)$ & $2.75 \pm 0.064$ & $(-26.0\%)$ & $32.59 \pm 0.25$ & $(-16.5\%)$ \\
Hindsight MAC & $\mathbf{1.81 \pm 0.032}$ & $(-4.0\%)$ & $\mathbf{3.30 \pm 0.095}$ & $(-11.4\%)$ & $\mathbf{33.84 \pm 0.37}$ & $(-13.3\%)$ \\
\bottomrule
\end{tabular*}
\end{table*}
\fi
We evaluate \GuidedRL on three resource management domains with different characteristics \srsedit{(our code is available at \href{https://github.com/seanrsinclair/hindsight-learning}{https://github.com/seanrsinclair/hindsight-learning})}. First, \textbf{Multi-Secretary}, where the exogenous inputs are the arriving candidates' qualities and the hindsight bias is negligible (see Theorem 4.2 of \citet{banerjee2020constant}).  Next we consider \textbf{Airline Revenue Management} where the exogenous inputs are the current request's (resource demands, revenue) and the hindsight bias is small (see \cref{thm:delta_ip_bin_packing}).  
Lastly, we consider \textbf{VM Allocation} where the exogenous inputs are VM requests and the hindsight bias is unknown. 
In \cref{sec:input_driven_mdp_examples,sec:planning_oracle_open_control} we show explicit constructions of the Exo-MDP and hindsight planner for each domain. 
For the first two domains, we use traces drawn from benchmark distributions and evaluate $\pip$ using Monte-Carlo rollouts. For the VM allocation domain we use real-world historical traces extracted from a large public cloud provider.

\vspace{-2mm}
\subsection{Multi-Secretary Problems}

Multi-secretary is the generalization of the classic secretary problem \citep{buchbinder2009secretary}, where $T$ candidates arrive sequentially but only $B$ can be selected.  An arriving candidate at time $t$ has ability $r_t \in (0,1]$ drawn i.i.d.~from a finite set of $K$ levels of expertise.  At each round, if the decision-maker has remaining budget (i.e., has chosen less than $B$ candidates thus far), they can \emph{accept} a candidate and collect the reward $r_t$, or \emph{reject} the candidate. The goal is to maximize the expected cumulative reward.

When $T$ is large relative to $N$, we can expect historical traces to provide sufficient information about $\PXi$. 
Recent results in~\citet{banerjee2020constant} use the ``Bayes Selector'' with a single exogenous trace to derive a policy with constant regret for a sufficiently large $T$. This suggests that the hindsight bias is negligible in this regime. Our experiment setup is identical to~\citet{banerjee2020constant} and is included in the supplementary material. 
We use $T=\{5,10,100\}, B=\frac{3}{5}T, K=4$ and $N = 1$. 
The Greedy heuristic accepts the first $B$ candidates regardless of their quality.
ML methods use a single trace sampled from the non-stationary candidate arrival process, and use a policy that maps a $3$-dim state (the rounds and budget remaining, and the current candidate ability) to an \emph{accept} probability. 
For the hindsight planner, we use Equation~$2$ from~\citet{banerjee2020constant} which implements a linear program with $2K$ variables. The Bayes Selector $\pip$ solves the LP with the historical trace in every possible state, and is only feasible for problems with small LPs and state spaces. 
We evaluate each policy using dynamic programming with the true arrivals distribution.

In \cref{tab:multi_sec_performance} (Top) we see that the HL algorithm (Hindsight MAC) is competitive with the optimal policy (which depends on the unknown $\PXi$ distribution) using just a \emph{single exogenous trace}. RL (implemented via Tabular Q-learning) however is very sample inefficient; for small $T \le 10$ it performs no better than the sub-optimal Greedy heuristic. 

\subsection{Airline Revenue Management}

 Airline Revenue Management~\citep{littlewood1972forecasting} is a special case of the multi-dimensional Online Bin Packing (OBP) problem (OBP exhibits vanishing hindsight bias via \cref{thm:delta_ip_bin_packing}).  The agent has capacity $B_k$ for $K$ different resources.  At each round, the decision-maker observes a request $A_t \in \mathbb{R}_+^K$ (the consumed capacity in each resource dimension), alongside a revenue $f_t$.  The algorithm can either \emph{accept} the request (obtaining revenue $f_t$ and updating remaining capacity according to $A_t$), or \emph{reject} it (note that partial acceptance is not allowed). The goal of the decision-maker is to maximize the expected revenue.

We use \texttt{ORSuite}~\citep{archer2022orsuite} as an ARM simulator with fixed capacity, i.i.d.~request types and job distribution, using a setting
from~\citet{vera2021bayesian} which shows large regret for existing heuristics. 
We vary $T$ from $5$ to $100$,  
and compute $\pi^\star$ through dynamic programming. Both RL (Tabular Q-learning) and HL (Hindsight MAC) were trained on the same dataset of $N = 100$ traces.  

In \cref{tab:multi_sec_performance} (Bottom)
we see that HL outperforms RL but is not as good as the Bayes Selector $\pip$. Since the state space is much larger in ARM, HL has not sampled all the relevant states for imitating $\pip$ and so its performance suffers. Moreover, as $T$ increases the performance of RL again approaches HL, highlighting that HL strikes a better bias-variance trade-off to perform better with limited data.  

\subsection{VM Allocation}

Arguably, our experiments thus far have been advantageous to HL because the hindsight bias is known to be small.  
We next examine HL on a large-scale allocation problem: allocating virtual machines (VMs) to physical servers. In contrast to previous experiments, in this problem, the bias can be arbitrary, and enumerating all states or solving the Bellman equations is infeasible. From an algorithmic perspective, the allocation problem is a multi-dimensional variant of OBP with stochastic (not i.i.d.) arrivals {and departures} (hence, the bound on hindsight bias does not apply). \srsedit{Due to problem scale we cannot compute $\pi^\star$ exactly with dynamic programming, so we instead benchmark a policy's performance with respect to a {\bf BestFit} heuristic.}

In the VM allocation problem we have $K$ physical machines (PM), each with a fixed capacity limit for both CPU and memory.
VM requests arrive over time, each with an associated CPU, memory requirement, and a lifetime (or duration); the lifetime is in principle unknown to the provider, but it can be predicted \citep{cortez2017resource}. Accordingly, we study below two variants where lifetime information is either available (\cref{subsubsec:stylized}) or not (\cref{sec:nylon}). The decision-maker must assign a feasible PM for the VM or reject the request (incurring a large penalty).
A PM is considered active when one or more VMs are assigned to it. The objective is to
minimize the total time that the machines remain active, normalized by the time horizon $T$ (i.e., the average number of active PMs per unit time). This objective is critical for cloud efficiency; see~\citet{buchbinder2021online} for a longer discussion. 

\subsubsection{Stylized environment} \label{subsubsec:stylized}
To gain insights into the problem domain, we consider first a stylized setting where the VMs arrive at discrete time steps (in practice, time is continuous, and VMs may arrive at any point in time); furthermore, the VM lifetime is perfectly predicted upon arrival. 
To carry out the experiments, 
we use the MARO simulator~\citep{MARO_MSRA} with $K=80$ PMs and $T=288$ (reflecting one day period discretized into time steps, each of which represents $5$ minutes of actual time).
MARO replays the VM requests in the Azure Public Dataset~\citep{cortez2017resource}\footnote{This dataset contains a uniform sample of VM requests received in a real data center over a one month period in 2019.} as follows: 
all VM requests arriving within a time step (i.e., $5$ minutes of actual time) are buffered and instead arrive simultaneously at the next discrete time step. 
The first half of the resulting trace is used for training and the remaining trace for testing. To evaluate any policy, we sampled $50$ different one-day traces from the held-out portion and report the average value of the objective function. 
For the hindsight planner, we implemented the integer program of \cref{sec:vm_oracle} in Gurobi and solved its linear relaxation. 
We used a modified objective function (inverse packing density) which is linear in the decision variables for computational feasibility (see discussion in \cref{sec:experiments_appendix}).

\begin{table}[!tb]
\caption{Performance of heuristics, RL, and HL algorithms on VM allocation benchmarked against the Best Fit baseline. $\star$ indicate significant improvement and $\circ$ indicate significant decrease, over BestFit by Welch's $t$-test. 
} \label{tab:performance}
\setlength\tabcolsep{0pt} 

\smallskip 
\begin{tabular*}{\columnwidth}{@{\extracolsep{\fill}}lc}
\toprule
  Algorithm  & PMs Saved \\
\midrule
  Performance Upper Bound (Oracle) & $4.96^\star$ \\

 \midrule 
 Best Fit & $0.0$ \\
 Bin Packing & $-1.05^\circ$ \\
\midrule
 DQN & $-0.64$ \\
 MAC & $-0.51^\circ$ \\ 
 PPO & $-0.50$ \\
\midrule
 PG with Hindsight Baseline~\citep{mao_variance_2019} & $-0.057$ \\
 \textbf{\GuidedMAC} & $\mathbf{4.33^\star}$ \\
 \GuidedQD & $3.71^\star$ \\
\bottomrule
\end{tabular*}
\end{table}

We compare four allocation approaches.
(1) \textbf{Heuristics}: We consider several heuristics that have been widely used for different bin packing  problems (round robin, first fit, load balance, etc.). We report here the results for the best performing heuristic \emph{BestFit}, which has been widely applied in practice \citep{panigrahy2011heuristics}, in particular for VM allocation \citep{hadary_protean_2020}; in a nutshell BestFit chooses the machine which leaves less amount of unused resources (
see \citep{panigrahy2011heuristics} for details). (2) \textbf{RL}: We benchmark several popular RL algorithms including Double-DQN~\citep{van2016deep}, MAC~\citep{asadi2017mean} and PPO~\citep{schulman2017proximal}. (3) \textbf{Hindsight approaches:} We test Hindsight MAC (Equation~\ref{eq:guided_mac}) and Hindsight Q-Distillation (Equation~\ref{eq:guided_q}). In addition, we test ~\citet{mao_variance_2019} which uses hindsight-aware control variates to reduce the variance of policy gradient (PG) methods. (4) \textbf{Oracle:} We report the $\plan(1, \bfxi, s_1)$ (objective of the relaxed IP) evaluated on the test traces, and use the experiment outcome as a performance upper bound. 

All the ML methods use a $4$-layer neural net to map features describing a PM and the VM request to a score. 
In \cref{sec:experiments_appendix}, we detail the network design, state features and the hyper-parameter ranges we used.
\cref{tab:performance} reports the \emph{PMs Saved} which is the regret for the objective function relative to \emph{BestFit}, averaged across the evaluation traces. 
We created realistic starting state distributions by executing the \emph{BestFit} heuristic for a random duration (greater than one day).
Error bars are computed by $(i)$ training each algorithm over 20 random seeds (neural network parameters and the offline dataset) and $(ii)$ evaluating each algorithm on 50 one-day traces sampled from the hold-out set. We then compared its performance to Best Fit on each evaluation trace with a paired $t$-test of value $p = 0.05$. 
We observe that HL outperforms all the heuristics and RL methods, requiring $4$ fewer PMs on average (or a $5\%$ improvement in relative terms, since $K=80$).

\subsubsection{Real-World Resource Allocation}
\label{sec:nylon}

We now consider a more realistic setting where VM arrivals are in continuous time and the allocation agent has no information about VM lifetimes.  
\srsedit{In real-world settings, the scale of clusters can be much larger than the one considered in \cref{subsubsec:stylized}, see~\citet{hadary_protean_2020}; scaling ML algorithms to larger inventory sizes is an ongoing research direction. 
}
Furthermore, each VM arrival or departure is modeled as a step in the Exo-MDP, resulting in much larger time horizon $T$ (order of 100k). Our total trace period was $88$ days, and we used the exact methodology as in \cref{subsubsec:stylized} to obtain the training and test datasets. Due to the large scale, even the linear relaxation of the integer program was not tractable. Consequently, we carefully designed a \emph{hindsight heuristic} (\cref{alg:hindsight_heuristic}) to derive $\plan(t, \bfxi, s)$. The heuristic prioritizes VMs according to both their size and lifetime (see \cref{app:heuristic_performance}).

\ifdefined\arxiv
\begin{table}[!t]
\caption{Average number of PMs saved by RL and HL policies across $5$ clusters, calculated over $44$ days and benchmarked against the production {\bf BestFit} heuristic.  $\star$ indicate significant improvement and $\circ$ indicate a significant decrease, over BestFit by Welch's $t$-test.}
\label{tab:nylon_results}
\setlength\tabcolsep{0pt} 

\smallskip 
\begin{tabular*}{\columnwidth}{@{\extracolsep{\fill}}rccccc}
\toprule
  Cluster  & A & B & C & D & E \\
\midrule
  RL & $-0.14$ & $-0.35$ & $-0.27^\circ$ & $1.34^\star$ & $-0.37$ \\
  HL & ${\bf 3.20^\star}$ & ${\bf 1.35}$ & ${\bf 1.13^\star}$ & ${\bf 2.27^\star}$ & ${\bf 0.02}$ \\
\bottomrule
\end{tabular*}
\end{table}
\else
\begin{table}[!htb]
\caption{Average number of PMs saved by RL and HL policies across $5$ clusters, calculated over $44$ days and benchmarked against the production {\bf BestFit} heuristic.  $\star$ indicate significant improvement and $\circ$ indicate a significant decrease, over BestFit by Welch's $t$-test.}
\label{tab:nylon_results}
\setlength\tabcolsep{0pt} 

\smallskip 
\begin{tabular*}{\columnwidth}{@{\extracolsep{\fill}}rccccc}
\toprule
  Cluster  & A & B & C & D & E \\
\midrule
  RL & $-0.14$ & $-0.35$ & $-0.27^\circ$ & $1.34^\star$ & $-0.37$ \\
  HL & ${\bf 3.20^\star}$ & ${\bf 1.35}$ & ${\bf 1.13^\star}$ & ${\bf 2.27^\star}$ & ${\bf 0.02}$ \\
\bottomrule
\end{tabular*}
\end{table}
\fi


We adapt \GuidedMAC (HL) and compare it with MAC~\cite{asadi2017mean} (RL), where both used the same network architecture, which embeds VM-specific and PM-specific features using a $6$-layer GNN. The resulting architecture is rich enough to represent the {\bf BestFit} heuristic, but can also express more flexible policies. 
The Bayes Selector $\pip$ is infeasible to run within the latency requirements for VM allocation, and so is not compared. 
    
 \cref{tab:nylon_results} summarizes the results over five different clusters. Unlike \cref{subsubsec:stylized} where we sampled many $1$-day periods from the test trace, the demands on the real clusters were non-stationary throughout the test period. Hence we report results on the entire $44$-day test trace. We trained each algorithm over $3$ random seeds and evaluated $5$ rollouts to capture the variation in the cluster state at the start of the evaluation trace. Unlike the other experiments, we cannot account for the randomness in exogenous samples because we only have one evaluation trace for each cluster.  Error metrics are computed with a paired $t$-test of value $p = 0.05$.
 
 We observe that RL exhibits unreliable performance: in fact, it is sometimes worse than {\bf BestFit}, intuitively this can happen because it overfits to the request patterns seen during training. In contrast, HL always improved over {\bf BestFit}, with relative improvements of $0.1\%-1.6\%$ over RL and $0.1\%-0.7\%$ over BestFit (note cluster sizes are much larger here than \cref{subsubsec:stylized}).
As noted earlier, any percent-point (or even fractions of a percent) improvement implies millions of dollars in savings. The relative gains obtained here are more modest than in the stylized setting due to a combination of reasons. First, intuitively, every packing ``mistake" is more costly in a smaller cluster, meaning that algorithms have more room to shine in smaller-scale problems. Second, using a heuristic for hindsight learning is inherently sub-optimal. Lastly, we have not used \emph{any} information about VM lifetime; an interesting direction for future work is incorporating lifetime predictions to HL.

\section{Conclusion}
\label{sec:conclusion}

\srsedit{In this paper, we introduced \GuidedRL (HL) as a family of algorithms that solve a subclass of MDPs with exogenous inputs, termed Exo-MDPs. Exo-MDPs capture a variety of important resource management problems, such as VM allocation. We show that the HL algorithms outperform both heuristics and RL methods. 
One direction for future work is to blend RL with HL using reward shaping~\cite{cheng2021heuristic} for solving Exo-MDPs with large hindsight bias. Intuitively, combining pessimistic value estimates from RL with optimistic estimates from HL can provide finer-grained control for trading-off hindsight bias and variance from exogenous inputs. Another direction is designing hindsight learning algorithms in ``nearly Exo-MDP'' environments where the action can have a limited impact on the exogenous variables, such as using recent results from~\citet{liu2021exploiting}.
}

\section*{Acknowledgements}

We thank Janardhan Kulkarni, Beibin Li, Connor Lawless, Siddhartha Banerjee, and Christina Yu for inspiring discussions. We thank Dhivya Eswaran, Tara Safavi and Tobias Schnabel for reviewing early drafts. Part of this work was done while Sean Sinclair and Jingling Li were research interns at Microsoft Research, and while Sean Sinclair was a visitor at Simons Institute for the semester on the Theory of Reinforcement Learning and Data-Driven Decision Processes program.  We gratefully acknowledge funding from the National Science Foundation under grants ECCS-1847393, DMS-1839346, CCF-1948256, CNS-195599, and CNS-1955997, the Air Force Office of Scientific Research under grant FA9550-23-1-0068, and the Army Research Laboratory under grants W911NF-19-1-0217 and W911NF-17-1-0094.

\bibliography{references}
\bibliographystyle{icml/icml2023}

    \newpage
    \appendix

\section{Table of Notation}
\label{sec:notation_table}

\begin{table*}[!h]
\begin{tabular}{p{0.25\linewidth}|p{0.73\linewidth}}
\textbf{Symbol} & \textbf{Definition} \\ \hline
\multicolumn{2}{c}{Problem Setting Specification}\\
\hline
$\mathcal{S}, \A, T, s_1, R, P$ & MDP primitives: state and action space, time horizon, starting state, \\& reward function and transition probabilities \\
$\X$ & Endogenous space for the system \\
$\Xi$ & Exogenous input space \\
$\PXi$ & Distribution over exogenous inputs \\
$f(s, a, \xi), r(s,a,\xi)$ & Underlying deterministic transition and reward as function of exogenous input \\
$s_t$ & MDP state space primitive, $(x_t, \bfxi_{<t})$ for shorthand \\
$s_t, a_t, \xi_t$ & State, action, and exogenous input for time step $t$ \\
$\bfxi$ & An exogenous input trace $(\xi_1, \ldots, \xi_T)$ \\
$\bfxi_{\geq t}$ & Component of an exogenous input trace $(\xi_{t}, \ldots, \xi_T)$ \\
$\Pi$ & Set of all admissible policies\\
$Q_t^\pi(s,a), V_t^{\pi}(s)$ & $Q$-Function and value function for policy $\pi$ at time step $t$ \\
$\pi^\star, Q_t^\star(s,a), V_t^\star(s)$ & Optimal policy and the $Q$ and value function for the optimal policy \\
$\Pr_t^\pi$ & State-visitation distribution for policy $\pi$ at time step $t$\\
$\D$ & Dataset containing $N$ traces of exogenous inputs $\{\bfxi^1, \ldots, \bfxi^N\}$ \\
$\plan(t,s,\bfxi_{\geq t}, f)$ & Hindsight optimal cumulative reward $r$ starting in state $s$ at time $t$ \\
\, & with exogenous inputs dictated by $\bfxi_{\geq t}$ and dynamics $f$ (see \cref{eq:plan_assumption}) \\
$Q^\pi_t(s,a,\bfxi_{\geq t}), V^\pi_t(s,a,\bfxi_{\geq t})$ & $Q$ and $V$ functions for policy $\pi$ starting from $s$\\
& where exogenous inputs are given by $\bfxi_{\geq t}$ (see \cref{thm:exog_bellman}) \\
$\BarExp{\cdot}$ & Empirical expectation taken where $\bfxi$ is sampled uniformly from $\D$ \\
$\perm$ & Policy obtained by ERM, i.e. solving $\argmax_{\pi \in \Pi} \BarExp{V_1^\pi(s_1, \bfxi)}$. \\
$\overline{\PXi}, \overline{Q}_t, \overline{V}_t$, $\overline{\pi}$ & Estimated exogenous distribution using $\D$, $Q_t^\star$, $V_t^\star$ estimates \\
& using the estimated distribution, and the resulting policy \\
\hline
\multicolumn{2}{c}{Hindsight Planner}\\
\hline
$\Qip_t(s,a,\bfxi_{\geq t})$ & $r(s,a,\xi_t) + \plan(t+1, \bfxi_{> t}, f(s,a,\xi_t))$ \\
$\Vip_t(s,\bfxi_{\geq t})$ & $\plan(t, \bfxi_{\geq t}, s)$ \\
$\Qip_t(s,a), \Vip_t(s,a)$ & Expectations of $\Qip_t(s,a,\bfxi_{\geq t})$ and $\Vip_t(s, \bfxi_{\geq t})$ over $\bfxi_{\geq t}$ \\
$\pip$ & Greedy policy with respect to $\Qip$ \\
$\DeltaIPGap_t(s)$ & Hindsight bias for state $s$ at time step $t$ (see \cref{eq:delta ip gap})\\
$\Delta$ & Absolute bound on $\DeltaIPGap_t(s)$ \\
$\overline{\PXi}$ & Empirical distribution over $\bfxi$ from $\D$ \\
$\overline{\pip}$ & Greedy policy with respect to $\overline{\Qip}$ where true expectation\\
& over $\PXi$ replaced with $\overline{\PXi}$ \\
$\bar{\Delta}_t^{\dagger}(s)$ & Value of $\DeltaIPGap_t(s)$ where expectation over $\PXi$ replaced with $\overline{\PXi}$\\
$\overline{\textrm{Pr}}_t^{\pi}$ & State visitation distribution of $\pi$ at time step $t$ with exogenous dynamics $\overline{\PXi}$ \\
\hline
\end{tabular}
\caption{List of common notation}
\label{table:notation}
\end{table*}

\section{Detailed Related Work}
\label{sec:related_work_full}

There is an extensive literature on reinforcement learning and its connection to tasks in operations management; below, we highlight the work which is closest to ours, but for more extensive references, see \citet{sutton_reinforcement_2018, agarwal_reinforcement_nodate,powell_reinforcement_nodate} for RL, and \citet{bubeck_regret_2012, slivkins_introduction_2019} for background on multi-armed bandits.

\noindent \textbf{Information Relaxation for MDP Control}: Information relaxation as an approach for calculating performance bounds on the optimal $Q^\star$ function has been developed recently using rich connections to convex duality~\citep{vera2021bayesian,brown2021information,balseiro2019approximations,brown2017information,kallus2021stateful,mercier2007performance}.  As discussed in the main paper, for general problems, using hindsight planning oracles as in \cref{ass:planning_oracle} creates estimates for the $Q^\star$ value which are overly optimistic of their true value.  These differences can be rectified by introducing a control variate, coined \emph{information penalties}, to penalize the planner's access to future information that a truly non-anticipatory policy would not have.  The goal is to construct penalties which ensure that the estimates of $Q^\star$ are truly consistent for the underlying value.  This work has been developed explicitly in the context of infinite horizon MDPs \citep{brown2017information} where constructions are given for penalty functions as a function of the future randomness of the $\bfxi$ process.  Moreover, concrete algorithmic implementations using hindsight planners and information penalties has been developed in the tabular setting with no finite sample guarantees~\citep{el2020lookahead}.  Constructing these penalties in practice using suitable functions for arbitrary $\bfxi$ is unknown.   Our work differs by foregoing consistency of the estimates to instead focus on showing that in problem domains of interest, the policy which is greedy with respect to the hindsight planner is indeed consistent.  

\srsedit{\noindent {\bf Behaviour Cloning}: One approach for using hindsight information is behavior cloning.  This will compute the hindsight-optimal actions for every $\bfxi \sim \D$, and learn to imitate these actions using a feasible non-anticipatory policy~\citep{fang2021universal}. This is an instance of the \emph{probability matching} principle which is widely used in Thompson sampling~\citep{russo2018tutorial,hart2000pattern}. 
Unfortunately, this \emph{value-agnostic} approach is uncontrollably biased.  Consider the example in \cref{sec:example}.  Since the winds in $\D$ will be west $51\%$ of the time, the hindsight-optimal distribution (marginalizing over wind) is $\Pr(\text{route1})=0.51; \Pr(\text{route2})=0.49$. A non-anticipatory learner policy trained via behavior cloning will converge either to this distribution (if learning a stochastic policy) or its mode (using a deterministic policy). Both these policies are very sub-optimal compared to the optimal policy.
}

\noindent \textbf{Policy Based Methods with Control Variates}: Recent work has developed black box tools to modify policy gradient algorithms with control variates that depend on the exogenous trace.  Recall that a typical policy-based algorithm uses either on-policy data (or off-policy with re-weighted importance sampling strategies), and estimates the gradient in the return via
$$\nabla V^{\pi_\theta} = \E_{S \sim \Pr_t^{\pi_\theta}, A \sim \pi_\theta}[\nabla \log \pi_\theta(A \mid S) \hat{Q}^{\pi_\theta}(S,A)].$$
From here, most methods subtract an appropriate baseline (commonly taken to be an estimate of the value function) as a form of Rao-Blackwellization to reduce the variance of the estimator while incurring no additional bias.  In particular, for any function $b : \mathcal{S} \rightarrow \mathbb{R}$ we can instead take
$$\nabla V^{\pi_\theta} = \E_{S \sim \Pr_t^{\pi_\theta}, A \sim \pi_\theta}[\nabla \log \pi_\theta(A \mid S) (Q_{\pi_\theta}(S,A) - b(S)) ]$$ while remaining unbiased.
However, due to the exogenous input structure on the MDP any function $b : \X \times \Xi^{T} \rightarrow \mathbb{R}$ also results in an unbiased gradient.  Through this, the existing literature has taken different approaches for constructing these \emph{input driven baselines}.  In \citet{mao_variance_2019} they consider directly using a baseline of the form $b(x, \bfxi)$.  As an architecture to learn a network representation of this baseline the authors propose either using a multi-value network or meta learning.
In \citet{mesnard_counterfactual_2021} they consider using future conditional value estimates for the policy gradient baseline.  In particular, they use $\Psi_t$ as a new statistic to calculate new information from the rest of the trajectory and learn value functions which are conditioned on the additional hindsight information contained in $\Psi_t$.  They provide a family of estimators, but do not specify which form of $\Psi_t$ to use in generating an algorithm.

\noindent \textbf{Recurrent Neural Network Policy Design}: A related line of work modifies black box policy gradient methods by using a recurrent neural network (RNN) explicitly in policy design.  In \citet{venuto_policy_2021} they augment the state space to include $\bfxi$ while simultaneously limiting information flow in the neural network to ensure that the algorithm is not overly relying on this privileged information.  This approach, named {\em policy gradients incorporating the future}, is easy to implement as it just augments the network using an LSTM and adds a new loss term to account for the information bottleneck.

\noindent \textbf{Learning to Search}: The Exo-MDP model is closely related to the learning-to-search model.  Expert iteration~\citep{anthony2017thinking} separates planning and generalization when learning to search, and provides an alternative approach to implement the $\plan$ oracle.  Retrospective imitation~\citep{song2018learning} faces a similar challenge as us: a given $\bfxi$ defines a fixed search space and we seek search policies that generalize across $\PXi(\bfxi)$. However, retrospective imitation reduces to realizable imitation problems because the learner witnesses $\bfxi$ beforehand whereas in Exo-MDPs, $\bfxi_{\ge t}$ is privileged information and imitating $\plan$ is typically unrealizable. Asymmetric Imitation Learning~\citep{warrington2021robust} studies problems when imitating an expert with privileged information but essentially use RNN policies to ameliorate unrealizability. 

\noindent \textbf{RL for OR}: In our work we primarily consider simulations on dynamic Virtual Machine (VM) scheduling.  On the theoretical side, variants of greedy algorithms have been usually proposed to solve the dynamic VM scheduling problems with competitive ratio analysis.  In \citet{stolyar2013infinite} they assume the VM creation requests can be modeled as a Poisson process with lifetimes as an exponential distribution and show that the greedy algorithm achieves the asymptotically optimal policy.  In \citet{li2015dynamic} they develop a hybrid \textsc{FirstFit} algorithm with an improvement on the competitive ratio.  On the more practical side using deep reinforcement learning techniques, in \citet{mao2016resource} they develop a DeepRM system which can pack tasks with multiple resource demands via a policy gradient method.  They also built a job scheduler named DECIMA by modifying actor critic algorithms with input driven baselines~\citep{mao2019learning}.  In \citet{zhang2020learning} they solved the heterogeneous scheduling problem with deep $Q$ learning.  Lastly, in \citet{sheng2022learning} they developed SchedRL, a modification of Deep Q Learning with reward shaping to develop a VM scheduling policy.  All of these algorithms modify existing RL algorithms and show empirical gains on variations of the VM scheduling problem.  Our work differs from two perspectives: 1) we consider using hindsight planning explicitly during training time, 2) our algorithms can be applied to any Exo-MDP problems.  

We also note that existing deep reinforcement learning has also been applied in other systems applications (without exploiting their exo-MDP structure) including ride-sharing systems \citep{feng2021scalable}, stochastic queueing networks~\citep{dai2021queueing}, power grid systems~\citep{chen2020powernet}, jitter buffers~\citep{fang2019reinforcement}, and inventory control~\citep{harsha2021math}.

\noindent \textbf{RL for Combinatorial Optimization}: A crucial assumption underpinning our algorithmic framework is the implementation of {hindsight planners} as in \cref{ass:planning_oracle}.  Our framework is well motivated for problems where hindsight planning is efficient, building on the existing optimization literature on solving planning problems~\citep{conforti2014integer,bertsimas1997introduction}.  However, in general these problems as a function of a fixed exogenous input trace can be written as combinatorial optimization problems.  While we consider using hindsight planners for RL problems, a dual lens is using machine learning techniques for combinatorial optimization, as has been explored in recent years \citep{vinyals2015pointer,bello2016neural,chitnis2020learning}.  In particular, in \citet{vinyals2015pointer} they designed a new network architecture and trained using supervised learning for traveling salesman problems.  Similarly in \citet{hu2017solving} they solve variants of online bin packing problems using policy gradient algorithms.  In \citet{tang2020reinforcement} they design novel heuristic branch and bound algorithms using machine learning for integer programming optimization.

\noindent \textbf{Exo-MDPs}: Exo-MDPs, as highlighted in \cref{sec:preliminary} were described in \citet{powell_reinforcement_nodate}.  They characterize sequential decision making problems as an evolution of \emph{information}, \emph{decision}, \emph{information} sequence represented mathematically as the sequence $(s_1, a_1, \xi_1, s_2, a_2, \xi_2, \ldots, s_T)$.  Here, the state variable $s_t$ is written explicitly to capture the information available to the decision maker to make a decision $a_t$, followed by the information we learn after making a decision, i.e. the exogenous information $\xi_t$.  Similar models have been outlined in \citep{mao_variance_2019,dietterich_discovering_2018,efroni2022sample}.

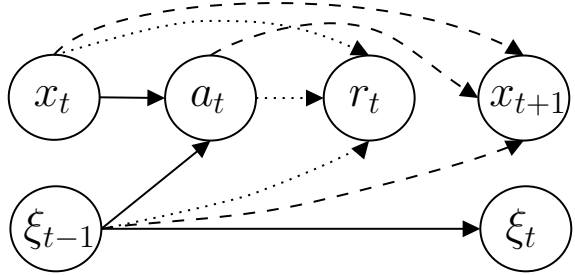
\begin{figure}[!t]
\begin{center}
\tikzset{every picture/.style={line width=0.75pt}} 

\tikzset{every picture/.style={line width=0.75pt}} 

\begin{tikzpicture}[x=0.75pt,y=0.75pt,yscale=-1,xscale=1]

\draw   (4,60.57) .. controls (4,48.72) and (14.26,39.11) .. (26.93,39.11) .. controls (39.59,39.11) and (49.85,48.72) .. (49.85,60.57) .. controls (49.85,72.43) and (39.59,82.04) .. (26.93,82.04) .. controls (14.26,82.04) and (4,72.43) .. (4,60.57) -- cycle ;
\draw   (4.86,126.94) .. controls (4.86,115.09) and (15.13,105.48) .. (27.79,105.48) .. controls (40.45,105.48) and (50.72,115.09) .. (50.72,126.94) .. controls (50.72,138.79) and (40.45,148.4) .. (27.79,148.4) .. controls (15.13,148.4) and (4.86,138.79) .. (4.86,126.94) -- cycle ;
\draw   (82.86,60.94) .. controls (82.86,49.09) and (93.13,39.48) .. (105.79,39.48) .. controls (118.45,39.48) and (128.72,49.09) .. (128.72,60.94) .. controls (128.72,72.79) and (118.45,82.4) .. (105.79,82.4) .. controls (93.13,82.4) and (82.86,72.79) .. (82.86,60.94) -- cycle ;
\draw   (162.86,60.94) .. controls (162.86,49.09) and (173.13,39.48) .. (185.79,39.48) .. controls (198.45,39.48) and (208.72,49.09) .. (208.72,60.94) .. controls (208.72,72.79) and (198.45,82.4) .. (185.79,82.4) .. controls (173.13,82.4) and (162.86,72.79) .. (162.86,60.94) -- cycle ;
\draw   (240.86,60.94) .. controls (240.86,49.09) and (251.13,39.48) .. (263.79,39.48) .. controls (276.45,39.48) and (286.72,49.09) .. (286.72,60.94) .. controls (286.72,72.79) and (276.45,82.4) .. (263.79,82.4) .. controls (251.13,82.4) and (240.86,72.79) .. (240.86,60.94) -- cycle ;
\draw   (241.86,126.94) .. controls (241.86,115.09) and (252.13,105.48) .. (264.79,105.48) .. controls (277.45,105.48) and (287.72,115.09) .. (287.72,126.94) .. controls (287.72,138.79) and (277.45,148.4) .. (264.79,148.4) .. controls (252.13,148.4) and (241.86,138.79) .. (241.86,126.94) -- cycle ;
\draw    (50.72,126.94) -- (238.86,126.94) ;
\draw [shift={(241.86,126.94)}, rotate = 180] [fill={rgb, 255:red, 0; green, 0; blue, 0 }  ][line width=0.08]  [draw opacity=0] (8.93,-4.29) -- (0,0) -- (8.93,4.29) -- cycle    ;
\draw    (49.85,60.57) -- (79.86,60.91) ;
\draw [shift={(82.86,60.94)}, rotate = 180.64] [fill={rgb, 255:red, 0; green, 0; blue, 0 }  ][line width=0.08]  [draw opacity=0] (8.93,-4.29) -- (0,0) -- (8.93,4.29) -- cycle    ;
\draw  [dash pattern={on 0.84pt off 2.51pt}]  (128.72,60.94) -- (159.86,60.94) ;
\draw [shift={(162.86,60.94)}, rotate = 180] [fill={rgb, 255:red, 0; green, 0; blue, 0 }  ][line width=0.08]  [draw opacity=0] (8.93,-4.29) -- (0,0) -- (8.93,4.29) -- cycle    ;
\draw  [dash pattern={on 0.84pt off 2.51pt}]  (26.93,39.11) .. controls (40.69,35.93) and (81.12,22.03) .. (104.12,22.03) .. controls (126.66,22.03) and (148.24,21.07) .. (183.61,38.39) ;
\draw [shift={(185.79,39.48)}, rotate = 206.7] [fill={rgb, 255:red, 0; green, 0; blue, 0 }  ][line width=0.08]  [draw opacity=0] (8.93,-4.29) -- (0,0) -- (8.93,4.29) -- cycle    ;
\draw  [dash pattern={on 4.5pt off 4.5pt}]  (26.93,39.11) .. controls (39.81,22.83) and (87.58,12.92) .. (139.12,13) .. controls (183.86,13.06) and (231.44,20.65) .. (261.51,38.11) ;
\draw [shift={(263.79,39.48)}, rotate = 211.65] [fill={rgb, 255:red, 0; green, 0; blue, 0 }  ][line width=0.08]  [draw opacity=0] (8.93,-4.29) -- (0,0) -- (8.93,4.29) -- cycle    ;
\draw    (50.72,126.94) -- (103.46,84.29) ;
\draw [shift={(105.79,82.4)}, rotate = 141.04] [fill={rgb, 255:red, 0; green, 0; blue, 0 }  ][line width=0.08]  [draw opacity=0] (8.93,-4.29) -- (0,0) -- (8.93,4.29) -- cycle    ;
\draw  [dash pattern={on 4.5pt off 4.5pt}]  (105.79,39.48) .. controls (118.12,31.03) and (153.12,20.03) .. (178.12,24.03) .. controls (202.49,27.93) and (210.55,41.79) .. (238.66,59.57) ;
\draw [shift={(240.86,60.94)}, rotate = 211.65] [fill={rgb, 255:red, 0; green, 0; blue, 0 }  ][line width=0.08]  [draw opacity=0] (8.93,-4.29) -- (0,0) -- (8.93,4.29) -- cycle    ;
\draw  [dash pattern={on 0.84pt off 2.51pt}]  (50.72,126.94) .. controls (72.12,123.03) and (116,113) .. (125,110) .. controls (133.24,107.26) and (172.55,91.98) .. (183.51,84.32) ;
\draw [shift={(185.79,82.4)}, rotate = 130.56] [fill={rgb, 255:red, 0; green, 0; blue, 0 }  ][line width=0.08]  [draw opacity=0] (8.93,-4.29) -- (0,0) -- (8.93,4.29) -- cycle    ;
\draw  [dash pattern={on 4.5pt off 4.5pt}]  (50.72,126.94) .. controls (72.12,123.03) and (117.12,123.03) .. (159.12,113.03) .. controls (199.23,103.48) and (242.07,92.11) .. (261.24,83.59) ;
\draw [shift={(263.79,82.4)}, rotate = 153.98] [fill={rgb, 255:red, 0; green, 0; blue, 0 }  ][line width=0.08]  [draw opacity=0] (8.93,-4.29) -- (0,0) -- (8.93,4.29) -- cycle    ;

\draw (15.14,53.63) node [anchor=north west][inner sep=0.75pt]  [font=\LARGE] [align=left] {$\displaystyle x_{t}$};
\draw (10,115) node [anchor=north west][inner sep=0.75pt]  [font=\LARGE] [align=left] {$\displaystyle \xi _{t-1}$};
\draw (94,53.63) node [anchor=north west][inner sep=0.75pt]  [font=\LARGE] [align=left] {$\displaystyle a_{t}$};
\draw (174,53.63) node [anchor=north west][inner sep=0.75pt]  [font=\LARGE] [align=left] {$\displaystyle r_{t}$};
\draw (245,53.63) node [anchor=north west][inner sep=0.75pt]  [font=\LARGE] [align=left] {$\displaystyle x_{t+1}$};
\draw (253,115) node [anchor=north west][inner sep=0.75pt]  [font=\LARGE] [align=left] {$\displaystyle \xi _{t}$};

\end{tikzpicture}
\end{center}
\caption{Causal diagram for an Exo-MDP where $k = 1$.  Here the dotted line indicates the influence of $(x_t, a_t, \xi_t)$ on the immediate reward $r_t$ via $r(x_t, a_t, \xi_t)$ and the dashed line on the transition evolution as $x_{t+1} = f(s_t, a_t, \xi_t)$.  The key facet to notice is the lack of influence on the $\xi$ process from the current endogenous state $x_t$ and action $a_t$.}
\label{fig:causal_diagram}
\end{figure}

\section{MDPs with Exogenous Inputs}
\label{sec:gen_examples_mdp}

In this section we further discuss the definition of Exo-MDPs and its relations to contextual bandits and MDPs and highlight some examples in the operations management literature.

\subsection{Generality of MDPs with Exogenous Inputs}
\label{sec:relation_input_driven_MDPs}

As highlighted in \cref{sec:preliminary} we consider the finite horizon reinforcement learning setting where an agent is interacting with a Markov Decision Process (MDP)~\citep{puterman2014markov}.  The underlying MDP is given by a five-tuple $(\mathcal{S}, \A, T, p, R, s_1)$ where $T$ is the horizon, $(\mathcal{S}, \A)$ denotes the set of states and actions, $R$ is the reward distribution, $p$ the distribution governing the transitions of the system, and $s_1$ is the given starting state.  

\begin{definition}
In an \textbf{MDP with Exogenous Inputs} (Exo-MDP) we let $\bfxi = (\xi_1, \ldots, \xi_T)$ be a trace of exogenous inputs with each $\xi_t$ supported on the set $\Xi$.  We assume that $\bfxi$ is sampled according to an unknown distribution $\PXi$.  The agent has access to an endogenous or system state $x \in \X$.  With this, the dynamics and rewards of the Markov decision process evolve where at time $t$, the agent selects their action $a_t \in \A$ based solely on $s_t = (x_t, \bfxi_{< t})$.  After, the endogenous state evolves according to $x_{t+1} = f(s_t, a_t, \xi_{t})$, and the reward earned is $r(s_t, a_t, \xi_{t})$, and $\xi_t$ is observed.  We assume that $f$ and $r$ are known by the principal in advance.  
\end{definition}

Note that this imposes that the state space for the underlying MDP can be written as $\mathcal{S} = \X \times \Xi^{T}$ where the first component corresponds to the endogenous state and the second to the exogenous input trace observed so far.  We use the shorthand $s_t$ to refer to $(x_t, \bfxi_{<t})$.

As written, the distribution $\PXi$ can be arbitrarily correlated across time.  We can relax this setting to assume that $\bfxi$ evolves according to a $k$-Markov chain.  More formally, that at each step $t$, $\xi_t \mid (\xi_{t-k}, \xi_{t-k+1}, \ldots, \xi_{t-1})$ is conditionally independent of $(\xi_1, \ldots, \xi_{t-k-1})$.  This allows the state space to be represented as $\mathcal{S} = \X \times \Xi^{k}$.  
Lastly, the dataset $\D$ contains a series of $N$ traces sampled independently according to $\PXi$ as $\D = \{\bfxi^1, \ldots, \bfxi^N\}$ where each $\bfxi^i = \{\xi_1^i, \ldots, \xi_T^i\}$.

For more intuition, consider the model under various values of $k$:
\begin{itemize}
    \item \textbf{Case $k = T$}: Here we assume that $\bfxi$ is an arbitrarily correlated process and $\mathcal{S} = \X \times \Xi^{T}$ so that $s_t = (x_t, \bfxi_{< t})$.  An example of this is VM allocation, where exogenous VM requests are highly correlated across time~\citep{hadary_protean_2020}.
    \item \textbf{Case $k = 1$}: Here we assume that $\bfxi$ process evolves according to a $1$-Markov chain.  The state space factorizes as $\X \times \Xi$ where $\X$ is the endogenous space and $\Xi$ is the exogenous space.  The current state is $s_t = (x_t, \xi_{t-1})$, and the state updates to $(f(s_t, a_t, \xi_t), \xi_t)$ where $\xi_t$ is drawn from the conditional distribution given $\xi_{t-1}$.  A representation of the causal diagram under this setting is in \cref{fig:causal_diagram}.  An example of this is bin-packing, where it is typically assumed that jobs arrive according to a Markov chain.
    \item \textbf{Case $k = 0$}: Here we have that $\mathcal{S} = \X$.  After taking an action $a_t$ based solely on $x_t$ we transition to $x_{t+1} = f(x_t, a_t, \xi_t)$ with $\xi_t$ sampled independently from an unknown distribution $\PXi^T$.  The previous variable $\xi_t$ can be either observed or unobserved.  An example of this is inventory control or newsvendor models, where the demand is typically assumed to be i.i.d. across periods.
\end{itemize}

\paragraph{Relation between Contextual Bandits, MDPs, and Exo-MDPs}
 We first notice that Exo-MDPs are a bridge Between Contextual Bandit and MDPs.
When $\X$ is empty or a singleton, an Exo-MDP describes several variants of the \emph{contextual bandit} introduced in~\citet{langford2007epoch}. 
They can be solved efficiently independent of the horizon~\citep{foster2021statistical}, unlike MDPs. 
If $|\Xi| \leq 1$, an Exo-MDP is simply an MDP whose complexity scales with $|\X|$.  When both $\left| \X \right| > 1$ and $\left| \Xi \right| > 1$ the complexity of learning an Exo-MDP is not known in general, to the best of our knowledge.  When the exogenous inputs are iid, an Exo-MDP is equivalent to an MDP with state space $\X$ much smaller than $\mathcal{S}$.

\paragraph{Difference in Dataset Assumptions}
We next focus briefly on the case when $k = 0$, showing that the key difference between Exo-MDPs and the typical MDP models is the assumptions on the historical dataset provided.  The typical assumptions in an MDP involve that $s_{t+1} \sim P(\cdot \mid s_t, a_t)$ where the underlying distribution $P$ is unknown.  This can be written equivalently as $s_{t+1} = f(s_t, a_t, \xi_t)$ where $\xi_t$ is sampled uniformly in $[0,1]$ and the underlying function $f$ is unknown.  As such, typically in an MDP we consider:
\begin{itemize}
    \item Unknown structure of the dynamics and rewards (i.e. unknown $f$ and $r$)
    \item Known distribution on the underlying exogenous inputs where each $\xi_t$ is uniform over $[0,1]$
    \item Access to a logged dataset of $(s_t, a_t, r_t, s_{t+1})$ pairs
\end{itemize}

In Exo-MDPs we instead assume:
\begin{itemize}
    \item Known structure of the dynamics and rewards (i.e. known $f$ and $r$)
    \item Unknown distribution on the exogenous inputs $\PXi$
    \item Access to a dataset of exogenous traces $\xi_1, \ldots, \xi_T$
\end{itemize}
These types of assumptions (where the \emph{form} of the randomness is known but the true underlying distribution is unknown) is common in the graphon literature~\citep{borgs2008convergent,lovasz2006limits}.  In the following lemma we show that these two models are equivalent, in that any MDP can be written as an MDP with exogenous inputs and $k = 0$ using the uniform random number trick.  However, the {\em assumptions are not equivalent} since in Exo-MDPs we assume access to a dataset of historical exogenous traces rather than trajectories under a fixed behavior policy.

\begin{lemma}
\label{lem:app_model_equiv}
Any MDP of the form $(\mathcal{S}, \A, T, p, R, s_1)$ where the distribution on $p$ and $R$ are unknown has an equivalent Exo-MDP form with $k = 0$, and vice-versa.
\end{lemma}

\begin{rproof}  Without loss of generality we will assume that both $\Xi$ and $\mathcal{S}$ are either discrete or one dimensional (where higher dimensions follow via the same chain of reasoning).

\noindent \textbf{Exo-MDP $\rightarrow$ MDP}: Suppose that $s_{t+1} = f(s_t, a_t, \xi_t)$ where $\xi_t$ is sampled from $\PXi$ and $f$ is known.

We can write this of the form where $f$ is unknown and $\PXi$ is known by setting $\widetilde{\xi_t} \sim U[0,1]$ and $\widetilde{f}(s_t, a_t, \widetilde{\xi_t}) = f(s_t, a_t, \PXi^{-1}(\widetilde{\xi_t})).$  Here the form of $\widetilde{f}$ is unknown as we cannot evaluate $\PXi^{-1}$, but the distribution on the underlying randomness $\widetilde{\xi}$ is known.

\noindent \textbf{MDP $\rightarrow$ Exo-MDP}: Suppose that $s_{t+1} \sim P(\cdot \mid s_t, a_t)$ where the distribution is unknown.  We can write this as $s_{t+1} = f(s_t, a_t, \xi_t)$ with a known $f$ and unknown distribution $\PXi$ as follows.  

First set $\Xi = \Delta(\mathcal{S})^{\mathcal{S} \times \A} \times [0,1]$.  Given any $\xi \in \Xi$ we define the transition kernel $f(s_t, a_t, \xi)$ as follows:
\begin{itemize}
    \item Set $\tilde{p} \in \Delta(\mathcal{S})$ to be the component of $\xi$ indexed via $s_t, a_t$
    \item Letting $z$ be the last component of $\xi$, set $s_{t+1} = \tilde{p}^{-1}(z)$.
\end{itemize}

The distribution over $\Xi$ is then defined as an indicator variable over the first $\mathcal{S} \times \A$ components indicating the true unknown distribution $p$, and the last component over $[0,1]$ being $\text{Uniform}[0,1]$.
\end{rproof}

\subsection{Examples of Exo-MDPs}
\label{sec:input_driven_mdp_examples}

We now give several examples of Exo-MDPs alongside with their exogenous decomposition of the transition distribution. We also highlight the underlying Markovian assumption on the exogenous inputs $\bfxi$. 

\subsubsection{Inventory Control with Lead Times and Lost Sales \texorpdfstring{$(k = 0)$}{k=0}}

This models a single product stochastic inventory control problem with lost sales and lead times~\citep{agrawal_learning_2019, goldberg_survey_2021, xin2015distributionally, feinberg2016optimality}.  In the beginning of every time step $t$, the inventory manager observes the current inventory level $\inv_t$ and $L$ previous unfulfilled orders in the pipeline, denoted $o_L, \ldots, o_1$ for a product.  $L$ denotes the lead time or delay in the number of time steps between placing an order and receiving it.  The next inventory is obtained as follows.  First, $o_1$ arrives and the on-hand inventory rises to $I_t = \inv_t + o_1$.  Then, an exogenous demand $\xi$ is drawn independently from the unknown demand distribution $\PXi$.  The cost to the inventory manager is $$h(I_t - \xi)^+ + p(\xi - I_t)^+$$ where $h$ is the holding cost for remaining inventory and $p$ is the lost sales penalty.  The on-hand inventory then finishes at level $(I_t - \xi)^+$.

This can be formulated as an Exo-MDP by letting $\X = [n]^{L+1}$ denote the current inventory level and previous orders, $\Xi = [n]$ as the exogenous demand, and $\A = [n]$ for the amount to order where $n$ is some maximum order amount.  The reward function is highlighted above, and the state transition updates as $x' = f(x_t, a_t, \xi)$ where $\inv_{t+1} = (\inv_t + o_{1} - \xi)^+$, $o_{k} = o_{k-1}$ for all $1 < k < L$ and $o_L = a$.  This model can also be expanded to include multiple suppliers with different lead times.  

\subsubsection{Online Stochastic Bin Packing \texorpdfstring{$(k = 1)$}{(k=1)}}
\label{sec:bin_packing}

Here we consider a typical online stochastic bin packing model~\citep{gupta2012online,ghaderi2014asymptotic, perboli2012stochastic} where a principal has access to an infinite supply of bins with maximum bin size $B$.  Items $u_t$ arrive over a sequence of rounds $t=1,\ldots, T$ where each $u_t \in [B]$ denotes the item size.  At every time step, the principal decides on a bin to allocate the item to, either allocating it to a previously opened bin or creating a new bin.  The goal is to allocate all of the items using the smallest number of bins.

This can be modeled in the framework as follows.  Here we let $\X = \mathbb{R}^{B}$, $\Xi = [B]$ and $\A = [B]$.  Each vector $x \in \X$ has components $x_1, \ldots, x_B$ as the current number of bins opened with current utilization of one up to $B$, with $\Xi$ corresponding to the current item arrival's size.  Hence the state space is $\mathcal{S} = \X \times \Xi$ where $s_t = (x_t, \xi_{t-1})$ corresponds to the current bin capacity and current item arrival.  Actions $a \in \A$ correspond to either $0$, for opening up a new bin, or $1, \ldots, B$ to be adding the current item to an existing bin with current utilization one up to $B$.  The reward is:
\begin{align*}
    r(x_t, \xi_{t-1}, a, \xi_t) = \begin{cases}
        -1 & a = 0\\
0 & a > 0, s_a > 0, \text{ and }a + \xi_{t-1} \leq B\\
-100 & \text{otherwise}
\end{cases}
\end{align*}
where $-1$ corresponds to the cost for opening a new bin, and the condition on zero reward verifies whether or not there is currently an open bin at level $a$ and the action is feasible (i.e. allocating the current item to the bin at size $a$ is smaller than the maximum bin capacity).

The transition distribution is updated similarly.  Let $\xi_t$ be drawn from the conditional distribution given $\xi_{t-1}$.  If $a = 0$ then $x' = x$ except $x'_{\xi_{t-1}}$ is incremented by one (for opening up a new bin at the level of the size of the current item).  If $a > 0$ and the action is feasible (i.e. $s_a > 0$ and $a + \xi_{t-1} \leq B$) then $x' = x$ with $x'_{a + \xi_{t-1}}$ incremented by one and $x'_{a}$ decreased by one.

We note again that this model can be extended to include different reward functions, multiple dimensions of item arrivals, and departures, similar to the Virtual Machine allocation scenario.

\subsubsection{Online Secretary \texorpdfstring{$(k = 1)$}{(k=1)}}
\label{sec:secretary_example}

Multi-secretary is the generalization of the classic secretary problem \citep{buchbinder2009secretary}, where $T$ candidates arrive sequentially but only $B$ can be selected.  Over time periods, a candidate arrives with ability $r_t \in (0,1]$ drawn i.i.d.~from a finite set of $K$ levels of expertise.  At each round, if the decision-maker has remaining budget (i.e., has chosen less than $B$ candidates thus far), they can \emph{accept} a candidate and collect the reward $r_t$, or \emph{reject} the candidate. The goal is to maximize the expected cumulative reward.

This can be modeled as an Exo-MDP as follows.  Here we let $\X = [B]$, $\Xi = [K]$, and $\A = \{0,1\}$.  The endogenous space $\X$ corresponds to the number of remaining candidates that can be accepted.  The exogenous space $\Xi$ corresponds to the ability level of the next time period's candidate.  Lastly, actions $a \in \A$ correspond to either accepting or rejecting the current candidate.  Hence the state space is $\mathcal{S} = \X \times \Xi$ where $s_t = (x_t, \xi_{t-1})$ corresponds to the number of accepted candidates thus far and the skill of the current candidate.  The reward is:
\[
r(x_t, \xi_{t-1}, a, \xi_t) = \begin{cases}
        \xi_{t-1} & a = 1 \text{ and } x_t > 0\\
0 & \text{otherwise}
\end{cases}
\]

The transition distribution is updated similarly by accounting for whether a candidate was accepted.  Indeed we have:
\[
f(x_t, \xi_{t-1}, a, \xi_t) = \begin{cases}
        x_{t-1} & a = 1 \text{ and } x_t > 0\\
0 & \text{otherwise}
\end{cases}
\]

\subsubsection{Airline Revenue Management \texorpdfstring{$(k = 0)$}{(k=0)}}
\label{sec:example_airline_revenue_management}

 Airline Revenue Management~\citep{littlewood1972forecasting} is a special case of the multi-dimensional Online Bin Packing (OBP) problem, but we reiterate its model here for completeness.  There are a set of $K \in \mathbb{N}$ resources, and each resource $i$ has a maximum capacity $B_i$. 
 Customers are segmented into $M \in \mathbb{N}$ types.  Customers of type $j \in [M]$ request $A_j \in \mathbb{R}_+^K$ resources and yield a revenue of $f_j$.  Over time, the algorithm will decide whether or not to accept customers of type $j$.  Afterwards, a customer type $j_t$ is drawn from an independent distribution.  If the algorithm decided to accept customers of type $j_t$, the relevant resources are consumed and revenue earned.  The goal of the decision-maker is to maximize the expected revenue.

 This is modeled as an Exo-MDP where $\X = [0, B_1] \times \ldots \times [0, B_K], \Xi = [M]$, and $\A = \{0,1\}^M$.  The system space $\X$ corresponds to the remaining capacity of the $K$ different resources, exogenous space $\Xi$ to the sampled customer type, and $\A$ to the accept / reject decisions for each of the customer types.  The reward is then defined via:
 \[
r(x_t, a, \xi_t) = 
    \begin{cases}
        f_{\xi_t} & a_{\xi_t} = 1 \text{ and } x_t - A_{\xi_t} \geq 0 \\
        0 & \text{otherwise} 
        \end{cases}
\]

The transition distribution is updated by accounting for consumed resources if a request is accepted:

 \[
f(x_t, a, \xi_t) = 
    \begin{cases}
        x_t - A_{\xi_t} & a_{\xi_t} = 1 \text{ and } x_t - A_{\xi_t} \geq 0 \\
        x_t & \text{otherwise} 
        \end{cases}
\]

\subsubsection{Virtual Machine Allocation \texorpdfstring{$(k = T)$}{(k=T)}}
\label{sec:virtual_machine_allocation}

The Cloud has modified the way that users are able access computing resources~\citep{sheng2021vmagent, MARO_MSRA, cortez2017resource, hubbs2020or, hadary_protean_2020}.  Cloud service providers allow customers easy access to resources while simultaneously applying efficient management techniques in order to optimize their return.  One of the most critical components is the Virtual Machine (VM) allocator, which assigns VM request to the physical hardware (henceforth referred to as PMs).  The important issue is how to allocate physical resources to service each VM efficiently by eliminating fragmentation, performance impact and delays, and allocation failures.

These VM allocation models can be thought of as a multi-dimensional variant of bin-packing with an additional component of arrival and departures.  The typical VM scheduling scenarios models users requesting resources over time, where each request contains the required CPU and memory uses, and its lifetime.  The allocator then decides which available physical machine to allocate the virtual machine to.  To limit notational overload, we provide a high level view of the Virtual Machine allocation scenario here, and defer concrete discussion and notation when discussing the planning oracle required for solving.

In this set-up the current system state of the model is measured by the physical resources available for each PM, including the physical cores and memory available.  Over time,
\begin{itemize}
    \item Coming VM requests ask for a certain amount of resources (CPU and memory requirements) and their lifetime.  Resource requirements are varied based on the different VM requests.
    \item Based on the action selected by the algorithm, the VM will be allocated to and be created in a specified PM as long as that PM's remaining resources are enough.
    \item After a period of execution, the VM completes its tasks.  The simulator will then release the resources allocated to this VM and deallocate this VM from the PM.
\end{itemize}

At a high level, these problems can be modeled as an MDP with exogenous inputs where the exogenous space $\Xi$ contains the space of possible VM requests (along with their lifetime, memory, and CPU requirements).  The endogenous space $\X$ measures the current capacity of each physical machine on the server, and action space $\A$ for allocation decisions for the current VM request to a given PM.  More details on the concrete experimental set-up will be in \cref{sec:experiments_appendix}.  We also note that the VM arrival process in practice is highly correlated so this fits under the model where $k = T$~\citep{hadary_protean_2020}.

\section{Hindsight Planners}
\label{sec:planning_oracle_open_control}

Here we outline the feasibility of \cref{ass:planning_oracle} in many operations tasks.  Planning problems induced by online knapsack problems as a function of a deterministic input sequence $\bfxi$ can be solved via their induced linear relaxation in pseudo-polynomial time (as the constraint polytope is a polymatroid).  Other problems, such as inventory control with lead times have planning problems where a simple greedy control policy is optimal.  More generally, \cref{ass:planning_oracle} requires us to solve large-scale combinatorial optimization problems.  However, we note that all of these computations are done \emph{offline} and so the computational burden is not required at run-time.  Moreover, it is easy to incorporate existing heuristics from the optimization literature for efficient solutions to these problems, including linear programming or fluid relaxations~\citep{conforti2014integer}.  This appeals to our \GuidedRL algorithms only relying on the objective value of the planner, instead of the actual sequence of actions.

\subsection{Inventory Control with Lead Times and Lost Sales}

In inventory control, given knowledge of the exact sequence of demands $\bfxi = (d_1, \ldots, d_T)$, the optimal open loop control policy is trivial to write down.  Indeed, setting:
\[
    a_t = \begin{cases}
            d_{t+L} & t \leq T - L \\
            0 & \text{otherwise}
        \end{cases}
\]
is clearly the optimal policy.  This is as, for any $t \leq T - L$ we ensure the current on-hand inventory is exactly equal to that period's demands.  For the last $L$ periods we order nothing in order to minimize the accumulated purchase costs for inventory which will be ordered and cannot be sold.  

\subsection{Online Stochastic Bin Packing}
\label{app:bin_pack_plan}

We give the integer programming representation of the optimal open loop control for Bin Packing as follows.  Consider a state $x$ with components $x_1, \ldots, x_B$ as the current number of bins opened with a utilization of $1$ up to $B$ and a sequence of items with sizes $\bfxi = (u_1, \ldots, u_T)$.  Given $(x_1, \ldots, x_B)$ we pre-process this list to a vector of length $\sum_i x_i$ where each component corresponds to the current utilization of any bin.  For example, if $B = 3$ and $x = (1,0,2)$ then we make a list containing $\alpha = (1,3,3)$ for two bins with a utilization of three and one bin with a total utilization of one.  Since the total number of bins required will be $\sum_i x_i + T$ we use the variables $y_b$ for $b \in [\sum_i x_i + T]$ to denote an indicator of whether or not bin $y_b$ is currently utilized.  We also use variables $z_{v,b}$ to denote whether item $v \in [T]$ is assigned to bin $b$.  Similarly, denote $\alpha_b$ as the current utilization of a bin $b$.  The optimization program can then be written as follows:
\begin{align*}
    \max_{z,y}&   - \sum_{b} y_b \\
    \text{s.t.} & \sum_b z_{v,b} = 1 \text{ for all } v \in [T] \\
    & \sum_v z_{v,b} + \alpha_b \leq B y_b \text{ for all } b \\
    & y_b = 1 \text{ for any $b$ with } \alpha_b > 0
\end{align*}

The objective corresponds to minimizing the number of utilized bins.  The first constraint ensures that each item is assigned to a bin.  The second constraint enforces capacity constraints for each bin, and the last constraint ensures that bins are marked as used if they have current utilization on them (i.e. $\alpha_b > 0$).  

\subsection{Online Secretary}
\label{app:secretary_planner}

In online secretary, given knowledge of the exact sequence of future candidate qualities the open loop control policy is trivial to write down.  Indeed, for any $\bfxi_{\geq t}$ we denote $\sigma$ as the ranking function over it such that $\xi_\sigma(1) > \xi_\sigma(2) > \ldots > \xi_\sigma(T - t)$, with ties broken arbitrarily.  Then given a remaining number of candidates to accept $x$, $\plan(t, \bfxi_{\geq t}, x)$ will simply be $\sum_{i=1}^x \xi_{\sigma(i)}$.  This corresponds to taking the best $x$-candidates from the future trace $\bfxi_{\geq t}$.

\subsection{Airline Revenue Management}
\label{app:arm_planner}

The planning oracle for the airline revenue management problem can be formulated as a so-called ``knapsack'' problem.  Indeed, suppose that $x_t$ is the remaining capacity for the $K$ resources.  We use variables $z_y$ for the number of customers of type $y \in [M]$ to accept.  Then, we solve the following optimization problem:
\begin{align*}
    \max_{z} & \sum_{y} f_y z_y \\
    \text{ s.t. } & 0 \leq z_y \leq N_y(\bfxi) \text{ for all } y \in [M] \\
    & Az \leq x_t
\end{align*}
where $N_y(\bfxi) = \sum_{t} \Ind{\xi_t = y}$ is the number of type $y$ customers in the exogenous dataset $\bfxi$.

\subsection{Virtual Machine Allocation}
\label{sec:vm_oracle}

The planning oracle for the VM allocation problem can be formulated as a large-scale mixed integer linear program.  In this section we discuss approaches which utilize this fact in developing an oracle for the hindsight planning problem.

Use $\Xi = V$ to denote the set of VM requests and $P$ as the set of physical machines.  We use the following constants which depend on the current inventory of virtual and physical machines contained in the current state $s_t$, including:
\begin{itemize}
    \item $\alpha_{t,p}$ for the remaining CPU cores for physical machine $p$ at event $t$
    \item $\beta_{t, p}$ for the remaining memory for physical machine $p$ at event $t$
    \item $\texttt{CPU-CAP}_p$ and $\texttt{MEM-CAP}_p$ the CPU and memory capacity of physical machine $p \in P$
    \item $\texttt{LIFETIME}_v, \, \texttt{CORE}_v, \, \texttt{MEM}_v$ as the lifetime, cores, and memory utilization of VM $v \in V$
    \item $\eta_{v, t}$ an indicator that the VM $v$ is active at time $t$ (i.e. $\eta_{v,t}$ is equal to one for any time starting from the time the VM $v$ arrives until the end of its lifetime)
\end{itemize}

With this we introduce variables $x_{v,p}$ for each virtual machine $v$ and physical machine $p$ to indicate the assignments.  We also use variables $y_{t,p}$ which encode whether physical machine $p$ has a VM assigned to it at time step $t$.

We start by considering the various constraints in the problem:
\begin{itemize}
    \item \textbf{Assignment Constraint} (\cref{cons:pack}): For every $v$ we need $\sum_{p} x_{v,p} = 1$ indicating that each virtual machine is assigned to a physical machine.
    \item \textbf{CPU Capacity Constraint} (\cref{cons:cpu}): For every $p$ and $t$ we need $\alpha_{p, t} + \sum_{v} \texttt{CORE}_v \eta_{v, t} x_{v,p} \leq \texttt{CPU-CAP}_p$ to ensure CPU usage capacity constraints are satisfied.
    \item \textbf{Memory Capacity Constraint} (\cref{cons:mem}): For every $p$ and $t$ we need $\beta_{p, t} + \sum_{v} \texttt{MEM}_v \eta_{v, t} x_{v,p} \leq \texttt{MEM-CAP}_p$ to ensure memory capacity constraints are satisfied.
\end{itemize}

We also need additional constraints which encode the $y_{t,p}$ variable as follows:
\begin{itemize}
    \item \textbf{PM Historical Utilization} (\cref{cons:hen_ind}): $y_{t,p} \geq 1$ for all $p$ and $t$ if $\alpha_{t, p} > 0$
    \item \textbf{PM VM Utilization} (\cref{cons:hen_vm}): $x_{v,p} \eta_{v,t} \leq y_{t,p}$ for all $v$ and $p$
    \item \textbf{PM OR Constraint} (\cref{cons:hen_or}): $\sum_{v} x_{v,p} \eta_{v, t} + \Ind{\alpha_{t, p} > 0} \geq y_{t,p}$ for all $t$ and $p$.
\end{itemize}
These constraints essentially encode that $y_{t,p}$ is an indicator for whether $y$ has a VM assigned to it from $\bfxi$ (in the second bullet), or has historical allocations on it (for when $\alpha_{t, p} > 0$).

We note that there always exists a feasible solution since we are in an over provisioned regime where we have capacity to service every VM request.

Lastly, the objective function (\cref{obj:packing}) for the packing density can be formulated via:
\[
    - \sum_t \frac{\sum_p y_{t,p} \texttt{CPU-CAP}_p}{\sum_p \alpha_{t,p} + \sum_{v} \texttt{CORE}_v}
\]

The numerator corresponds to the total CPU capacity of all physical machines which are in use.  The denominator corresponds to the total utilization (both from the VMs currently in service and the VMs arriving over the time horizon).  This then encodes the inverse of the core packing density, as described earlier.

The full integer program is now summarized below:
\begin{align}
    \max_{x,y} & - \sum_{t \in [T]} \frac{\sum_{p \in P} y_{t,p} \texttt{CPU-CAP}_p}{\sum_{p \in P} \alpha_{t,p} + \sum_{v \in V} \texttt{CORE}_v} \label{obj:packing} \\
    \text{s.t. } & \sum_{p \in P} x_{v,p} = 1 ~~~ \forall v \in V \label{cons:pack}\\
    & \alpha_{p, t} + \sum_{v \in V} \texttt{CORE}_v \eta_{v, t} x_{v,p} \leq \texttt{CPU-CAP}_p ~~~ \forall t \in [T], \, p \in P \label{cons:cpu}\\
    & \beta_{p, t} + \sum_{v \in V} \texttt{MEM}_v \eta_{v, t} x_{v,p} \leq \texttt{MEM-CAP}_p ~~~ \forall t \in [T], \, p \in P \label{cons:mem}\\
    & y_{t,p} \geq \Ind{\alpha_{t,p} > 0} ~~~ \forall t \in [T], \, p \in P \label{cons:hen_ind}\\
    & x_{v,p} \eta_{v,t} \leq y_{t,p} ~~~ \forall v \in V, \, t \in [T], \, p \in P \label{cons:hen_vm}\\
    & \sum_{v \in V} x_{v,p} \eta_{v, t} + \Ind{\alpha_{t, p} > 0} \geq y_{t,p} ~~~ \forall t \in [T], \, p \in P \label{cons:hen_or}
\end{align}


\section{Existing Approaches to MDPs in Exo-MDPs}
\label{sec:exogenous_value}

Here we briefly discuss existing approaches to MDPs applied in the context of Exo-MDPs to highlight the advantages and disadvantages of our \GuidedRL approach. 

\subsection{Predict Then Optimize}
\label{sec:pto}
Given the historical trace dataset $\D = \{\bfxi^1, \ldots, \bfxi^N\}$, a popular plug-in approach learns a generative model $\overline{\P_\Xi}(\xi_t \mid \bfxi_{< t})$ to approximate the true distribution $\P_\Xi(\xi_t \mid \bfxi_{< t})$, since the exogenous process is the only unknown. 
Given this model $\overline{\P_\Xi}$, estimates for the $Q_t^\star$ value for the optimal policy can be obtained by solving the Bellman equation with the learned predictor $\overline{\P_\Xi}$ in place of 
the true distribution $\P_\Xi$. More concretely, we denote $\overline{Q}_t$ as the model-based estimate of $Q_t^\star$, which follows
\begin{align*}
    \overline{Q}_t(s,a) &\coloneqq \E_{\xi | \bfxi_{< t}}[r(x,a,\xi) + \overline{V}_{t+1}(f(x,a,\xi) \mid \overline{\PXi}] \\
    \overline{V}_t(s) & \coloneqq \max_{a \in \A} \overline{Q}_t(s,a) \\
    \overline{\pi}_t(s) & \coloneqq \argmax_{a \in \A} \overline{Q}_t(s,a).
\end{align*}

While intuitive, this ML forecast approach requires high-fidelity modeling of the exogenous process to guarantee good downstream decision-making, due to the quadratic horizon multiplicative factor in regret we show below. This quadratic factor in the horizon is due to the compounding errors of distribution shift, similar to those shown in the imitation learning literature~\citep{ross2011reduction}.

\begin{theorem}
\label{thm:pto_regret}
Suppose that $\sup_{t \in [T], \bfxi_{< t} \in \Xi^{[t-1]}} \norm{\overline{\PXi}(\cdot | \bfxi_{< t}) - \PXi(\cdot | \bfxi_{< t})}_{TV} \leq \epsilon$ where $\norm{\cdot}_{TV}$ is the total variation distance. Then we have that $\Regret(\pibar) \leq 2T^2 \epsilon$.  
In addition, if $\bfxi \sim \PXi$ has each $\xi_t$ independent from $\bfxi_{<t}$ with $\Xi$ discrete, $\overline{\PXi}$ is the empirical distribution, then $\forall \delta \in (0,1)$, with probability at least $1 - \delta$,
\ifdefined\arxiv
$\Regret(\pibar) \leq T^{3/2}|\Xi|\sqrt{\frac{2 \log(2 |\Xi| / \delta)}{N}}.$
\else
$\Regret(\pibar) \leq T^{3/2}|\Xi|\sqrt{\frac{2 \log(2 |\Xi| / \delta)}{N}}.$
\fi
\end{theorem}

The $T^2$ dependence here is tight (see \citet{domingues2021episodic}), in contrast to the $O(T)$ dependence in \cref{th:weak theorem}.  Moreover, the ML forecast approach can be impractical when the exogenous process is complex since $\epsilon$ in the worst case can scale as $|\Xi|^T$ if the $\xi_t$ are strongly correlated across $t$.  An example of this is VM allocation where researchers observed that the VM lifetime varies substantially across time, the demand has spikes and a diurnal pattern, and that subsequent requests are highly correlated~ \citep{hadary_protean_2020}.  

This discrepancy highlights an advantage of our \GuidedRL approach.  Consider a VM allocation example with two physical machines each large enough to satisfy the entire demand. Under chaotic and unpredictable arrivals, a planner using erroneous forecasts might spread the requests over the two machines. In contrast, the hindsight learning policy will correctly learn that one machine is sufficient and achieve low regret, even if the total variation distance on the underlying distribution over exogenous inputs is large.

\subsection{Reinforcement Learning}
\label{sec:erm}
Recall that our objective is to solve $\argmax_{\pi \in \Pi} V_1^\pi(s_1)$.  This can be written as $\argmax_{\pi \in \Pi} \E_{\bfxi}[V_1^\pi(s_1, \bfxi)]$ by \cref{thm:exog_bellman}.
Therefore, an alternative way to find approximately optimal policies for an Exo-MDP is to maximize the empirical return directly, similar to the empirical risk minimization strategy of supervised learning:
$    \perm = \argmax_{\pi \in \Pi} \BarExp{V_1^{\pi}(s_1, \bfxi)}
$
where $\BarExp{V_1^{\pi}(s_1, \bfxi)} = \frac{1}{N} \sum_n V_1^\pi(s_1, \bfxi^n).$ 
First observe that the number of samples required to learn a near optimal policy scales linearly with $T$ in this approach. If an additive control variate $\phi(\bfxi)$ (as in~\citet{mao_variance_2019}) is used, the $T$ term is replaced with $T - \E_{\bfxi}[\phi(\bfxi)]$. 
\begin{theorem}
\label{thm:erm_regret}
\ifdefined\arxiv
For any $\delta \in (0,1)$, with probability at least $1 - \delta$ we have that
\[
\Regret(\perm) \leq T \sqrt{\frac{2\log(2|\Pi| / \delta)}{N}}.\]
\else $\forall \delta \in (0,1)$, with probability at least $1 - \delta$, $\Regret(\perm) \leq T \sqrt{\frac{2 \log(2|\Pi| / \delta)}{N}}.$
\fi
\end{theorem}

\cref{thm:erm_regret} highlights that model-free RL methods are a {\em theoretically viable} approach for Exo-MDPs (especially compared to Predict-Then-Optimize in \cref{thm:pto_regret}).  Indeed, \cref{thm:erm_regret} shows that RL methods have asymptotic consistency guarantees {\em if} they converge to the optimal policy in the empirical MDP.  This convergence is an idealized computation assumption that hides optimization issues when studying statistical guarantees, and is incomparable to \cref{ass:planning_oracle} for the hindsight planner (for which we showed several examples in \cref{sec:planning_oracle_open_control}.

Moreover, \cref{tab:multi_sec_performance} suggests that HL and RL are trading bias and variance differently.  When variance is the dominating factor, we expect an algorithm’s performance to improve with additional data. When bias is the dominating factor, however, we expect no marginal benefit from additional data. In Table 1 we see that as $T$ (and accordingly, $N$, the number of data points) increases, the Tabular RL algorithm performance improves. However, hindsight learning has a stable non-zero regret even as we increase $T$.

\section{Proofs of Main Results}
\label{sec:proofs}

\begin{lemma}[\cref{thm:exog_bellman} of \cref{sec:preliminary}]
\label{app:exog_bellman}
For every $t \in [T], (s,a) \in \mathcal{S} \times \A,$ and $\pi \in \Pi$, we have the following:
\begin{align*}
    Q_t^\pi(s,a) & = \E_{\bfxi_{\geq t}}[Q_t^{\pi}(s,a,\bfxi_{\geq t})] \\
    V_t^\pi(s) & = \E_{\bfxi_{\geq t}}[V_t^{\pi}(s, \bfxi_{\geq t})].
\end{align*}
In particular $V_1^{\pi}(s_1) = V_1^{\pi} = \E_{\bfxi}[V_1^{\pi}(s_1, \bfxi)]$.
\end{lemma}
\begin{rproof}
First note that if $Q_t^\pi(s,a)$ is as defined then we have that:
\begin{align*}
    V_t^\pi(s) & = \sum_a \pi(a \mid s) Q_t^\pi(s,a) \text{ (by Bellman equations)} \\
    & = \sum_a \pi(a \mid s) \E_{\bfxi_{\geq t}} \left[ Q_t^\pi(s,a,\bfxi_{\geq t})\right] \\
    & = \E_{\bfxi_{\geq t}} \left[ \sum_a \pi(a \mid s) Q_t^\pi(s,a,\bfxi_{\geq t}) \right] \text{ (by $\pi$ being non-anticipatory)} \\
    & = \E_{\bfxi_{\geq t}} \left[ V_t^\pi(s, \xi_{\geq t})\right].
\end{align*}

Now we focus on showing the result for $Q_t^\pi(s,a)$ by backwards induction on $t$.  The base case when $t = T$ is trivial as $Q_T^\pi(s,a) = \E_{\xi}[r(s,a,\xi)] = \E_{\bfxi_{\geq T}}[Q_T^{\pi}(s,a,\bfxi_{\geq T})]$.

\noindent \emph{Step Case: $(t+1 \rightarrow t)$} For the step-case a simple derivation shows that
\begin{align*}
    Q_t^\pi(s,a) & = \E_\xi\left[ r(s,a,\xi) + V_{t+1}^{\pi}(f(s,a,\xi)) \right] \\
        & = \E_\xi\left[ r(s,a,\xi) + \E_{\bfxi_{\geq t+1}} \left[ V_{t+1}^{\pi}(f(s,a,\xi), \bfxi_{\geq t+1}) \right]\right] \\
        & = \E_{\bfxi_{\geq t}}\left[r(s,a,\xi_t) + V_{t+1}^{\pi}(f(s,a,\xi_t),\xi_{\geq t+1})\right] \\
        & = \E_{\bfxi_\geq t}[Q_t^\pi(s,a,\xi_{\geq t})].
\end{align*}
\end{rproof}

\begin{theorem}[\cref{thm:ip_gap_regret} of \cref{sec:theory}]
\label{app:ip_gap_regret}
$\Regret(\pip) \leq \sum_{t=1}^T \E_{S_t \sim \Pr_t^{\pip}}[\DeltaIPGap_t(S_t)]$ where $\Pr_t^{\pip}$ denotes the state distribution of $\pip$ at step $t$ induced by the exogenous randomness.
In particular, if $\DeltaIPGap_t(s) \leq \Delta$ for some constant $\Delta$ then we have:
$\Regret(\pip) \leq \Delta \sum_{t=1}^T \E_{S_t \sim \Pr_t^{\pip}}[\Pr(\pip_t(S_t) \neq \pi^\star(S_t))].$
\end{theorem}
\begin{rproof}
First we note that via the performance difference lemma we have that for any two non-anticipatory policies $\pi$ and $\tilde{\pi}$ that
\begin{align*}
    V_1^\pi(s_1) - V_1^{\tilde{\pi}}(s_1) & = \sum_{t} \E_{s_t \sim \Pr_t^{\pi}}\left[ \sum_a \pi(a \mid s_t)(Q_h^{\tilde{\pi}}(s,a) - V_h^{\tilde{\pi}}(s)) \right] \text{ and so } \\
    V_1^{\tilde{\pi}}(s_1) - V_1^\pi(s_1) & = \sum_{t} \E_{s_t \sim \Pr_t^{\pi}}\left[ \sum_a \pi(a \mid s_t)(V_h^{\tilde{\pi}}(s) - Q_h^{\tilde{\pi}}(s,a)) \right]
\end{align*}

Moreover, for any state $s$ we also have $Q_t^\star(s, \pi^\star(s)) - Q_t^\star(s, \pip(s)) \leq \DeltaIPGap_t(s)$ since:
\begin{align*}
    & Q_t^\star(s, \pi^\star(s)) - Q_t^\star(s, \pip(s)) - \DeltaIPGap_t(s) \\ & = Q_t^\star(s, \pi^\star(s)) - Q_t^\star(s, \pip(s)) - \Qip_t(s,\pip(s)) + Q^\star_t(s, \pip(s)) - Q^\star_t(s, \pi^\star(s)) + \Qip_t(s, \pi^\star(s)) \\
    & = \Qip_t(s, \pi^\star(s)) - \Qip_t(s, \pip(s)) \leq 0
\end{align*}
as $\pip$ is greedy with respect to $\Qip$.

Finally, recall the definition of the regret of $\Regret(\pip)$ via $V_1^\star(s_1) - V_1^{\pip}(s_1)$.  However, using the previous performance difference lemma with $\tilde{\pi} = \pi^\star$ and $\pi = \pip$ we have that
\begin{align*}
    V_1^\star(s_1) - V_1^{\pip}(s_1) & = \sum_t \E_{S_t \sim \Pr_t^{\pip}} \left[Q_t^\star(S_t, \pi^\star(S_t)) - Q_t^\star(S_t, \pip(S_t)) \right] \\
    & \leq \sum_t \E_{S_t \sim \Pr_t^{\pip}} \left[\DeltaIPGap_t(S_t)\right].
\end{align*}
The second statement follows immediately from the absolute bound on $\DeltaIPGap_t(s)$.
\end{rproof}

\begin{theorem}[\cref{th:weak theorem} of \cref{sec:theory}]
Let $\overline{\pi}^\dagger$ denote the hindsight planning surrogate policy for the empirical MDP w.r.t. $\mathcal{D}$. Assume $\overline{\pi}^\dagger \in \Pi$ and \cref{alg:training} achieves no-regret in the optimization problem. 
Let $\pi$ be the best policy generated by \cref{alg:training}.
Then, for any $\delta\in(0,1)$, with probability $1-\delta$, it holds
\begin{align*}
    \Regret(\pi) \leq  T \sqrt{\frac{2 \log(2|\Pi| / \delta)}{N}} + \sum_{t=1}^T \E_{s_t \sim \overline{\textrm{Pr}}_t^{\overline{\pi}^\dagger}}[\bar{\Delta}_t^{\dagger}(s_t)] + o(1)
\end{align*}
where $\bar{\Delta}_t^{\dagger}$ is the SAA approximation of \eqref{eq:delta ip gap} and $\overline{\textrm{Pr}}_t^{\overline{\pi}^\dagger}$ is the state probability of   $\overline{\pi}^\dagger$ in the empirical MDP.
\end{theorem}
\begin{rproof}
The proof follows the standard proof technique of online IL (cf. \citep{yan2021explaining}). 
\begin{align*}
 \Regret(\pi)&=  \Exp{V_1^{*}(s_1, \bfxi)} - \Exp{V_1^{\pi}(s_1, \bfxi)} \\
 &\leq \left( | \BarExp{V_1^{\pi}(s_1, \bfxi)} - \Exp{V_1^{\pi}(s_1, \bfxi)}| + | \BarExp{V_1^{\pi^*}(s_1, \bfxi)} - \Exp{V_1^{\pi^*}(s_1, \bfxi)}| \right) \\
 &\quad + \left( 
 \BarExp{V_1^{\pi^*}(s_1, \bfxi)} - \BarExp{V_1^{\bar{\pi}^\dagger}(s_1, \bfxi)}
 + \BarExp{V_1^{\bar{\pi}^\dagger}(s_1, \bfxi)} -  \BarExp{V_1^{\pi}(s_1, \bfxi)} \right) \\
 &\leq   2T \sqrt{\frac{2 \log(2|\Pi| / \delta)}{N}}
 + \BarExp{V_1^{{\pi}^*}(s_1, \bfxi)} -  \BarExp{V_1^{\bar{\pi}^\dagger}(s_1, \bfxi)} + o(1)\\
  &\leq   2T \sqrt{\frac{2 \log(2|\Pi| / \delta)}{N}}
  + \sum_{t=1}^T \E_{s_t \sim \bar{\textrm{Pr}}_t^{\bar{\pi}^\dagger}}[\bar{\Delta}_t^{\dagger}(s_t)] + o(1)
\end{align*}
We use the results of \cref{app:erm_regret} to bound the first terms in the second line; we invoke the no-regret optimization assumption and the realizability assumption $\overline{\pi}^\dagger \in \Pi$  for the last term of the second line. Finally, we apply \cref{thm:ip_gap_regret} in the empirical MDP w.r.t. $\D$ and recognize that the middle term is the empirical regret to derive the last step.
\end{rproof}

\begin{lemma}[\cref{lem:neg_instance} of \cref{sec:theory}]
There exists a set of Exo-MDPs such that $\Regret(\pip) \geq \Omega(T)$.
\end{lemma}
\begin{rproof}
We first construct a three-step MDP such that $\Regret(\pip) \geq \Omega(1)$.  The main result then follows by replicating the MDP across $T$ periods to construct a $3T$ step MDP with $\Regret(\pip) \geq \Omega(T)$.

We consider a modification of the prototypical Pandora's Box problem~\citep{weitzman1979optimal}.  The endogenous state space $\X = \{0,1\}$ where state $0$ corresponds to ``not yet accepted an item'' and $1$ corresponds to ``accepted an item''.  The action space $\A = \{0,1\}$ where $a = 0$ corresponds to ``reject next item'' and $a = 1$ corresponds to ``accept next item''.  We consider a modification of the typical Pandora box model where at time step $t$, the next item arrivals $\xi_t$ value is unobserved before deciding whether or not to accept.

The trace distribution has $\xi_1 \sim U[0,1]$ and $\xi_2, \xi_3 \sim U[0, .9]$.  Important to note is that $\Exp{\xi_1} = .5, \Exp{\xi_2} = 0.45, \Exp{\xi_3} = 0.45,$ and a straightforward calculation shows that $\Exp{\max(\xi_2, \xi_3)} = 0.6$.  

The rewards and dynamics are:
\begin{align*}
    r(0,0,\xi) & = 0 & f(0,0,\xi) & = 0 \\
    r(0,1,\xi) & = \xi & f(0,1,\xi) & = 1\\
    r(1, a, \xi) & = 0 & f(1,a,\xi) & = 1 \text{ for } a \in \{0, 1\}.
\end{align*}
Lastly, the starting state $s_1 = 0$.  This properly encodes the exogenous dynamics and rewards.  At step $t$ in state $x = 0$ (i.e. not yet accepted an item) taking action $0$ (do not accept) yields no return and transitions to the next state.  However, accepting the next item returns reward $\xi$ and transitions to state $x=1$.

A straightforward calculation following the Bellman equations shows the following for $Q_t^\star$ and $V_t^\star$:
\begin{align*}
    Q_3^\star(0,0) & = 0 & Q_2^\star(0,0) & = \Exp{\xi_3} & Q_1^\star(0,0) & = \max(\Exp{\xi_2}, \Exp{\xi_3}).\\
    Q_3^\star(0,1) & = \Exp{\xi_3} & Q_2^\star(0,1) & = \Exp{\xi_2} & Q_1^\star(0,1) & = \Exp{\xi_1} \\
    Q_3^\star(1,\cdot) & = 0 & Q_2^\star(1, \cdot) & = 0 & Q_1^\star(1, \cdot) & = 0\\
    V_3^\star(0) & = \Exp{\xi_3} & V_2^\star(0) & = \max(\Exp{\xi_3}, \Exp{\xi_2}) & V_1^\star(0) & = \max(\Exp{\xi_1}, \Exp{\xi_2}, \Exp{\xi_3}) \\
    V_3^\star(1) & = 0 & V_2^\star(1) & = 0 & V_1^\star(1) & = 0.
\end{align*}
Using the choice of the distributions for $\xi_1, \xi_2, \xi_3$ we have that $\pi_1^\star(0) = 1$, as in, we will accept the first item since on average it has larger expected return.  This results in $V_1^{\pi^\star}(s_1) = \Exp{\xi_1} = 0.5$.

We can similarly compute $\Qip_t$ and $\Vip_t$ as follows:
\begin{align*}
    \Qip_3(0,0) & = 0 & \Qip_2(0,0) & = \Exp{\xi_3} & \Qip_1(0,0) & = \Exp{\max(\xi_2, \xi_3)}.\\
    \Qip_3(0,1) & = \Exp{\xi_3} & \Qip_2(0,1) & = \Exp{\xi_2} & \Qip_1(0,1) & = \Exp{\xi_1} \\
    \Qip_3(1,\cdot) & = 0 & \Qip_2(1, \cdot) & = 0 & \Qip_1(1, \cdot) & = 0.\\
\end{align*}
In this scenario, we see first hand the bias introduced when considered $\Qip$.  In particular, $\Qip_1(0,0) = \Exp{\max(\xi_2, \xi_3)} \geq Q_1^\star(0,0) = \max(\Exp{\xi_2, \xi_3})$.  Using the choice of distributions for $\xi_1, \xi_2, \xi_3$ we see that the hindsight planning policy $\pip$ is fooled and has $\pip_1(0) = 0$, so the policy rejects the first item thinking it will get the maximum value of the next two items.  As a result we see that $V_1^{\pip}(0) = 0.45$.

Hence, we have that $\Regret{\pip} = 0.5 - 0.45 = 0.05 = \Omega(1)$ as needed.
\end{rproof}

\begin{lemma}[Statement in \cref{sec:theory}]
\label{lem:app_eq_equiv_ip}
Suppose that for every $t$ and state $s$ that $\max_{\pi \in \Pi} \E_{\bfxi_{\geq t}}[V_t^\pi(s, \bfxi_{\geq t})] = \E_{\bfxi_{> t}}[\plan(t, s, \bfxi_{\geq t})]$.  Then we have that $Q_t^\star(s,a) = \Qip_t(s,a)$ for every $t$, $s$, and $a$.
\end{lemma}
\begin{rproof}
First notice that $\max_{\pi \in \Pi} \E_{\bfxi_{\geq t}}[V_t^\pi(s, \bfxi_{\geq t})] = \E_{\bfxi_{\geq t}}[V_t^\star(s, \bfxi_{\geq t})]$ by definition.  Thus using the Bellman equations and definition of $\Qip_t(s,a)$ we trivially have that:
\begin{align*}
    Q_t^\star(s,a) & = \E_{\xi_t}[r(s,a,\xi_t) + V_{t+1}^\star(f(s,a,\xi_t))] \\
    & = \E_{\xi_t}[r(s,a,\xi_t) + \E_{\bfxi_{> t}}[V_{t+1}^\star(f(s,a,\xi_t), \bfxi_{> t})]] \\
    & =  \E_{\xi_t}[r(s,a,\xi_t) + \E_{\bfxi_{> t}}[\plan(t+1, f(s,a,\xi_t), \bfxi_{> t})]] \\
    & = \Qip_t(s,a).
\end{align*}
\end{rproof}

\srsedit{
\begin{lemma}[Statement in \cref{sec:theory}]
\label{lem:mercier_example}
Define $\widetilde{GAG} = \Exp{\max_{\pi} \sum_{t=1}^T \DeltaIPGap_t(S_t) \mid S_t \sim \Pr_t^{\pi}}$.  Then there exists an Exo-MDP such that $\widetilde{GAG} = \Omega(T)$ and yet the upper bound in \cref{thm:ip_gap_regret} is zero.
\end{lemma}

The ``MakeDecision'' function in \citet{mercier2007performance} implements the empirical Bayes Selector policy using the offline dataset (see \cref{eq:qip}).  There are several key differences between our regret analysis and theirs.  First, \citet{mercier2007performance} define regret with respect to the hindsight optimal policy $\Vip(s_1)$, whereas our regret is with respect to the best non-anticipatory policy $V^\star(s_1)$. Their main result (Theorem 1) shows that the hindsight optimal regret is bounded by the “Global Anticipatory Gap” $GAG = \Exp{\max_{\pi} \sum_{t=1}^T \Omega_t(S_t, A_t) \mid (S_t, A_t) \sim \Pr_t^{\pi}}$. However, there are situations where the GAG can be large, and one could incorrectly conclude that hindsight optimization should not be applied when trying to learn the true optimal non-anticipative policy $\pi^\star$. In the Sailing example of \cref{sec:example}, the GAG will be positive since knowing the direction of the wind one can ex-post identify the optimal route compared to any non-anticipatory algorithm. However, $Q^\star(\text{route}) = \Qip(\text{route})$ and so there is no hindsight bias.  Hence, our results which adjust for the difference in benchmark to $\pi^\star$ is more appropriate.

Moreover, their regret bound measures a worst-case overestimation bias on the states visited by any decision policy. In contrast, \cref{thm:ip_gap_regret} requires that the hindsight bias be small only on states visited by $\pip$. Even if we set aside the difference between our benchmarks for regret (adjusting the definition by $V^\star(s_1) - \Vip(s_1)$ resulting in $\widetilde{GAG}$ defined in the statement of \cref{lem:mercier_example}), our analysis is much tighter which we illustrate with an example below.

\begin{rproof}
We construct an Exo-MDP $\M$ as follows.  Consider a starting state $s_0$ with two actions $A$ and $B$ that deterministically transition to two different ``sub-MDPs'' (which we denote as $\M_A$ and $\M_B$ respectively).  

The first action, $A$, transitions to sub-MDP $\M_A$ which contains an absorbing state $s_A$ (i.e. transitions are deterministic to the same state $s_A$ regardless of the action) with large rewards.  Note that since this sub-MDP is deterministic it has no hindsight bias (so $\DeltaIPGap_t(s_A) = 0)$.

The second action $B$, transitions to an MDP $\M_B$ which witnesses \cref{lem:neg_instance}, and hence has $\Omega(T)$ hindsight bias.  We adjust the rewards along $\M_B$ such that the value of the optimal policy in this sub-MDP is much smaller than the deterministic reward accrued in $\M_A$.

In this example, $\pip$ will always select action $A$ in the initial state and collect higher rewards (since the optimal policy knowing the exogenous inputs in MDP $\M_B$ will still collect smaller rewards than the deterministic value accrued in $\M_A$).  Hence, our regret bound in \cref{thm:ip_gap_regret} will be zero (since $\DeltaIPGap(s_A) = 0$ and $\pip$ will never visit $s_B$).  However, $\widetilde{GAG}$ will conservatively account for the sub-optimal $B$ decision which transitions to a state with a large anticipatory gap ($\Omega(T)$), thereby concluding a large regret bound.
\end{rproof}

}

\begin{theorem}[\cref{thm:pto_regret} of \cref{sec:exogenous_value}]
\label{thm:app_pto_regret}
Suppose that $\sup_{t \in [T], \bfxi_{< t} \in \Xi^{[t-1]}} \norm{\overline{\PXi}(\cdot | \bfxi_{< t}) - \PXi(\cdot | \bfxi_{< t})}_{TV} \leq \epsilon$ where $\norm{\cdot}_{TV}$ is the total variation distance. Then we have that $\Regret(\pibar) \leq 2T^2 \epsilon$.  
In addition, if $\bfxi \sim \PXi$ has each $\xi_t$ independent from $\bfxi_{<t}$ with $\Xi$ discrete, $\overline{\PXi}$ is the empirical distribution, then $\forall \delta \in (0,1)$, with probability at least $1 - \delta$, $\Regret(\pibar) \leq T^{3/2}|\Xi|\sqrt{\frac{2 \log(2 |\Xi| / \delta)}{N}}.$
\end{theorem}
\begin{rproof}
First note that $\overline{Q}_t$ and $\overline{V}_t$ refer to the $Q$ and $V$ values for the optimal policy in a modified MDP $\overline{M}$ where the true exogenous input distribution $\PXi(\cdot \mid \bfxi_{< t})$ is replaced by its estimate $\overline{\PXi}(\cdot \mid \bfxi_{< t})$.  As such, denote by $\overline{V}_t^\pi$ as the value function for the policy $\pi$ in the MDP $\overline{M}$.  Note here that $\overline{V}_t^{\pibar} = \overline{V}_t$ by construction.  With this we have that:
\begin{align*}
\Regret(\pibar) & = V_1^\star(s_1) - V_1^{\pibar}(s_1) \\
& = V_1^\star(s_1) - \overline{V}_1^{\pi^\star}(s_1) + \overline{V}_1^{\pi^\star}(s_1) - \overline{V}_1(s_1) + \overline{V}_1(s_1) - V_1^{\pibar}(s_1) \\
& \leq 2 \sup_{\pi} | V_1^\pi(s_1) - \overline{V}_1^\pi(s_1) |.
\end{align*}
However, using the finite horizon simulation lemma (see Lemma 1 in \citet{abbeel2005exploration}) we have that this is bounded from above by $2 T^2 \norm{P(\cdot \mid s,a) - \overline{P})(\cdot \mid s,a)}_{TV}$ where $\norm{P(\cdot \mid s,a) - \overline{P})(\cdot \mid s,a)}_{TV}$ is the total variation distance in the induced state-transition distributions between $M$ and $\overline{M}$.  However, by definition we have that:
\begin{align*}
    \norm{P(\cdot \mid s,a) - \overline{P}(\cdot \mid s,a)}_{TV} & = \frac{1}{2} \int_{\mathcal{S}} \abs{P(s' \mid s,a) - \overline{P}(s' \mid s,a)} ds \\
    & = \frac{1}{2} \int_{\mathcal{S}} \abs{\int_\Xi \Ind{s' = f(s,a,\xi)} d \PXi(\xi \mid \bfxi_{\geq t}) - \Ind{s' = f(s,a,\xi)} d \overline{\PXi}(\xi \mid \bfxi_{\geq t})} ds \\
    & = \frac{1}{2} \int_{\mathcal{S}} \int_\Xi \Ind{s' = f(s,a,\xi)} \abs{d \PXi(\xi \mid \bfxi_{\geq t}) - d \overline{\PXi}(\xi \mid \bfxi_{\geq t})} \\
    & \leq \norm{\PXi(\cdot \mid \bfxi_{\geq t}) - \overline{\PXi}(\cdot \mid \bfxi_{\geq t})}_{TV} \leq \epsilon.
    \end{align*}
Thus we get that $\Regret(\pibar) \leq 2T^2 \epsilon$ as required.

Now suppose that $\bfxi \sim \PXi$ has each $\xi_t$ independent from $\bfxi_{< t}$ and let $\overline{\PXi}$ be the empirical distribution, i.e. $\overline{\PXi}(\xi) = \frac{1}{NT} \sum_{i \in [N], t \in [T]} \Ind{\bfxi_{t}^i = \xi}$.  A straightforward application of Hoeffding's inequality shows that the event:
$$\mathcal{E} = \left\{\forall \xi: |\overline{\PXi}(\xi) - \PXi(\xi)| \leq \sqrt{\frac{\log(2|\Xi| / \delta)}{2NT}} \right\}$$
occurs with probability at least $1 - \delta$.  Under $\mathcal{E}$ we then have that:
\begin{align*}
    \sup_{t \in [T], \bfxi_{< t} \in \Xi^{[t-1]}} \norm{\overline{\PXi}(\cdot | \bfxi_{< t}) - \PXi(\cdot | \bfxi_{< t})}_{TV} & \leq \sup_{t \in [T], \bfxi_{< t} \in \Xi^{[t-1]}} \norm{\overline{\PXi}(\cdot | \bfxi_{< t}) - \PXi(\cdot | \bfxi_{< t})}_{1} \leq |\Xi| \sqrt{\frac{\log(2|\Xi| / \delta)}{2NT}}.
\end{align*}
Taking this in the previous result shows the claim.
\end{rproof}

\begin{theorem}[\cref{thm:erm_regret} of \cref{sec:exogenous_value}]
\label{app:erm_regret}
Given any $\delta \in (0,1)$ then with probability at least $1 - \delta$ we have that if $\perm = \argmax_{\pi} \BarExp{V^{\pi}(\xi)}$ that 
\[
    \Regret(\perm) \leq \sqrt{\frac{2T^2 \log(2|\Pi| / \delta)}{N}}.
\]
\end{theorem}

\begin{rproof}
A quick calculation using Hoeffding's inequality and a union bound shows that the event
$$\mathcal{E} = \left\{\forall \pi \in \Pi: \abs{V_1^\pi(s_1) - \BarExp{V_1^\pi(s_1)}} \leq \sqrt{\frac{T^2 \log(2|\Pi| / \delta)}{2N^2}} \right\}$$
occurs with probability at least $1 - \delta$.  Under $\mathcal{E}$ we then have that:
\begin{align*}
    \Regret(\perm) & = V_1^{\pstar}(s_1) - V_1^{\perm}(s_1) \\
    & = V_1^{\pstar}(s_1) - \BarExp{V_1^{\pstar}(s_1, \bfxi)} + \BarExp{V_1^{\pstar}(s_1, \bfxi)} - \BarExp{V_1^{\perm}(s_1, \bfxi)} + \BarExp{V_1^{\perm}(s_1, \bfxi)} - V_1^{\perm}(s_1) \\
    & \leq 2 \sqrt{\frac{V_{max}^2 \log(2|\Pi| / \delta)}{2N^2}}.
\end{align*}
\end{rproof}

\begin{theorem}[\cref{thm:delta_ip_bin_packing} of \cref{app:bin_packing_results}]
\label{app:delta_ip_bin_packing}
In stochastic online bin packing with i.i.d. arrivals we have that $\sup_{t, s} \DeltaIPGap_t(s) \leq O(1)$, independent of the time horizon and any problem primitives.  As a result, $\Regret(\pip) \leq O(1)$.
\end{theorem}

We show the result by starting with the lemma, highlighting that the value functions for the planning policy and the optimal non-anticipatory policy are ``Lipschitz'' with respect to the capacity of the current bins.  Recall that the state space representation $s \in \mathcal{S}$ corresponds to $s = (x, \xi_{t-1})$ where $x \in \mathbb{R}^{B}$ is the current number of bins at that size, and the last component to the current arrival.  We write this explicitly as containing $s \in \mathbb{R}^{B}$ for the bin utilization and $\xi_{t-1} \in \mathbb{R}$ for the current arrival.
\begin{lemma}
For any $t \in [T]$, current bin capacity $x \in \mathbb{R}^{|B|}$, current arrival $\xi_{t-1}$, $\bfxi_{\geq t} \in \Xi^{T - t}$, and $\Delta \in \mathbb{R}^{B} \geq 0$ we have that:
\begin{itemize}
    \item $\Vip_t(x, \xi_{t-1}, \bfxi_{\geq t}) \geq \Vip_t(x-\Delta, \xi_{t-1}, \bfxi_{\geq t}) \geq \Vip_t(x,\xi_{t-1},\bfxi_{\geq t}) - \norm{\Delta}_1$
    \item $V_t^\star(x,\xi_{t-1}) \geq V_t^\star(x - \Delta,\xi_{t-1}) \geq V_t^\star(x,\xi_{t-1}) - \norm{\Delta}_1$
\end{itemize}
As a result for any $x$ and $x'$ in $\mathbb{R}^{B}$ and current arrival $\xi_{t-1}$ we have that:
\begin{itemize}
    \item $\Vip_t(x, \xi_{t-1}, \bfxi_{\geq t}) - \Vip_t(x', \xi_{t-1}, \bfxi_{\geq t}) \leq \norm{(x-x')^+}_1$
    \item $V_t^\star(x, \xi_{t-1}) - V_t^\star(x', \xi_{t-1}) \leq \norm{(x - x')^+}_1$
\end{itemize}
\end{lemma}
\begin{rproof}
First consider the top statement in terms of the optimal planning policy starting from a fixed state $s = (x, \xi_{t-1})$ and sequence of future exogenous variables $\bfxi_{\geq t}$.

We have that $\Vip_t(x, \xi_{t-1}, \bfxi_{\geq t}) \geq \Vip_t(x - \Delta, \xi_{t-1}, \bfxi_{\geq t})$ since the sequence of actions generated by the planning oracle starting from state $(x - \Delta, \xi_{t-1})$ is feasible for the same problem starting from $(x, \xi_{t-1})$.  Hence, as $\Vip_t(x, \xi_{t-1}, \bfxi_{\geq t})$ denotes the optimal such policy, the inequality follows.

For the other direction consider the sequence of actions starting from $(x, \xi_{t-1})$.  Using at most $\norm{\Delta}_1$ bins the policy is feasible for the same problem starting at $(x - \Delta, \xi_{t-1})$.  Indeed, suppose the sequence of actions starting from the problem at $(x, \xi_{t-1})$ attempts to use a bin which is not available in the problem starting from $(x - \Delta, \xi_{t-1})$.  Then by opening a new bin instead and shifting all future references of the old bin to the newly created bin, the sequence of actions is feasible.  As there are at most $\norm{\Delta}_1$ bins different in the $(x,\xi_{t-1})$ problem versus the $(x - \Delta, \xi_{t-1})$ problem, the bound follows.

Now consider the second statement in terms of the optimal non-anticipatory policy starting from a fixed state $(x, \xi_{t-1})$.  First note that $V_t^\star(x, \xi_{t-1}) = \E_{\bfxi_{\geq t}}[V_t^\star(x, \xi_{t-1}, \bfxi_{\geq t})]$ and similarly for $V_t^\star(x - \Delta, \xi_{t-1})$.  We have that $V_t^\star(x, \xi_{t-1}) \geq V_t^\star(x - \Delta, \xi_{t-1})$ as the optimal policy starting from $(x - \Delta, \xi_{t-1})$ is feasible on all sample paths generated by $\bfxi_{\geq t}$ to the same problem starting at $(x, \xi_{t-1})$. Hence by optimality of $\pi^\star$ the inequality must follow.

For the other direction, on any sample path consider the sequence of actions generated by the optimal policy starting from $(x, \xi_{t-1})$.  By a similar argument, again using at most $\norm{\Delta}_1$ extra bins the policy is feasible for the problem starting at $(x - \Delta, \xi_{t-1})$.  Hence by optimality, the inequality follows.

The second result follows via straightforward algebraic manipulations.  Indeed, the previous statement can be thought of as showing that for $x \in \mathbb{R}^{B}$ and $\Delta \in \mathbb{R}^{B}_+$ that $f(x) \geq f(x - \Delta) \geq f(x) - \norm{\Delta}_1$.  However, 
$$f(x) - f(x') = f(x) - f(x'+(x-x')^+) + f(x'+(x-x')^+) - f(x') \leq f(x'+(x-x')^+) - f(x') \leq \norm{(x-x')^+}_1$$
where the first inequality uses that $x'+(x-x')^+ \geq x$ and the second the previous result.
\end{rproof}

We are now ready to show the bound that $\DeltaIPGap_t(s) \leq O(1)$.

\begin{rproof}
For a fixed time $t$ and state $s$ consider $\DeltaIPGap_t(s) = \Qip_t(s,\pip(s)) - Q^\star_t(s, \pip(s)) + Q^\star_t(s, \pi^\star(s)) - \Qip_t(s, \pi^\star(s))$

However consider $\Qip_t(s,\pip(s)) - \Qip_t(s, \pi^\star(s))$ (with the other term dealt with similarly).  On any sample path, the difference in these terms is bounded by the immediate reward plus the difference of the value at the next states.  By problem definition, the difference in immediate rewards is bounded by one.  However, consider the difference in value functions at the next state.  Their state representation has a value of $\norm{(x - x')^+}_1$ of at most 2 (for the two bins that were potentially modified).  Hence, this difference is bounded by $3$ in total.  A similar argument for $Q^\star$ completes the proof. 
\end{rproof}

\section{Simulation Details}
\label{sec:experiments_appendix}

In this section we provide full details on the simulations conducted, including a formal description of virtual machine allocation scenarios along with its fidelity to the real-world cloud scenarios, training implementations, hyperparameter tuning results, and a description of the heuristic algorithms compared.

\subsection{Online Bin-Packing}
\label{app:bin_packing_results}

In a \textbf{Stochastic Online Bin Packing (OBP)} problem the agent has an infinite supply of bins of size $B$. Each round, items $u_t \in \{0,\dots B\}$ arrive sampled iid from an unknown distribution. The agent either \emph{packs} the item into an opened feasible bin or \emph{opens} a new bin, with the goal to minimize the expected number of opened bins after $T$ rounds. 

\cref{sec:gen_examples_mdp} describes how OBP are Exo-MDPs with $u_t$ as the exogenous inputs. 
Lemma 3.1 from \citet{freund_good_2019} shows that in stochastic OBP, $\Pr(\pip_t(S_t) \neq \pi^\star_t(S_t)) \leq O(\frac{1}{t^2})$. Intuitively, as $t \rightarrow T$ there is little contribution from $\plan(t,\bfxi_{>t})$ to the $\Qip$ of Equation~\ref{eq:qip} and so $\pip$ is more likely to coincide with $\pi^\star$. 
Using a novel absolute bound on $\Delta$ and \cref{thm:ip_gap_regret}, we have:
\begin{lemma}
\label{thm:delta_ip_bin_packing}
In stochastic online bin packing with i.i.d. arrivals, $\sup_{t, s} \DeltaIPGap_t(s) \leq O(1)$, independent of 
any problem primitives.  Hence, $\Regret(\pip) \leq O(1)$.
\end{lemma}

To numerically validate this claim, we use \ORSuite~\citep{archer2022orsuite} as an OBP simulator with $B=5$ and vary $T$ from $5$ to $100$. For these small problem sizes, we can compute $\pi^\star$ by solving Bellman equations exactly. The exogenous process $\PXi$ is uniform: $u_t \sim \text{Unif}[B]$ and we generate $|\D| = 1000$ traces, and benchmark the resulting learned policy against $\pi^\star$. The hindsight planner is represented with the integer program in \cref{app:bin_pack_plan} (solved efficiently by linear relaxation).
Any learned policy maps a $B+2$-dim state (a vector denoting number of bins open with utilization from $1$ to $B$, the current item and the remaining rounds) to a decision $A \in \{0,\dots, B-u_t\}$ to select a feasible bin ($0$ opens a new bin).

\begin{table}[ht]
\caption{Hindsight bias in OBP decreases as $T$ increases. Thus, $\pip$ becomes a better surrogate for $\pi^\star$.} 
\label{tab:bin_results}
\setlength\tabcolsep{0pt} 

\smallskip 
\begin{tabular*}{\columnwidth}{@{\extracolsep{\fill}}rcc}
\toprule
  $T$  & MaxBias & \% $\{\exists s: \pi^*_t(s) \neq \pip_t(s)\}$ \\
\midrule
  $5$ & $1.240$ & $6.8\%$ \\
  $10$ & $0.646$ & $3.4\%$\\
  $100$ & $0.066$ & $0.3\%$\\
\bottomrule
\end{tabular*}
\end{table}

For each OBP problem with $T \in \{ 5, 10, 100\}$ we report in Table~\ref{tab:bin_results} 
 the maximum hindsight bias $\text{MaxBias}=\Exp{\max_{s,t} \DeltaIPGap_t(s)}$, where the expectation averages over $1000$ sampled problem instances. We also report the percentage of problem instances where at least a single state witnesses $\pi^*_t(s) \neq \pip_t(s)$. We see that as $T$ increases, $\pip$ becomes a good surrogate for $\pi^\star$ which bodes well for HL.

\subsection{Multi-Secretary}

Multi-secretary is the generalization of the classic secretary problem \citep{buchbinder2009secretary}, where $T$ candidates arrive sequentially but only $B$ can be selected.  An arriving candidate at time $t$ has ability $r_t \in (0,1]$ drawn i.i.d.~from a finite set of $K$ levels of expertise.  At each round, if the decision-maker has remaining budget (i.e., has chosen less than $B$ candidates thus far), they can \emph{accept} a candidate and collect the reward $r_t$, or \emph{reject} the candidate. The goal is to maximize the expected cumulative reward.  \cref{sec:secretary_example} shows how the multi-secretary problem can be formulated as an Exo-MDP.

We use $T=\{5,10,100\}, B=\frac{3}{5}T, K=4$ for our experiments. 
 With four expertise levels the corresponding abilities for the expertise levels was chosen to be $\{1/4, 1/2, 3/4, 1\}$. 
The arrival process for the ability types is non-stationary and sinusoidal with a type-dependent shift and frequency.  Denoting $p_j^t$ as the arrival probability of a type $j$ customer at timestep $t$, the distribution is as follows.  First, $p_j^1$ is chosen to be uniformly at random from $[0, 2\pi]$.  Next, the frequency for each $j$ is chosen to be uniformly from $[0, \pi/4]$.  The final arrival probabilities $p_j^t$ are then chosen to be sinusoidal with that shift and frequency value.  These values are then normalized appropriately to be a valid distribution.

In \cref{tab:performance} we report the performance by evaluating each policy using dynamic programming with the true arrivals distribution.
The Greedy heuristic accepts the first $B$ candidates regardless of their quality.
ML methods uses a single trace sampled from the non-stationary candidate arrival process, and use a policy that maps a $3$-dim state (the rounds and budget remaining, and the current candidate ability) to an \emph{accept} probability. 
For the hindsight planner, we use Equation~$2$ from~\citet{banerjee2020constant} which implements a linear program with $2N$ variables.

\subsection{Airline Revenue Management}

 Airline Revenue Management~\citep{littlewood1972forecasting} is a special case of the multi-dimensional Online Bin Packing (OBP) problem (recall that OBP exhibits vanishing hindsight bias via \cref{thm:delta_ip_bin_packing}).  The agent has capacity $B_k$ for $K$ different resources.  At each round, the decision-maker observes a request $A_t \in \mathbb{R}_+^K$ (the consumed capacity in each resource dimension), alongside a revenue $f_t$.  The algorithm can either \emph{accept} the request (obtaining revenue $f_t$ and updating remaining capacity according to $A_t$), or \emph{reject} it (note that partial acceptance is not allowed). The goal of the decision-maker is to maximize the expected revenue.

We use \ORSuite~\citep{archer2022orsuite} as an ARM simulator with fixed capacity, iid.~request types and job distribution. 
We use $T = \{5,10,100\}, K=3$, and $2$ request types.  The starting capacity for the three resources set to be $[8,4,4]$.  The iid arrival distribution is $(1/3, 1/3, 1/3)$ (where the last category corresponds to no arrival).  Job one arrivals have resource requests $[2,3,2]$ with revenue $1$, and job two arrivals have resource requests $[3,0,1]$ with revenue $2$.  This setting satisfies a dual-degeneracy condition of the hindsight planner from ~\citet{vera2021bayesian} which shows large regret for existing heuristics on these problems.

The optimal policy is computed through dynamic programming.  Both RL (Tabular Q-learning) and HL (Hindsight MAC) were trained on the same dataset, which contained $100$ traces.  In \cref{tab:performance} we report the performance of the policies through Monte Carlo simulations averaged over 500 iterations.

\subsection{Virtual Machine Allocation}

\begin{figure}[t]
\centering
\subfigure[Figure A]{\label{fig:core_usage}\includegraphics[width=.46\linewidth]{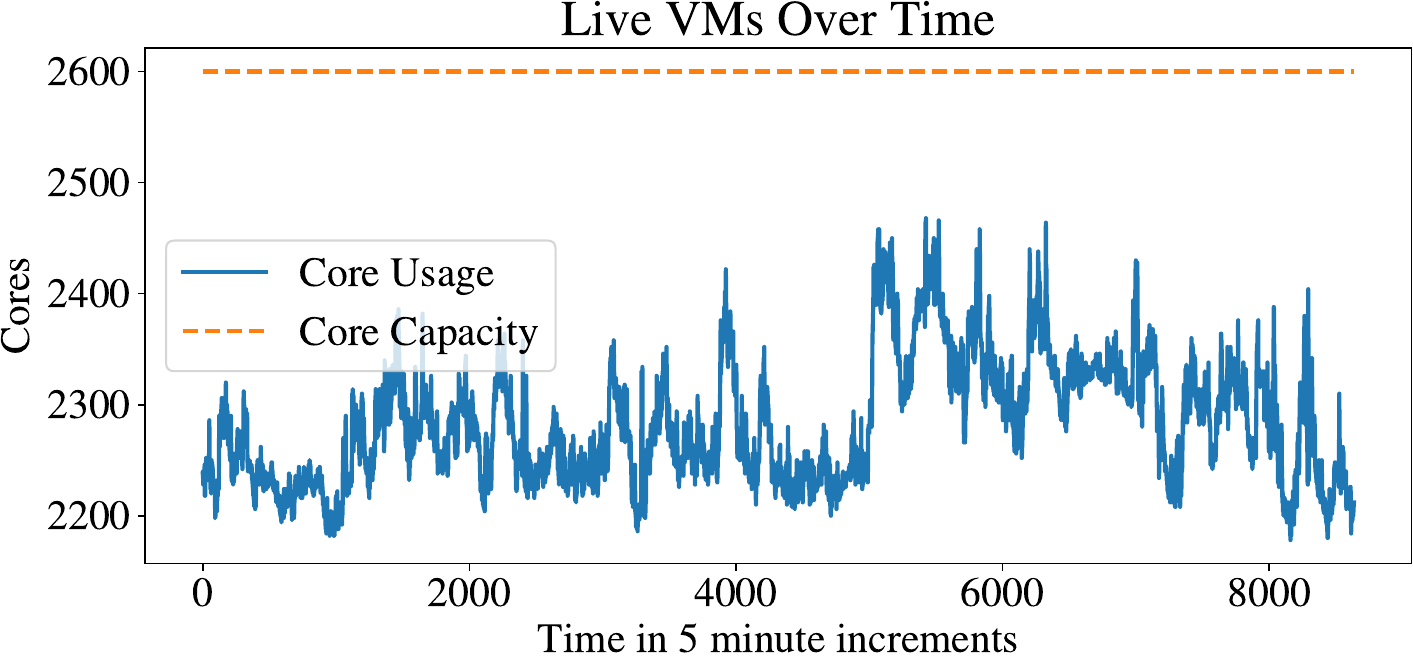}}
\subfigure[Figure B]{\label{fig:histogram}\includegraphics[width=.46\linewidth]{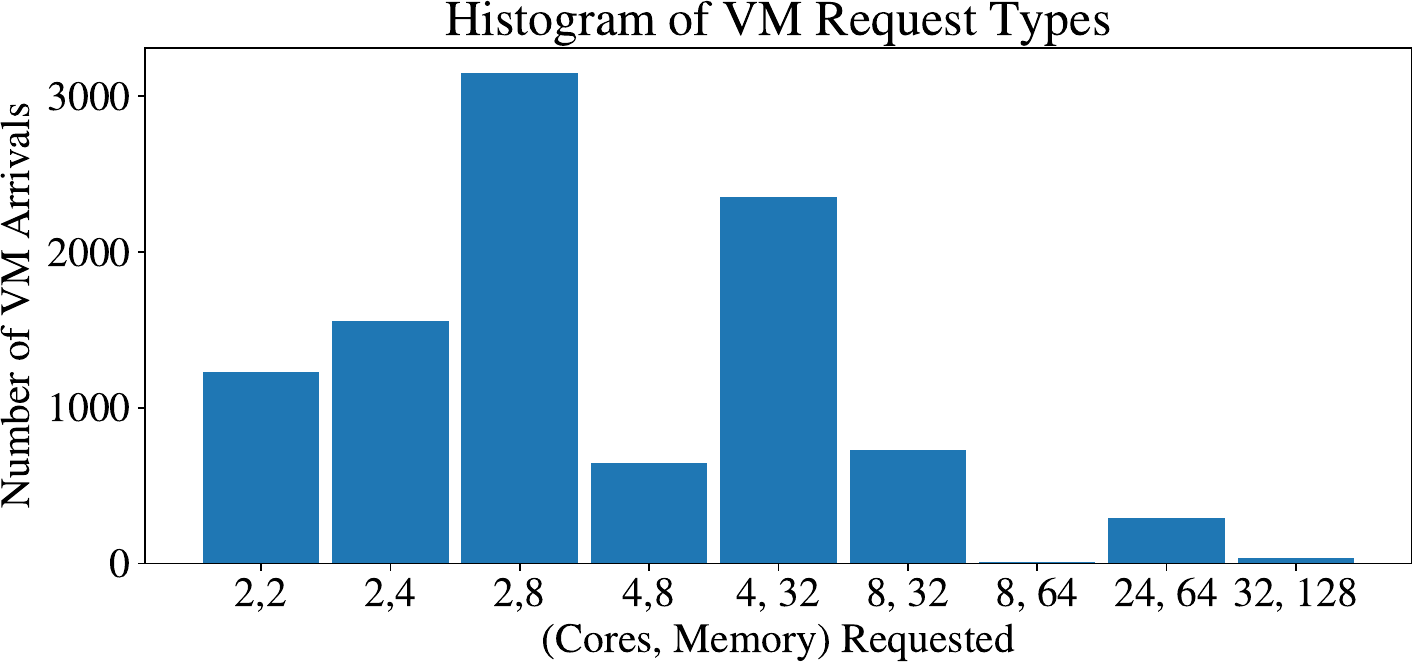}}
\caption{Sample of a thirty-day roll-out from the Azure Public Dataset.  In \cref{fig:histogram} we show a histogram of the various VM types and their corresponding cores and memory resources requested.  In \cref{fig:core_usage} we plot the used cores over time on the observed trace, along with the capacity of the $80$-node cluster simulated using MARO.}
\label{fig:vm_trace_summary}
\end{figure}

Cloud computing has revolutionized the way that computing resources are consumed.  These providers give end-users easy access to state-of-the-art resources.  One of the most crucial components to cloud computing providers is the Virtual Machine (VM) allocator, which assigns specific VM requests to physical hardware.  Improper placement of VM requests to physical machines (henceforth referred to as PMs) can cause performance impact, service delays, and create allocation failures.
The VMs serve as the primary units of resource allocation in these models. We focus on designing allocation policies at the \emph{cluster} level, which are a homogeneous set of physical machines with the same memory and CPU cores capacity.  

The cluster-specific allocator is tasked with the following:
\begin{itemize}[leftmargin=*]
    \item Coming VM requests ask for a certain amount of resources (CPU and memory requirements) along with their lifetime.  Resource requirements are varied based on the different VM requests.
    \item Based on the action selected by the allocator policy, the VM will be allocated to and be created in a specified PM as long as that PM's remaining resources are enough.
    \item After a period of execution, the VM completes its tasks.  The simulator will then release the resources allocated to this VM and de-allocate this VM from the PM.
\end{itemize}

\subsection{Stylized Environment}

We use \MARO, an open source package for multi-agent reinforcement learning in operations research tasks as a simulator for the VM allocator~\citep{MARO_MSRA}.  In this scenario the VM requests are uniformly sampled from the 2019 snapshot of the Azure Public Dataset~\citep{cortez2017resource}. 
The cluster is a fictitious one consisting of $80$ PMs that we found were similarly over-provisioned as in real-world clusters. See \cref{fig:core_usage} to highlight the demand workload against the cluster capacity for our experiment setup.

By default, \MARO provides reward metrics that can be used when specifying the objective of the algorithm.  The metrics provided include income, energy consumption, profit, number of successful and failed allocations, latency, and total number of overloaded physical machines.  However, typical cloud computing systems run in an \emph{over-provisioned} regime where the capacity of the physical machines is larger than the demand in order to ensure quality of service to its customers.  As a result, any reasonable algorithm has no failed allocations.  Hence, any reasonable algorithm also has identical values for income, energy consumption, profit, etc.  

\begin{figure}[!t]
\centering
\includegraphics[scale=0.35]{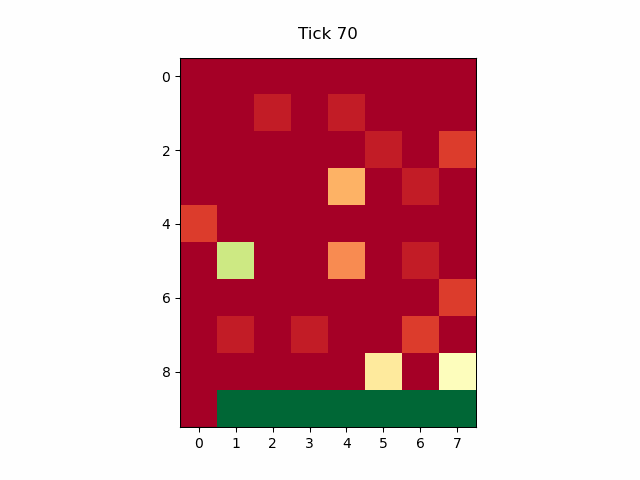}
\caption{Packing density for the Best Fit policy at one time-step. Each square corresponds to a specific PM and the colour corresponds to what portion of that PMs capacity is currently utilized.  Red corresponds to fully used, green is completely empty.  The packing density ignores the empty (or green) PMs on the bottom and counts the cumulative utilization ratio for the remaining PMs.}
\label{fig:packing_density}
\end{figure}

For the reward function we instead consider $r(s,a,\xi) = - 100*\texttt{Failed-Allocation} - 1 / \texttt{Packing Density}.$ The first component, one hundred times the number of failed allocations, helps to penalize the algorithms in training to ensure valid assignments for all of the VM requests.  The second component corresponds to the \texttt{Packing-Density}, computed via:
\[
    \texttt{Packing-Density} = \frac{\sum_{v \in VM} \texttt{Cores Usage}_v}{\sum_{p \in PM} \Ind{p \text{ is utilized }} \texttt{Capacity}_p}
\]

The numerator here is the total cumulative cores used for all of the VMs currently assigned on the system.  The denominator is the total capacity of all physical machines which are currently utilized (i.e. have a VM assigned and currently running).  The reason for picking this reward (and the inverse of it) is that:
\begin{itemize}[leftmargin=*]
    \item It allows for an easily expressible linear programming formulation (see \cref{sec:vm_oracle}).
    \item For any two policies which allocate all virtual machines, packing density serves as a criteria to differentiate the policies.  An algorithm which has large packing density equivalently uses the physical machines efficiently so that unused PMs can be re-purposed, reassigned, or potentially turned off.
    \item It serves as a proxy to ensure that the virtual machines are packed in as minimal number of physical machines possible.  This allows the allocator to be robust to hardware failures, where entire physical machines are potentially rendered unusable.
\end{itemize}
See \cref{fig:packing_density} for an illustration.

\subsubsection{Simulator Fidelity}

Our training procedure requires a faithful simulator of Virtual Machine allocation focused at a cluster level to validate our experimental results.  We found that the \MARO simulator captures all first-order effects of cloud computing environments specifically at the cluster level.  However, there are several effects not included in the simulations:
\begin{itemize}[leftmargin=*]
    \item When a virtual machine arrives to the system, they request a \emph{maximum} amount of CPU and memory capacity that they can use.  However, over time, any given VM request might only use a fraction of their requested resources.  Current cloud computing systems use an \emph{over-subscription} model where the requested memory and CPU cores for the VMs assigned to a PM can surpass its capacity.  However, when the total realized demand surpasses the PMs capacity, all of the VMs assigned to that system are failed and migrated to a different PM.  In contrast, \MARO assumes that each VM uses exactly its requested cores and memory over time, hence eliminating the need to model over-subscription on the cluster level.
    \item Typical systems involve live migration where virtual machines can be moved between physical machines without disrupting the VM request.  This is used in order to eliminate stranding which occurs when a physical machine has only a few long-running virtual machines allocated to it.  However, such an operation is costly and requires a large amount of system overhead.
    \item Our neural networks explicitly use the VM's lifetime as a feature in the state.  However, in true cloud computing systems the lifetime is unknown.  Only when a user decides to cancel a VM does the system have access to that information.  As such, it is typically observed in the trace dataset but cannot be used in policy network design.  We forgo this when modelling the policy as unlike the dynamics of VM request types, lifetimes for a VM are typically easy to model and these forecasts can be used as a replacement in the network representation~\citep{hadary_protean_2020}.
\end{itemize}
We believe that \MARO serves as a high fidelity simulator of the VM allocation problem at the cluster level while providing open source implementation for additional research experimentation.  However, we complement these results in \cref{sec:nylon} with a real-world model of cloud computing platforms.

\begin{algorithm}[tb]
   \caption{Training Procedure in \MARO}
   \label{alg:training_procedure}
\begin{algorithmic}[1]
   \STATE {\bfseries Input:} number of roll-outs, number of actors, number of training iterations.
   \FOR{each roll-out}
   \STATE For each actor, sample a random duration uniformly at random, and execute the \textbf{Best Fit} heuristic on a historical dataset of that length starting from an empty cluster
   \STATE For each actor, sample a one-hour trace of VM requests $\bfxi^i$
    \FOR{each actor}
        \STATE Collect dataset of $(s_t, a_t, \xi_t, r_t, s_{t+1})$ pairs under the current policy $\pi_\theta$
        \ENDFOR
    \STATE Add collected dataset to current experience buffer
    \FOR{each training iteration}
    \STATE Sub-sample batch from current collected dataset \\
    \STATE Update policy by gradient descent along the sampled batch \\
    \ENDFOR
   \ENDFOR
\end{algorithmic}
\end{algorithm}

\subsubsection{Training Implementation Details}

\cref{alg:training_procedure} presents the pseudocode for our training procedure using \MARO.  We repeat the following process for five hundred roll-outs.  First, we created realistic starting state distributions by executing the \emph{BestFit} heuristic for a random duration (greater than one day).  Line 4 samples one-hour traces of VM requests from the 2019 historical Microsoft Azure dataset.  Then, in line five and six, with fifteen actors in parallel we evaluate the current policy $\pi_\theta$ on the sampled VM request trace, adding the dataset of experience to the experience buffer.  Lines 9-11 samples batches of size $256$ where we update the policy and algorithmic parameters $\theta$ by gradient descent on the loss function evaluated on the sampled batch.  This process repeats for five thousand gradient steps.

We implemented the training framework using PyTorch~\citep{NEURIPS2019_9015} along with \MARO~\citep{MARO_MSRA}, and all experiments were conducted using the Microsoft Azure ML training platform.  For hyperparameters and neural network architectures for the Sim2Real RL and \GuidedRL algorithms, see \cref{sec:experiment_blackbox}.  All experiments were run on the same compute hardware and took similar runtimes to finish.  Runtime was dominated by the \MARO simulator executing the roll-outs under the curent policy versus the \plan~calls or the ML model updates.  As such, each algorithm was essentially given the same computational budget.

To evaluate the trained policies we subdivided the Azure Public Dataset into a temporally contiguous and non-overlapping training (first $15$ days) and test (last $15$ days) portions.  For evaluation we sampled fifty different one-day traces of VM requests from the held-out portion. For each of the different traces, we executed the policy in parallel to a greedy \textbf{Best Fit} algorithm.  Each deep learning algorithm was evaluated over five different random seed initializations and we tuned hyperparameters using grid search.  All metrics are reported with respect to cumulative differences against the baseline \textbf{Best Fit} policy, alongside statistical significance tests.  In \cref{tab:performance} we evaluate the number of machines required to pack the jobs.  Negative numbers correspond to fewer required PMs on average.  Asterisks correspond to statistical significance computed with Welch's $t$-test with $p=0.05$. In \cref{tab:test_performance} we provide the performance metrics on the underlying rewards as well.

\subsubsection{Heuristic Algorithms}

\begin{itemize}[leftmargin=*]
    \item \textbf{Random}: Picks a physical machine uniformly at random from the list of physical machines which have capacity to service the current VM request.
    \item \textbf{Round Robin}: Allocates to physical machines in a round-robin strategy by selecting a physical machine from the list of physical machines which have capacity to service the current VM request that was least recently used.
    \item \textbf{Best Fit}:
This algorithm picks a physical machine based on a variety of metric types.
\begin{itemize}
    \item \textbf{Remaining Cores}: Picks a valid physical machine with minimum remaining cores
    \item \textbf{Remaining Memory}: Picks a valid physical machine with minimum remaining memory
    \item \textbf{Energy Consumption}: Picks the valid physical machine with maximum energy consumption
    \item \textbf{Remaining Cores and Energy Consumption}: Picks a valid physical machine with minimum remaining cores, breaking ties via energy consumption
\end{itemize}
Similar heuristics to this are currently used in most large-scale cloud computing systems~\citep{hadary_protean_2020}.
    \item \textbf{Bin Packing}: Selects a valid physical machine which minimizes the resulting variance on the number of virtual machines on each physical machine. 
\end{itemize}

\subsubsection{Sim2Real RL Algorithms}
\label{sec:experiment_blackbox}

We also compared our \GuidedRL approaches to existing Sim2Real RL algorithms in the literature with custom implementation built on top of the \MARO package.  These include:
\begin{itemize}[leftmargin=*]
    \item \textbf{Deep Q Network (DQN)}: Double $Q$-Learning algorithm from~\citet{van2016deep}.
    \item \textbf{Actor Critic (AC)}: Actor Critic algorithm implementation from~\citet{konda2000actor}.
    \item \textbf{Mean Actor Critic (MAC)}: A modification of the actor critic algorithm where the actor loss is calculated along all actions instead of just the selected actions~\citep{asadi2017mean}.
    \item \textbf{Policy Gradient (VPG)}: A modification of the vanilla policy gradient with a exogenous input dependent baseline from~\citet{mao_variance_2019}.  Instead of training a baseline explicitly, we use $Q_t^{\text{Best Fit}}(s,a,\bfxi_{\geq t})$.
\end{itemize}

\subsubsection{State Features, Network Architecture, and Hyperparameters}
\label{sec:deep_rl_details}

The state space of the VM allocation scenario is combinatorial as we need to include the CPU and memory utilization of each physical machine across time to account for the lifetimes of the VMs currently assigned to the PM.  To rectify this, for each of the deep RL algorithms we use action-dependent features when representing the state space.  In particular, once a VM request arrives, we consider the set of physical machines that are available to service this particular virtual machine.  Each (PM, VM) pair has associated state features, including:
\begin{itemize}[leftmargin=*]
    \item The VM's CPU cores requirement, memory requirement, and lifetime
    \item The PM's CPU cores capacity, memory capacity, and type
    \item The historical CPU cores allocated, utilization, energy consumption, and memory allocated over last three VM requests
\end{itemize}
The last component is serving as a proxy for the historical utilization of the PM across all time to account for all VMs currently assigned to the PM.  The final action dependent features corresponds to the concatenation of these state features for each valid PM to service the current request.

We note that to use action-dependent features some of the algorithms required slight tweaking to their implementations.  In particular, when considering algorithms using a policy network representation (i.e. policy gradient, actor critic, or mean actor critic) when executing the policy we take $\pi_\theta = \text{Softmax}(\pi(s, a_1), \ldots, \pi(s, a_N))$ where $a_1, \ldots, a_N$ is the set of physical machines that can service the current request and $(s, a_i)$ is the corresponding state-features for the physical machine $a_i$.

For the actor and critic network representations in all algorithms we use a four layer neural network with $(32, 16, 8)$ hidden dimensions, an output dimension of one (due to the action-dependent features), and LeakyReLU activation functions.  For each of the algorithms we use the \emph{RMSprop} optimization algorithm.  We implemented the training framework using PyTorch~\citep{NEURIPS2019_9015} and \MARO.  All experiments were run on the same compute hardware and took similar runtimes to finish.  Runtime was dominated by the \MARO simulator executing the roll-outs, and the ML model updates and \plan~oracle calls were quicker.

Lastly we provide a list of the hyperparameters used and which algorithm they apply to when tuning algorithm performance.

\begin{table}[ht]
\captionsetup{size=footnotesize}
\caption{List of hyperparameters tuned over for the Sim2Real RL and \GuidedRL algorithms.} \label{tab:hp_tuning}
\setlength\tabcolsep{0pt} 

\smallskip 
\begin{tabular*}{\columnwidth}{@{\extracolsep{\fill}}llc}
\toprule
  Hyperparameter  & Algorithm & Values \\
\midrule
 Discount Factor & DQN, AC, MAC, PG & 0.9, 0.95, 0.99, 0.999 \\
 Learning Rate & All Algorithms & 0.05, 0.005, 0.0005, 0.00005, 0.000005 \\
 Entropy Regularization & PG, AC, MAC & 0, 0.1, 1, 10 \\
 Actor Loss Coefficient & AC, MAC & 0.1, 1, 10, 100 \\
 Target Update Smoothing Parameter & DQN & 0.0001, 0.001, 0.01, 0.1 \\
\bottomrule
\end{tabular*}
\end{table}

For concrete parameters evaluated and experiment results, see the attached code-base.

\subsubsection{Training Performance}

In \cref{fig:loss_curves} we include a plot of the loss curves for the various algorithms.

\begin{figure}[t]
\centering     
\subfigure[DQN]{\label{fig:aaa}\includegraphics[width=.45\linewidth]{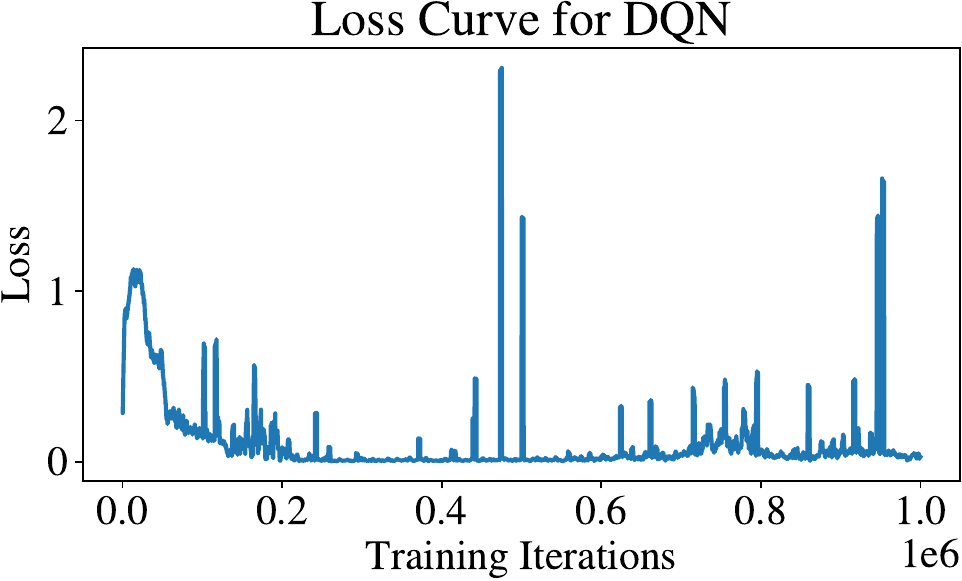}}
\subfigure[AC]{\label{fig:bbb}\includegraphics[width=.45\linewidth]{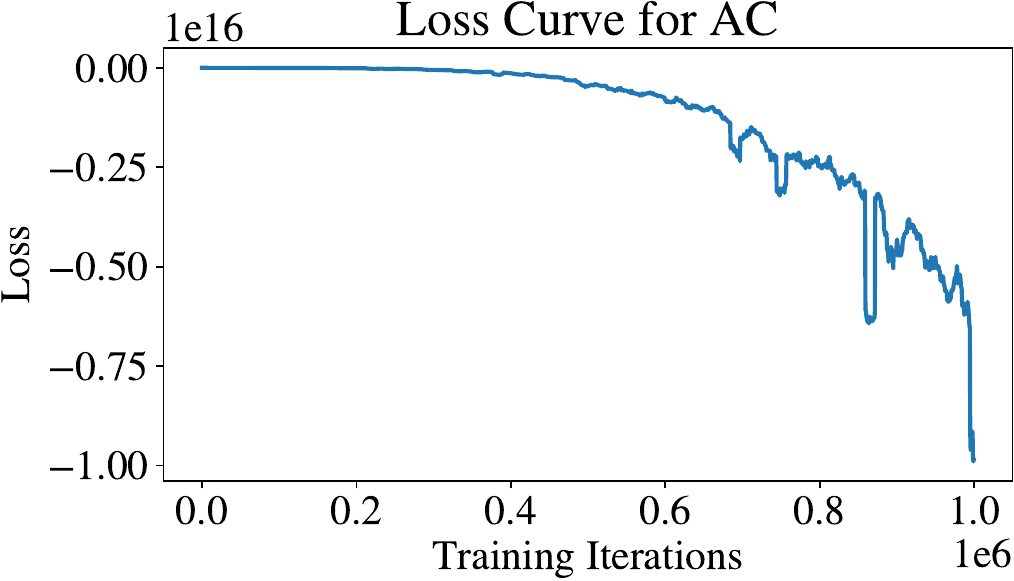}}
\subfigure[MAC]{\label{fig:ccc}\includegraphics[width=.45\linewidth]{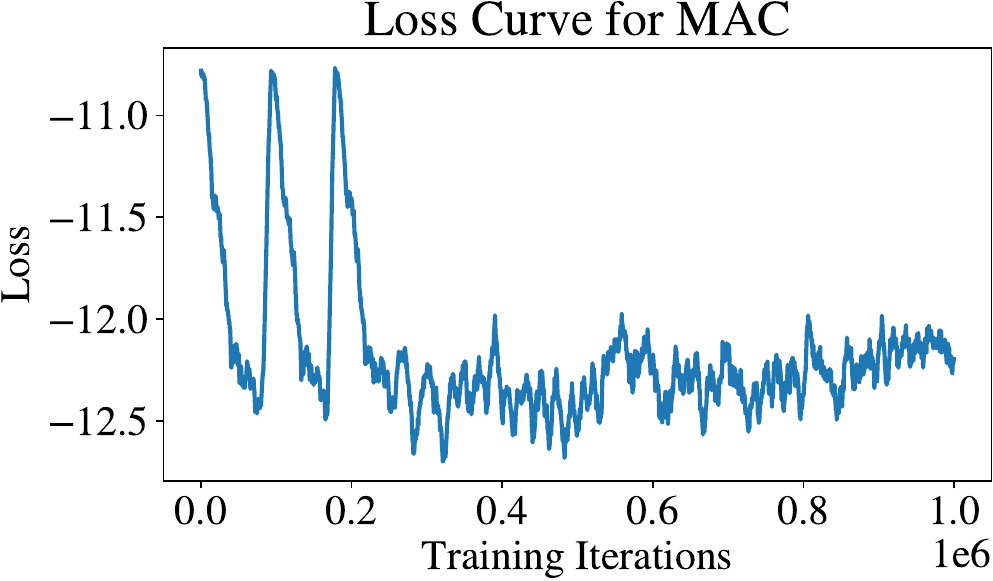}}
\subfigure[PG]{\label{fig:ddd}\includegraphics[width=.45\linewidth]{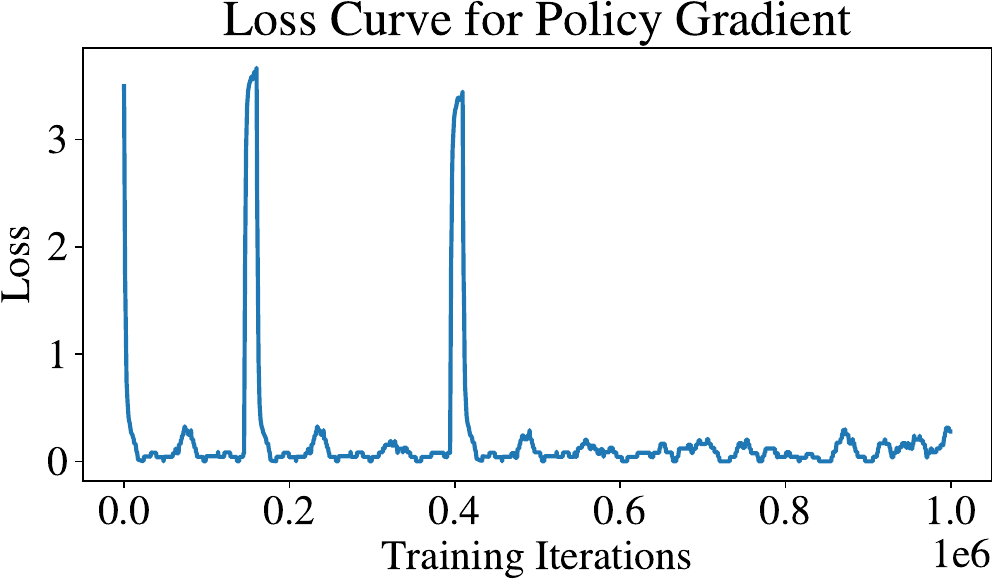}}
\subfigure[\GuidedQD]{\label{fig:eee}\includegraphics[width=.45\linewidth]{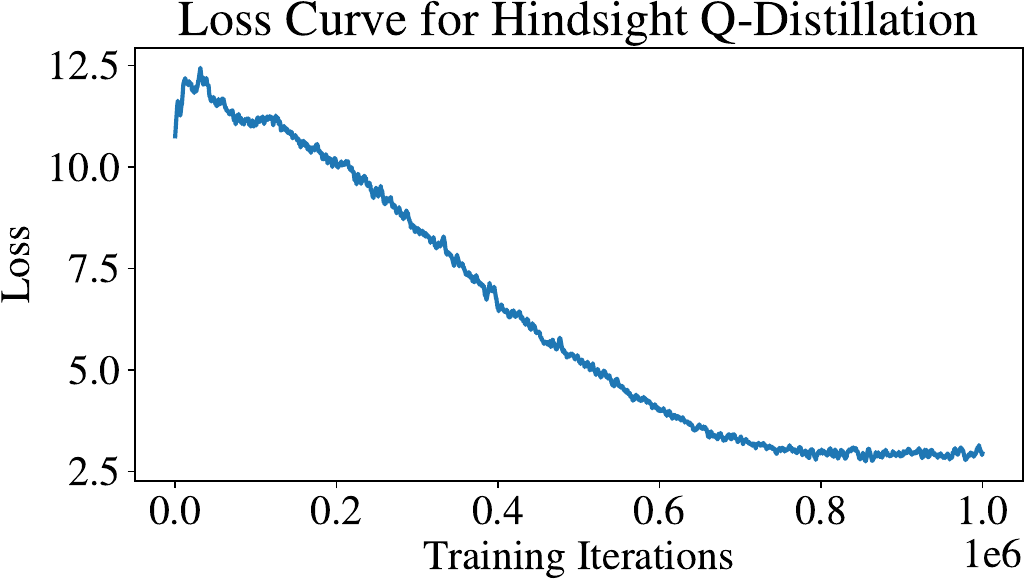}}
\subfigure[\GuidedMAC]{\label{fig:fff}\includegraphics[width=.45\linewidth]{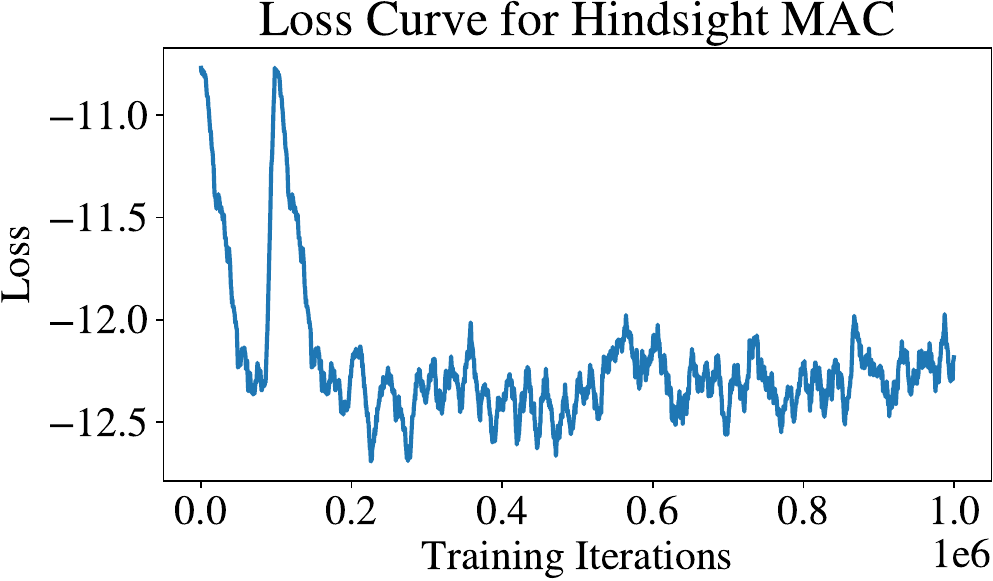}}
\caption{Moving average of the loss curves for the Sim2Real RL and \GuidedRL algorithms over the one million gradient steps computed in each of the experiments.  Note that some of the volatility occurs when an algorithm observes a datapoint with a failed allocation as there is a large penalty.}
\label{fig:loss_curves}
\end{figure}

\subsubsection{Testing Performance}

In \cref{tab:test_performance} we plot the performance of the algorithms on the true reward function.  Here we include the true reward considered:
$$r(s,a,\xi) = - 100*\texttt{Failed-Allocation} - 1 / \texttt{Packing Density}(s)$$
as well as simply $\texttt{Packing Density}(s)$.  All measures are reported with a $95\%$ confidence interval, and computed as differences against the \textbf{Best Fit} allocation policy.

\begin{table}[!tb]
\captionsetup{size=footnotesize}
\caption{Performance of heuristics, Sim2Real RL, and \GuidedRL algorithms on VM allocation.  Here we include the true reward metric the algorithms were trained on (negative inverse of the packing density) and the packing density improvements on average.} \label{tab:test_performance}
\setlength\tabcolsep{0pt} 

\smallskip 
\begin{tabular*}{\columnwidth}{@{\extracolsep{\fill}}lcc}
\toprule
  Algorithm  & Performance $r = -1 / \texttt{Packing Density}$ & Packing Density\\
\midrule
  Performance Upper Bound & $0.66 \pm 0.29$ & $.09\% \pm 0.03\%$ \\

 \midrule 
 Best Fit & 0.0 & 0.0\\
 Bin Packing & $-64.44 \pm 2.49$ & $-5.34\% \pm 0.2\%$ \\
 Round Robin & $-56.36 \pm 2.65$ & $-4.67\% \pm 0.22\%$\\
 Random & $-48.94 \pm 2.33$ & $-4.08 \% \pm 0.19\%$\\

\midrule
 DQN & $-1.00 \pm 0.41$ & $-0.05\% \pm 0.04\%$\\
 MAC & $-0.38 \pm 0.033$ & $-0.03\% \pm 0.00\%$\\
 AC & $-2.94 \pm 0.61$ & $-0.21\% \pm 0.06\%$\\
 Policy Gradient & $-1.03 \pm 0.39$ & $-0.06\% \pm 0.04\%$ \\
\midrule
 \textbf{\GuidedMAC} & $\mathbf{0.18 \pm 0.093}$ & $\mathbf{0.05 \% \pm 0.00\%}$\\
 \GuidedQD & $0.08 \pm 0.32$ & $0.04 \% \pm 0.029\%$\\
\bottomrule
\end{tabular*}
\end{table}

\subsection{Real-World VM Allocation}
\label{app:nylon_details}

In this simulation we consider the more realistic setting of VM arrivals in continuous time and where the allocation agent has no information about VM lifetimes.  We avoid giving a concrete description of the reward function trained, cluster sizes, etc. to preserve the confidentiality of the cloud operator.  However, we briefly describe the training implementation details, heuristic algorithms, as well as the hindsight heuristic used.

\subsubsection{Heuristic Algorithms}

\cref{tab:nylon_results} summarizes the results of performance for three different algorithms:

\paragraph{BestFit} Performance is shown relative to a {\bf BestFit} strategy used in production.  This strategy follows a proprietary implementation that prioritizes between CPU and memory depending on their scarcity.

\paragraph{\GuidedRL and Sim2Real RL} We adapt \GuidedMAC (HL) and compare it with MAC~\cite{asadi2017mean} (RL), where both used the same network architecture, which embeds VM-specific and PM-specific features using a $6$-layer GNN. The resulting architecture is rich enough to represent the {\bf BestFit} heuristic, but can also express more flexible policies.

\subsubsection{Training Implementation Details}

We trained each algorithm over $3$ random seeds and evaluated $5$ rollouts to capture the variation in the cluster state at the start of the evaluation trace. Unlike the other experiments, we cannot account for the randomness in exogenous samples because we only have one evaluation trace for each cluster.  Error metrics are computed with a paired $t$-test of value $p = 0.05$.

\subsubsection{Hindsight Heuristic}
\label{app:heuristic_performance}

Due to the large scale of the real-world scenarios, even the linear relaxation of the integer program was not tractable. Consequently, we resort here to using a carefully designed \emph{hindsight heuristic} (\cref{alg:hindsight_heuristic}) to derive $\plan(t, \bfxi, s)$. The heuristic is based on prioritizing VMs according to both their size and duration. See~\cref{alg:hindsight_heuristic} for pseudocode.

In \cref{tab:heuristic_check} we separately tested the accuracy of our heuristic for the hindsight planner (both by comparing to the optimal in small instances, as well as by comparing to a lower bound given in \citet{buchbinder2021online}), and concluded that the heuristic obtains a value that is within few percentages of the optimum.  We found that the dual gap of~\cref{alg:hindsight_heuristic} was typically within $4\%$ of the optimum. 

\begin{algorithm}[t]
\begin{algorithmic}[1]
  \STATE {\bfseries Input:} A cluster state $s$, sequence of VM requests $\bfxi_{t:T}$.
    \STATE Sort requests in descending order of their lifetimes
    \FOR{Each request $\xi$}
    \STATE Allocate to the feasible PM where $\xi$ is the only live VM on it for the least amount of time
    \ENDFOR
\end{algorithmic}
\captionof{algorithm}{Hindsight Heuristic.}
\label{alg:hindsight_heuristic}
\end{algorithm}

\begin{table}[t]
\captionsetup{size=footnotesize}
\caption{How close to optimal is the Upper Bound (Oracle)? We measure the average UsedPMs of the Oracle's solution and compare with the lower bound which assumes that VMs can be fractionally split across PMs. } \label{tab:heuristic_check}
\setlength\tabcolsep{0pt} 

\smallskip 
\begin{tabular*}{\columnwidth}{@{\extracolsep{\fill}}lccc}
\toprule
  Cluster  & Upper Bound (Oracle) & Lower Bound & Gap (\%)\\
A	& 467.42		&	499.625	& 6.89\% \\
B   &	538.35	&		578.214	& 7.41\% \\
C & 383.19	&		391.329	& 2.12\% \\
D & 27.47	&		32.9769	& 20.07\% \\
E & 448.86	&		475.968	& 6.04\% \\
F &	577.97	&		625.083	& 8.15\% \\
G &	2252.65	&		2287.21	& 1.53\% \\
H &	2295.19 &		2332.31	& 1.62\% \\
I & 341.90	&		361.654	& 5.78\% \\
J	&1212.91	&		1239.9	& 2.23\% \\
K	&565.34	&		570.532	& 0.92\% \\
L	&8.23	&		8.85826	& 7.64\% \\
M	&8.77	&		9.3152	& 6.25\% \\
N	&305.10	&		310.096	& 1.64\% \\
O	&43.27	&		45.3596	& 4.82\% \\
P	&2528.72	&		2588.38	& 2.36\% \\
Q	&1457.37	&		1481.13	& 1.63\% \\
R	&123.04	&		124.866	& 1.49\% \\
S	&2452.25	&		2491.68	& 1.61\% \\
T	&68.68	&		70.0956	& 2.07\% \\
U	&533.61	&		539.872	& 1.17\% \\
V	&1243.70	&		1260.14	& 1.32\% \\
W	&1678.88	&		1702.09	& 1.38\% \\
X	&158.71	&		171.058	& 7.78\% \\
ALL & & & 4.33\% \\
\midrule
 \bottomrule
\end{tabular*}
\end{table}

\end{document}